\documentclass{article}
\usepackage{amsmath, amssymb, amsthm}
\usepackage{mathtools}
\usepackage{cite}
\usepackage{hyperref}

\newtheorem{rmk}{Remark}[section]
\usepackage{xcolor}
\newtheorem{defn}{Definition}
\newtheorem{thm}{Theorem}
\newtheorem{lem}{Lemma}
\newtheorem{exmp}{Example}

\newtheorem{claim}{Claim}
\newtheorem{cor}{Corollary}

\title{Evidential-Based Higher-Order Set Argumentation Framework}
\author{Shuai Tang\thanks{Corresponding author. Email: TangShuaiMath@outlook.com.}}

\date{} 

\begin{document}
	\maketitle
	
	\begin{abstract}
		Evidential argumentation extends Dung's abstract argumentation by requiring arguments and interactions to be backed by chains of evidence rooted in prima-facie elements. However, existing formalisms lack a unified treatment of evidential support, higher-order relations (attacks and supports targeting arbitrary elements), and collective interactions (sources as sets). In this paper, we introduce the Evidential-Based Higher-Order Set Argumentation Framework (EHSAF), which conservatively generalises several existing frameworks within a single expressive setting. We develop two complete semantics for EHSAFs: an \emph{adjacent complete labelling semantics} that admits multiple truth values (true, false, undecided) for arguments in support cycles, reflecting an open epistemic attitude toward future evidence; and an \emph{extension-based complete semantics} that follows a strict evidentialist stance, accepting only arguments with well-founded support chains. We show that these two semantics diverge in the presence of support cycles, and prove their equivalence under support‑acyclicity. To enable computational reasoning, we provide a normal propositional encoding of EHSAFs and prove that, in three-valued {\L}ukasiewicz logic, its models correspond precisely to the adjacent complete labellings. We further extend this encoding to continuous fuzzy logics (G{\"o}del, Product, and {\L}ukasiewicz), defining a continuous fuzzy normal encoded semantics. We establish that this fuzzy semantics satisfies key properties---continuity, monotonicity, boundary conditions, and solution existence---and that its ternarisation recovers the adjacent complete labellings under natural t-norm conditions. Our framework thus unifies expressive argumentation with principled three-valued and fuzzy semantics, bridging the gap between qualitative and quantitative reasoning about evidence.
	\end{abstract}
	
	
	\section{Introduction}
	
	Since its introduction by Dung \cite{dung1995acceptability}, abstract argumentation has become a cornerstone of knowledge representation and reasoning in artificial intelligence. In Dung's argumentation frameworks (DAFs), arguments are treated as abstract entities, and the only interaction between them is a binary attack relation. This simple yet powerful model allows the identification of acceptable sets of arguments—called extensions—under various semantics, such as admissible, complete, preferred, grounded, and stable extensions \cite{dung1995acceptability}. The framework has been successfully applied in diverse domains, including legal reasoning \cite{benchcapon2009argumentation}, negotiation \cite{dimopoulos2019argumentation}, and decision making \cite{amgoud2009using}.
	
	Despite its elegance and wide adoption, the abstract nature of DAFs imposes limitations on their expressive power. In particular, DAFs only represent negative interactions (attacks) between arguments, leaving positive interactions unmodeled. To address this limitation, researchers have proposed various extensions that introduce a support relation between arguments \cite{cohen2014survey,cayrol2013bipolarity}. Among these, the evidential interpretation of support has gained significant attention. The Evidential Argumentation System (EAS), introduced by Oren and Norman \cite{oren2008semantics} and later refined by Polberg and Oren \cite{polberg2014revisiting}, captures the intuition that arguments cannot be accepted in isolation; rather, they must be backed by evidence. In an EAS, a special argument \(\eta\) represents indisputable evidence from the environment, and an argument is considered evidentially supported only if it can be traced back to \(\eta\) through a chain of supports. This framework distinguishes between \emph{prima-facie} arguments (which require no support) and \emph{standard} arguments (which must be supported by evidence) \cite{oren2008semantics}. The EAS framework has been further studied and compared with other support formalisms, such as Argumentation Frameworks with Necessities (AFNs) \cite{nouioua2013afs}, revealing deep structural connections between different interpretations of support \cite{polberg2014revisiting}.
	
	Alongside the evolution of formal frameworks, two strands of research have deepened the study of evidence-based argumentation from computational and quantitative perspectives, respectively.
	On the computational side, Chen et al. \cite{chen2023evidence} proposed an incremental semantics for evidence-based argumentation frameworks. Based on strongly connected component decomposition, their approach layers the argumentation structure and computes acceptability in a bottom-up manner, which effectively reduces recomputation cost under dynamic updates.
	Despite these advances, existing evidence-based formalisms remain confined to first-order, binary interactions between arguments. They can neither represent higher-order scenarios where attacks or supports are themselves attacked or supported, nor model collective interactions where a set of arguments jointly carries out an attack or support. More critically, there is still no unified labelling-based semantic system that simultaneously integrates evidential support, higher-order interactions and collective sources, with corresponding propositional logic encodings and continuous fuzzy extensions. This leaves a significant gap between expressive formalisms and principled, computable semantic foundations.
	
	Another important line of generalization concerns the nature of the attack relation itself. In Dung's original framework, attacks are binary relations between individual arguments. However, in many real-world scenarios, attacks may require the cooperation of multiple arguments to be effective, or they may target not only arguments but also other attacks or supports. The former generalization gives rise to frameworks with \emph{collective attacks} (also known as SETAFs), where a set of arguments can jointly attack an argument \cite{nielsen2006generalization,flouris2019comprehensive}. The latter leads to \emph{higher-order argumentation frameworks}, where attacks can target other attacks—a concept explored in frameworks such as AFRA \cite{baroni2011afra} and RAF \cite{cayrol2020valid}. These two lines of generalization have been combined in the Higher-order Set Argumentation Framework (HSAF), which allows both set-based attackers and higher-order targets, thereby providing a unified treatment of complex interaction patterns \cite{tang2025encoding2}.
	
	A further step toward unifying these extensions is the Recursive Evidence-Based Argumentation Framework (REBAF) \cite{cayrol2018argumentation}, which integrates higher-order attacks and evidential supports. In a REBAF, attacks and supports are explicitly named, their sources can be sets of arguments, and their targets can be arguments, attacks, or supports. A set of \emph{prima-facie} elements (which may include arguments, attacks, or supports) serves as the root of evidential support chains. Semantics for REBAF are defined in terms of \emph{structures}—triples consisting of accepted arguments, valid attacks, and valid supports—and have been shown to generalize both Dung's semantics and those of EAS. A logical encoding of REBAF into first-order logic has also been proposed, establishing a correspondence between REBAF structures and models of a logical theory \cite{cayrol2020logical}. More recently, Besnard et al. \cite{besnard2023generic} developed a generic logical encoding that uniformly captures several families of abstract argumentation frameworks, including those with coalitions, higher-order relations, and evidential supports.
	
	From a semantic perspective, principled comparison has become a standard approach to systematise labelling-based argumentation semantics. Most notably, Al Anaissy et al. \cite{alanaissy2025principle} established a systematic taxonomy of seven variants of complete labelling semantics for bipolar argumentation, covering deductive, necessary and evidential interpretations of support. Their work clarifies the inclusion relations and behavioural differences among different support interpretations under the labelling paradigm, laying a principled foundation for semantic selection.
	Nevertheless, such principled labelling analyses are still limited to standard first-order bipolar frameworks with singleton sources and argument-level targets. They do not extend to collective interactions or higher-order targets, nor do they provide a unified logical encoding that connects discrete labellings to fuzzy quantitative semantics. It therefore remains an open question how to construct a unified evidential higher-order set argumentation framework whose labelling semantics conservatively generalise existing ones, admit a sound and complete propositional encoding, and naturally extend to continuous fuzzy t-norm semantics.
	
	While these frameworks significantly enhance the expressive power of abstract argumentation, they have largely been studied in isolation or with specific combinations of features. The need for a unified framework that simultaneously accommodates (i) evidential support, (ii) higher-order attacks and supports, and (iii) collective (set-based) interactions, while providing a coherent semantic and propositional logic foundation, remains an open challenge. Such a framework would not only subsume existing formalisms but also enable the systematic study of their interrelationships and the development of generic computational tools.
	
	In this paper, we introduce the \emph{Evidential-based Higher-order Set Argumentation Framework} (EHSAF), a unified abstract argumentation framework that integrates the three key generalizations discussed above:
	
	\begin{enumerate}
		\item \emph{Evidential support}, following the EAS tradition, where arguments and interactions must be supported by chains of evidence rooted in \emph{prima-facie} elements;
		\item \emph{Higher-order interactions}, where attacks and supports may target not only arguments but also other attacks or supports;
		\item \emph{Collective interactions}, where the source of an attack or support can be a set of elements (arguments or interactions), rather than a single element.
	\end{enumerate}
	
	The EHSAF generalizes the HSAF by incorporating evidential support and by allowing supports—in addition to attacks—to be higher-order and collective. It also extends the REBAF by permitting the source of any interaction to be a set of arbitrary elements (not just arguments) and by providing a more flexible treatment of support cycles.

	In this paper, we provide two distinct semantic attitudes toward arguments without current prima-facie support. Extension-based complete semantics adheres to a strict evidentialist stance grounded in real-world epistemic practice: it regards an argument as false if it currently lacks prima-facie support. By contrast, adjacent complete labelling semantics takes prospective valid evidence into account, and admits possible truth-value scenarios for arguments that have no immediate prima-facie support but may obtain such support in the future.
	
	To elaborate, in standard evidential-based argumentation semantics, a claim must be considered defeated if it is successfully attacked by an effective counterargument. When evaluating argument effectiveness, we focus on the dimension of evidential support, distinguishing between arguments that possess prima-facie evidential support and those that do not. An ordinary non-prima-facie argument, when isolated---that is, lacking any form of evidential backing---must be regarded as ineffective. Indeed, it is commonly observed that a position advanced by an agent, yet endorsed by neither others nor the proponent themselves, typically fails to maintain its tenability.
	
	Conversely, if an argument receives support from other arguments, or if the proponent firmly adheres to it, the position may remain plausible even in the absence of currently available prima-facie evidence. Crucially, one cannot assert that a claim lacking such evidence is false per se, since the requisite evidence may surface in the future.
	
	A paradigmatic illustration is the historical development of the theory of relativity in physics, which can be construed as one long-running argumentative claim. When Albert Einstein first proposed the theory, it faced substantial opposition but was not undermined by any decisive effective counterargument. Moreover, although experimental evidence was initially absent, Einstein himself, along with several insightful physicists, steadfastly supported it. Under these circumstances, the theory ought not to be hastily dismissed as false; rather, an open epistemic attitude is warranted, as subsequent experiments could either refute or confirm it.
	
	Formally, the truth value of such a claim may be assigned as true ($1$), false ($0$), or undecidable ($\frac{1}{2}$), where the last denotes empirical untestability or the mere absence of completed testing. In the case of relativity, subsequent verification established its truth value as $1$. Thus, to assert its falsity during the pre-empirical stage would constitute a premature and erroneous judgment.
	
	However, the prevailing evidentialist criterion precludes the acceptance of any argument that lacks currently surfaced prima-facie evidence. To address this limitation and broaden the scope of argumentation semantics, we propose a three-valued labelling framework termed \textit{adjacent complete labelling semantics}, which is designed to accommodate arguments that remain undefeated yet lack immediate prima-facie support. Specifically, this semantics applies a tolerant stance to arguments situated in evidential support cycles, systematically enumerating all possible truth-value assignments for such arguments. For instance, consider a framework in which two arguments mutually provide evidential support. Neither one of the arguments possesses current prima‑facie support: they may be jointly true (should the anticipated evidence eventually emerge), jointly false (should future evidential counterexamples refute either of them), or perpetually undecided.
	Hence, our framework incorporates a more comprehensive set of logical possibilities without categorically excluding potentially sound claims solely on the grounds of absent current evidence.
	
	Subsequently, drawing on an extended syntax that allows attacks and supports to collectively target or back other elements (including arguments, attacks, and supports), we formally define the proposed adjacent complete labelling semantics and compare it with the traditional extension-based strict semantics. We develop both labelling-based and extension-based semantics for EHSAFs, drawing on and extending the structure-based approach of REBAF \cite{cayrol2018argumentation}. We show that the resulting semantics conservatively generalize those of EAS, REBAF, and HSAF. We first prove that the proposed adjacent complete labelling semantics is equivalent to the three-valued equational semantics. Furthermore, we propose a logical encoding of EHSAFs into propositional logic, extending the encoding techniques developed for REBAF \cite{cayrol2020logical} and HSAF \cite{tang2025encoding2}. We establish that the models of the encoded theory in the three-valued {\L}ukasiewicz propositional logic system correspond exactly to the adjacent complete labellings of the EHSAF, providing a computational and logical foundation for reasoning in this unified framework. Moreover, by encoding EHSAFs into fuzzy logic systems, we naturally derive $[0,1]$-valued equational semantics. Finally, we explore the core formal properties of this fuzzy encoded semantics and its precise relationships with the adjacent complete labelling framework.
	
	The contributions of this paper are as follows:
	\begin{enumerate}
		\item We define the syntax of EHSAFs, unifying evidential support, higher-order interactions, and collective interactions in a single abstract framework (Section 3.1).
		\item We introduce adjacent complete, stable, preferred, and grounded labelling semantics for EHSAFs, generalizing the corresponding semantics from EAS and REBAF (Section 3.2.1).
		\item We develop extension-based semantics for EHSAFs, showing their equivalence to the labelling semantics under acyclicity conditions (Section 3.2.2).
		\item We provide a normal encoding of EHSAFs into propositional logic and prove the equivalence between the encoded semantics and the adjacent complete semantics (Section 4.1--4.2).
		\item We extend the encoding to fuzzy propositional logics, defining continuous fuzzy normal encoded semantics and establishing its core properties, including continuity, monotonicity, boundary conditions, and solution existence (Section 4.3).
		\item We instantiate the fuzzy framework with G{\"o}del, Product, and {\L}ukasiewicz t-norms, and prove the correspondence between these fuzzy semantics and the 3-valued adjacent complete semantics (Section 4.4).
	\end{enumerate}
	
	The remainder of this paper is organized as follows. Section 2 reviews the necessary background on EAS, REBAF, and propositional logic systems. Section 3 presents the syntax and basic semantics of EHSAFs, including both labelling-based and extension-based semantics. Section 4 develops the logical encoding of EHSAFs, covering both discrete and fuzzy encodings. Section 5 discusses related work, and Section 6 concludes the paper.
	
	\section{Preliminaries}

	\subsection{Syntax and Semantics of Evidential Argumentation}
	\subsubsection{Syntax and Semantics of EASs}
	In this part, we recall the formal background of Evidential Argumentation Systems (EASs). Originally introduced by Oren and Norman \cite{oren2008semantics}, the framework was later revised by Polberg and Oren \cite{polberg2014revisiting} in a generalized and consistent frame with addressing the relationship between EASs and argumentation framework with necessities (AFNs). We review the revised definitions presented in \cite{polberg2014revisiting}.
	\begin{defn}[Evidential Argumentation System \cite{polberg2014revisiting}]
		\label{def:eas_syntax}
		An \textit{Evidential Argumentation System} (EAS) is a tuple $EAS = \langle A, R, E \rangle$, where:
		\begin{itemize}
			\item $A$ is a finite set of arguments. It includes a special \textit{prima facie} argument, denoted as $\eta$, which represents support from the environment (i.e., indisputable evidence).
			\item $R \subseteq (2^A \setminus \{\varnothing\}) \times A$ is the attack relation. We require that $\eta$ cannot be part of the attacking set and cannot be attacked (i.e., for any $(X, a) \in R$, $\eta \notin X$ and $a\neq\eta$).
			\item $E \subseteq (2^A \setminus \{\varnothing\}) \times A$ is the support relation. The special argument $\eta$ cannot be the target of support (i.e., there is no $(X, \eta) \in E$).
		\end{itemize}
	\end{defn}

	The semantics of EASs are built upon the notion that an argument is only valid if it can be traced back to the environmental evidence $\eta$ through a chain of support.
	
	\begin{defn}[Evidential Support \cite{polberg2014revisiting}]
		\label{def:esupport}
		An argument $a \in A$ is \textit{evidentially supported} (or \textit{e-supported}) by a set $S \subseteq A$ iff:
		\begin{enumerate}
			\item $a = \eta$; or
			\item there exists a non-empty subset $S' \subseteq S$ such that $(S', a) \in E$, and for all $x \in S'$, $x$ is e-supported by $S \setminus \{a\}$.
		\end{enumerate}
	\end{defn}
	
	In this revised definition, $\eta$ is trivially e-supported by any set, which resolves the deadlock in support chains.
	
	\begin{defn}[E-supported Attack \cite{polberg2014revisiting}]
		\label{def:eattack}
		A set $S \subseteq A$ carries out an \textit{evidence-supported attack} (or \textit{e-supported attack}) on an argument $a$ iff there exists $S' \subseteq S$ such that $(S', a) \in R$, and for all $s \in S'$, $s$ is e-supported by $S$.
	\end{defn}
	
	\begin{defn}[Acceptability \cite{polberg2014revisiting}]
		\label{def:acceptability}
		An argument $a \in A$ is \textit{acceptable} with respect to a set $S \subseteq A$ iff:
		\begin{enumerate}
			\item $a$ is e-supported by $S$; and
			\item for every \textit{minimal} e-supported attack by a set $T \subseteq A$ against $a$, $S$ carries out an e-supported attack against at least one member of $T$.
		\end{enumerate}
	\end{defn}
	Based on the concepts above, the standard Dung-style extensions are then adapted for EAS in \cite{polberg2014revisiting}.

	 \subsubsection{Syntax and Semantics of REBAFs} \label{subsec:rebaf}
	 
	 We recall the syntax and semantics of REBAF following
	 \cite{cayrol2018argumentation,cayrol2020logical,besnard2023generic}.
	 
	 \begin{defn}[REBAF]\label{def:rebaf}
	 	A \emph{recursive evidence-based argumentation framework} (REBAF) is a sextuple
	 	$(A, R_a, R_e, s, t, P)$ where:
	 	\begin{itemize}
	 		\item $A$ is a (possibly infinite) set of \emph{arguments};
	 		\item $R_a$ is a set of \emph{attack} names, with $A\cap R_a=\varnothing$;
	 		\item $R_e$ is a set of \emph{support} names, with
	 		$A\cap R_e = R_a\cap R_e = \varnothing$;
	 		\item $P\subseteq A\cup R_a\cup R_e$ is the set of \emph{prima-facie} elements
	 		(arguments, attacks or supports that need not be supported);
	 		\item $s : R_a\cup R_e \longrightarrow 2^{A}\setminus\{\varnothing\}$
	 		maps each attack or support to its \emph{source}, a non-empty set of arguments;
	 		\item $t : R_a\cup R_e \longrightarrow A\cup R_a\cup R_e$
	 		maps each attack or support to its \emph{target},
	 		which may be an argument, an attack or a support.
	 	\end{itemize}
	 \end{defn}
	 
	 \noindent
	 Intuitively, an attack $\alpha\in R_a$ defeats its target whenever $\alpha$ is
	 \emph{valid} and its source is accepted; a support $\beta\in R_e$ contributes to
	 the evidential support of its target.  Prima-facie elements are regarded as
	 already supported \cite{cayrol2018argumentation,cayrol2020logical}.
	 
	 \begin{defn}[Structure]\label{def:rebaf-structure}
	 	A \emph{structure} of a REBAF $(A, R_a, R_e, s, t, P)$ is a triple
	 	$U = (S,\Gamma,\Delta)$ with
	 	$S\subseteq A$, $\Gamma\subseteq R_a$ and $\Delta\subseteq R_e$.
	 	The set $S$ collects the \emph{accepted arguments},
	 	$\Gamma$ the \emph{valid attacks}, and $\Delta$ the \emph{valid supports}
	 	with respect to $U$.
	 \end{defn}
	 
	 Given a structure $U=(S,\Gamma,\Delta)$ of a REBAF $(A, R_a, R_e, s, t, P)$, the following auxiliary notions are
	 defined in \cite{cayrol2018argumentation,cayrol2020logical}.	 
	 \begin{itemize}
	 	\item Sets of \emph{defeated elements} w.r.t. $U$ are defined as follows:
	 	 $$Def_X(U)\stackrel{\text{def}}{=}\{\,x\in X \mid \exists\alpha\in\Gamma,\; s(\alpha)\subseteq S,\; t(\alpha)=x\,\},$$
	 	where $X\in\{A,R_a,R_e\}$; 
	 	$$Def(U)\stackrel{\text{def}}{=} Def_A(U)\cup Def_{R_a}(U)\cup Def_{R_e}(U).$$
	 	
	 	\item The set of \emph{evidentially supported elements} w.r.t. $U$ is defined as follows:
	 	\[
	 	 Sup(U)\stackrel{\text{def}}{=} P\;\cup\;
	 	\{\,t(\alpha)\mid \alpha\in\Delta\cap Sup(U'),\;
	 	s(\alpha)\subseteq S\cap Sup(U')\,\},
	 	\]
	 	where $U'=U\setminus\{t(\alpha)\}$.
	 	
	 	\item The set of \emph{unsupportable elements} w.r.t. $U$ is defined as follows:
	 	$$Def_s(U)\stackrel{\text{def}}{=}(A\cup R_a\cup R_e)\setminus{Sup(U'')},$$ 
	 	where $U''=(A\setminus{Def_A(U)}, R_a\setminus{Def_{R_a}(U)} \footnote{We modify $R_a$ presented in the original paper \cite{cayrol2018argumentation} as $R_a\setminus{Def_{R_a}(U)}$, which coincides with all original semantics.}, R_e\setminus{Def_{R_e}(U)})$.
	 	
	 	\item The set of \emph{unacceptable elements} w.r.t. $U$ is defined as follows:
	 	$$UnAcc(U)\stackrel{\text{def}}{=} Def(U)\cup Def_s(U).$$
	 	
	 	\item The set of \emph{unactivable attacks} w.r.t. $U$ is defined as follows:
	 	$$UnAct(U)\stackrel{\text{def}}{=}\{\alpha\in R_a\mid \alpha\in UnAcc(U)\text{ or }
	 	s(\alpha)\cap UnAcc(U)\neq\varnothing\}.$$
	 \end{itemize}
	 
	 \begin{defn}[Acceptability]\label{def:rebaf-acc}
	 	An element $x\in A\cup R_a\cup R_e$ is \emph{acceptable} w.r.t.\ a structure $U$ of a REBAF $(A, R_a, R_e, s, t, P)$
	 	iff\\[2pt]
	 	\indent$(i)$ $x\in Sup(U)$, and\\
	 	\indent$(ii)$ every attack $\alpha\in R_a$ with $t(\alpha)=x$
	 	is unactivable, i.e.\ $\alpha\in UnAct(U)$.
	 \end{defn}
	 
	 \noindent
	 Let $ Acc(U)$ denote the set of all acceptable elements w.r.t.\ $U$.
	 
	 \begin{defn}[REBAF Semantics]\label{def:rebaf-semantics}
	 	A structure $U=(S,\Gamma,\Delta)$ is:
	 	\begin{enumerate}
	 		\item \textbf{self-supporting} iff $(S\cup\Gamma\cup\Delta)\subseteq Sup(U)$;
	 		\item \textbf{conflict-free} iff
	 		$X\cap Def_Y(U)=\varnothing$ for every
	 		$(X,Y)\in\{(S,A),(\Gamma,R_a)$, $(\Delta,R_e)\}$;
	 		\item \textbf{admissible} iff it is conflict-free and
	 		$S\cup\Gamma\cup\Delta\subseteq Acc(U)$;
	 		\item \textbf{complete} iff it is conflict-free and
	 		$ Acc(U)=S\cup\Gamma\cup\Delta$;
	 		\item \textbf{preferred} iff it is a $\subseteq$-maximal admissible structure;
	 		\item \textbf{grounded} iff it is the $\subseteq$-minimal complete structure;
	 		\item \textbf{stable} iff $(S\cup\Gamma\cup\Delta)=(A\cup R_a\cup R_e)\setminus{UnAcc(U)}$.
	 	\end{enumerate}
	 \end{defn}
	 
	 \noindent
	 The usual relations among semantics hold: every complete structure is admissible,
	 every preferred structure is complete, and every stable structure is preferred
	 \cite{cayrol2018argumentation}.  Moreover, REBAF is a conservative generalisation of
	 Dung's DAF: when $R_e=\varnothing$ and $P=A\cup R_a$ (all elements are
	 prima-facie), the REBAF semantics coincide with those of an DAF (or a RAF when
	 higher-order attacks are present) \cite{cayrol2018argumentation,besnard2023generic}.
	 
	 \begin{rmk}
	 	In \cite{cayrol2020logical} a logical encoding of REBAF is proposed under the
	 	restrictions that (i) the source of every interaction is a \emph{single}
	 	argument ($|s(\alpha)|=1$) and (ii) no \emph{support cycle} exists.
	 	Under these assumptions, REBAF structures under each semantics correspond
	 	exactly to particular models of a first-order logical theory.
	 \end{rmk}
	
\subsection{Propositional Logic Systems}

In this subsection, we recall the basic notions of propositional logics ($\mathcal{PL}$s) that underpin our work. For an extensive treatment of many-valued and fuzzy propositional logics, the reader is referred to \cite{klir1995fuzzy,Hajek1998,Klement2000Triangular,bergmann2008introduction,belohlavek2017fuzzy}.

Throughout the paper, we employ the following meta-linguistic shorthand: the expression $A \Longleftrightarrow B$ is used as an abbreviation for ``$A$ if and only if $B$'', while $B \Longrightarrow A$ stands for ``if $B$, then $A$'', where $A$ and $B$ are meta-level propositions.

The formal language of any propositional logic $\mathcal{PL}$ is built from the ingredients listed below.

\begin{enumerate}
	\item \textbf{Alphabet} --- it consists of:
	\begin{itemize}
		\item a countable collection of \textit{propositional constants}, denoted $\operatorname{Con}$, which includes the distinguished constants $\bot$ (falsum) and $\top$ (verum);
		\item a countable collection of \textit{propositional variables}, written as $\operatorname{Var} = \{p_1, p_2, p_3, \dots\}$;
		\item the set of \textit{logical connectives} $\{\neg, \wedge, \vee, \rightarrow, \leftrightarrow\}$, where $\neg$ is unary (negation) and the remaining four are binary (conjunction, disjunction, implication, and equivalence, respectively);
		\item auxiliary symbols: left and right parentheses $( $ and $)$.
	\end{itemize}
	
	\item \textbf{Formation rules} --- the set $\mathcal{F}_{\mathcal{PL}}$ of \textit{well-formed formulas} (wffs) is defined inductively as the smallest set satisfying:
	\begin{itemize}
		\item every element of $\operatorname{Con} \cup \operatorname{Var}$ belongs to $\mathcal{F}_{\mathcal{PL}}$;
		\item whenever $\varphi \in \mathcal{F}_{\mathcal{PL}}$, then $(\neg \varphi) \in \mathcal{F}_{\mathcal{PL}}$;
		\item whenever $\varphi, \psi \in \mathcal{F}_{\mathcal{PL}}$ and $\circ \in \{\wedge, \vee, \rightarrow, \leftrightarrow\}$, then $(\varphi \circ \psi) \in \mathcal{F}_{\mathcal{PL}}$.
	\end{itemize}
	
	\item \textbf{Notational conventions}:
	\begin{itemize}
		\item outermost parentheses may be omitted when no ambiguity arises;
		\item the precedence of connectives, from strongest to weakest, is: $\neg$, then $\wedge$, then $\vee$, then $\rightarrow$, and finally $\leftrightarrow$.
	\end{itemize}
\end{enumerate}

In what follows, we confine ourselves to truth-functional propositional logics whose truth domains $D$ are numerical and satisfy $\{0,1\} \subseteq D \subseteq [0,1]$. The semantics of any such logic is completely determined by the truth functions assigned to its connectives. We use the notation $\|\cdot\|$ ambiguously for both a truth assignment over $\operatorname{Con} \cup \operatorname{Var}$ and its canonical extension to all formulas. For every logic $\mathcal{PL}$ and every assignment $\|\cdot\|$, we impose $\|\bot\| = 0$ and $\|\top\| = 1$. Equivalence is defined in the standard way, namely
\[
a \leftrightarrow b \stackrel{\text{def}}{=} (a \rightarrow b) \wedge (b \rightarrow a)
\]
for arbitrary formulas $a,b$.

We begin by recalling the three-valued {\L}ukasiewicz logic $\mathcal{PL}_3^L$ \cite{Lukasiewicz1970ThreeValued}, whose truth values are taken from the set $\{0, \frac{1}{2}, 1\}$. The interpretations of $\neg$, $\wedge$ and $\vee$ are given pointwise as follows, for any formulas $a,b$:
\begin{itemize}
	\item $\|\neg a\| = 1 - \|a\|$;
	\item $\|a \wedge b\| = \min\{\|a\|, \|b\|\}$;
	\item $\|a \vee b\| = \max\{\|a\|, \|b\|\}$.
\end{itemize}

The implication $\rightarrow$ is evaluated via the truth function $\Rightarrow$, so that $\|a \rightarrow b\| = \|a\| \Rightarrow \|b\|$. The table of $\Rightarrow$ is shown in Table~\ref{Rightarrow}.

\begin{table}[htbp]
	\centering
	\begin{tabular}{c | c c c}
		$\Rightarrow$ & 0 & $\frac12$ & 1 \\ 
		\hline
		0 & 1 & 1 & 1 \\
		$\frac12$ & $\frac12$ & 1 & 1 \\
		1 & 0 & $\frac12$ & 1
	\end{tabular}
	\caption{Truth table for $\Rightarrow$ in $\mathcal{PL}_3^L$.}
	\label{Rightarrow}
\end{table}

From the definition of equivalence, we have
\[
\|a \leftrightarrow b\| 
= \min\{\|a \rightarrow b\|, \|b \rightarrow a\|\} 
= \min\{\|a\| \Rightarrow \|b\|, \|b\| \Rightarrow \|a\|\}.
\]
The resulting truth function $\Leftrightarrow$ is presented in Table~\ref{Leftrightarrow}.

\begin{table}[htbp]
	\centering
	\begin{tabular}{c | c c c}
		$\Leftrightarrow$ & 0 & $\frac12$ & 1 \\ 
		\hline
		0 & 1 & $\frac12$ & 0 \\
		$\frac12$ & $\frac12$ & 1 & $\frac12$ \\
		1 & 0 & $\frac12$ & 1
	\end{tabular}
	\caption{Truth table for $\Leftrightarrow$ in $\mathcal{PL}_3^L$.}
	\label{Leftrightarrow}
\end{table}

We now shift attention to fuzzy propositional logics with truth values in $[0,1]$ \cite{klir1995fuzzy,Hajek1998,Klement2000Triangular}. In these systems, the connectives $\neg$, $\wedge$ and $\rightarrow$ are customarily interpreted by means of a negation function, a triangular norm (t-norm), and its residuated implication (R-implication), respectively.

A negation $N$ is called the \emph{standard negation} whenever $N(m) = 1 - m$ for every $m \in [0,1]$. Below we list the three most important continuous t-norms together with their associated R-implications; for all $m,n \in [0,1]$:

\begin{itemize}
	\item G{\"o}del t-norm: $T_G(m, n) = \min\{m, n\}$; its R-implication is
	\[
	I_G(m, n) = 
	\begin{cases}
		1 & m \leq n, \\
		n & m > n.
	\end{cases}
	\]
	\item {\L}ukasiewicz t-norm: $T_L(m, n) = \max\{0, m + n - 1\}$; its R-implication is
	\[
	I_L(m, n) = \min\{1 - m + n, 1\}.
	\]
	\item Product t-norm: $T_P(m, n) = m \cdot n$; its R-implication is
	\[
	I_P(m, n) = 
	\begin{cases}
		1 & m \leq n, \\
		\frac{n}{m} & m > n.
	\end{cases}
	\]
\end{itemize}

A t-norm $\otimes$ is termed \emph{zero-divisor-free} if, for all $x, y \in (0,1]$, we have $x \otimes y > 0$; equivalently, $x \otimes y = 0$ entails $x = 0$ or $y = 0$.

For any R-implication $I$, the equivalence $I(m,n) = I(n,m) = 1$ holds precisely when $m = n$. Consequently, $\|a \leftrightarrow b\| = 1$ iff $\|a\| = \|b\|$. Moreover, we adopt the usual De Morgan definition of disjunction,
\[
a \vee b = \neg(\neg a \wedge \neg b),
\]
which is sound for all $[0,1]$-valued logics considered here.

Finally, we fix the following notation for the remainder of the paper:
\begin{itemize}
	\item $\mathcal{PL}_{[0,1]}^G$ stands for the $[0,1]$-valued logic with the standard negation, the G{\"o}del t-norm $T_G$, and its R-implication $I_G$;
	\item $\mathcal{PL}_{[0,1]}^P$ stands for the $[0,1]$-valued logic with the standard negation, the Product t-norm $T_P$, and its R-implication $I_P$;
	\item $\mathcal{PL}_{[0,1]}^L$ stands for the $[0,1]$-valued logic with the standard negation, the {\L}ukasiewicz t-norm $T_L$, and its R-implication $I_L$.
\end{itemize}
	\section{Syntax and Basic Semantics of EHSAFs}
    	\subsection{Syntax of EHSAFs}
    
    	\begin{defn}\label{pEHSAF}
    		A preparatory evidential-based higher-order set argumentation framework ($pre\text{-}EHSAF$) is a 5-tuple $(\mathbb{A}, \mathbb{R}_a, \mathbb{R}_e, \mathbb{U}, \mathbb{P})$, where $\mathbb{A}$ is a set of arguments, $\mathbb{R}_a$ is named the attack relation, $\mathbb{R}_e$ is named the evidential support relation, $\mathbb{R}_a\cup \mathbb{R}_e \subseteq (2^{\mathbb{A} \cup \mathbb{R}_a\cup \mathbb{R}_e} \setminus\{\varnothing\}) \times (\mathbb{A} \cup \mathbb{R}_a\cup \mathbb{R}_e)$, the finite universal set $\mathbb{U}=\mathbb{A} \cup \mathbb{R}_a\cup \mathbb{R}_e$ and $\mathbb{P}\subseteq \mathbb{U}$ called the set of prima-facie elements.
    	\end{defn}
    	
    	\begin{defn}\label{EHSAF}
    		Given a $pre\text{-}EHSAF = (\mathbb{A}, \mathbb{R}_a, \mathbb{R}_e, \mathbb{U}, \mathbb{P})$, an \emph{evidential-based higher-order set argumentation framework} (EHSAF)  is a 5-tuple $EHSAF = (\mathbf{A}, \mathbf{R_a}, \mathbf{R_e}, \mathbf{U}, \mathbf{P})$, where the argument set $\mathbf{A}=\mathbb{A}\cup\{\bot\}\cup\{\top\}$, the attack relation $\mathbf{R_a}=\mathbb{R}_a\cup \{(\{\bot\}, \beta)\mid \beta\in \mathbb{U} \text{ and } \nexists \alpha\in \mathbb{U}: (\alpha, \beta)\in \mathbb{R}_a\}$, the support relation $\mathbf{R_e}=\mathbb{R}_e\cup \{(\{\bot\}, \beta)\mid \beta\in \mathbb{U}\setminus \mathbb{P} \text{ and } \nexists \alpha\in \mathbb{U}: (\alpha, \beta)\in \mathbb{R}_e\}\cup \{(\{\top\}, \beta)\mid \beta\in \mathbb{P} \}$, the universal set $\mathbf{U}=\mathbf{A}\cup \mathbf{R_a}\cup \mathbf{R_e}$ and the prima-facie set $\mathbf{P}=\mathbb{P}\cup\{\top\}\cup(\mathbf{R_a}\setminus \mathbb{R}_a)\cup(\mathbf{R_e}\setminus \mathbb{R}_e)$.
    	\end{defn}
    	
    	To distinguish attacks from evidential supports in the following, we denote an attack $(\alpha, \beta)\in\mathbf{R_a}$ either as $(\alpha, \beta)_a$ or $\mathbf{k}_\alpha^\beta$, and an evidential support $(\gamma, \delta)\in\mathbf{R_e}$ either as $(\gamma, \delta)_e$ or $\mathbf{t}_\gamma^\delta$. Here, $\alpha$ and $\gamma$ are the \emph{sources} of the attack and support, respectively, while $\beta$ and $\delta$ denote their corresponding \emph{targets}.
    	
    	\subsection{Basic Semantics of EHSAFs}
    	 \begin{defn}\label{semantics}
    		Let $\mathcal{EHSAF}$ be the set of all EHSAFs and $\mathcal{LAB}$ be the set of all labellings of all EHSAFs over $[0,1]$.
    		A labelling semantics for $\mathcal{EHSAF}$ is a function
    		\[
    		\mathfrak{L}\mathfrak{S}: \mathcal{EHSAF} \to 2^{\mathcal{LAB}}
    		\]
    		that maps each $EHSAF \in \mathcal{EHSAF}$ to a set of valid labellings of the $EHSAF$, denoted by $\mathfrak{L}\mathfrak{S}(EHSAF)$. 
    		For any labelling $\|\cdot\|: \mathbf{U}\to D$ $(D\subseteq[0,1])$ of an EHSAF, always require that $\|\bot\|=0$, $\|\top\|=1$ and $\|\alpha\|=1$ for each $\alpha \in (\mathbf{R_e}\cup\mathbf{R_e})\setminus(R_a\cup R_e)$. Each labelling $\|\cdot\| \in \mathfrak{L}\mathfrak{S}(EHSAF)$ is called a \emph{model} of the EHSAF under the semantics $\mathfrak{L}\mathfrak{S}$. $\|\cdot\| \in \mathfrak{L}\mathfrak{S}(EHSAF)$ is also denoted by $\|\cdot\| \models_\mathfrak{LS}EHSAF$.
    	\end{defn}
    		\begin{defn}[Equational System for EHSAF]
    		Let $\mathcal{F}=(\mathbf{A},\mathbf{R}_{a},\mathbf{R}_{e},\mathbf{U},\mathbf{P})$ be an EHSAF. 
    		An \emph{equational system} $eq$ over $D\subseteq[0,1]$ for $\mathcal{F}$ is defined by the labelling requirements $\|\bot\|=0$, $\|\top\|=1$ and $\|\alpha\|=1$ for each $\alpha \in (\mathbf{R_e}\cup\mathbf{R_e})\setminus(R_a\cup R_e)$, and for each element $\beta\in \mathbb{U}$ an associated equation of the form
    		\[
    		\|\beta\| = h_{\beta}\bigl(\|x_{1}\|,\|x_{2}\|,\ldots,\|x_{m_{\beta}}\|\bigr),\qquad \|\beta\|\in D,
    		\]
    		where $\{x_1,\ldots,x_{m_{\beta}}\}\subseteq \mathbf{U}$ is the set of all elements that occur in the source sets of any attack or support targeting $\beta$, as well as the attack/support names themselves, and $h_{\beta}:D^{m_{\beta}}\to D$ is a function.
    		If an assignment $\|\cdot\|:\mathbf{U}\to D$ satisfies all equations in $eq$, we denote this by $\|\cdot\|\models_{eq}\mathcal{F}$.
    	\end{defn}
    	
    	\begin{defn}[Equational Semantics for EHSAF]
    		An \emph{equational semantics} for $\mathcal{EHSAF}$ is a labelling semantics
    		\[
    		\mathfrak{L}\mathfrak{S}_{Eq}:\mathcal{EHSAF}\to 2^{\mathcal{LAB}}
    		\]
    		such that for every EHSAF $\mathcal{F}=(\mathbf{A},\mathbf{R}_{a},\mathbf{R}_{e},\mathbf{U},\mathbf{P})$ and its associated equational system $eq$,
    		\[
    		\mathfrak{L}\mathfrak{S}_{Eq}(\mathcal{F}) = \{\|\cdot\| \mid \|\cdot\|\models_{eq}\mathcal{F}\}.
    		\]
    	\end{defn}  	
    	
    \subsubsection{Labelling-Based Semantics of EHSAFs}
    \begin{defn}\label{comlab}
    	A \emph{adjacent complete labelling} for an $EHSAF = (\mathbf{A}, \mathbf{R_a}, \mathbf{R_e}, \mathbf{U}, \mathbf{P})$ is a function
    	$\|\cdot\| : \mathbf{U}\to\{0, 1, \frac{1}{2}\}$ such that it satisfies the labelling requirements in Definition \ref{semantics} and for each $\beta\in \mathbb{U}$:
    	\begin{itemize}
    		\item $\|\beta\|=1$ iff [$\forall S_i(\mathbf{k}_{S_i}^{\beta}\in\mathbf{R_a}) \exists b_{ij} \in S_i: \|b_{ij}\| = 0$ or $\|\mathbf{k}_{S_i}^{\beta}\| = 0$] and [$\exists T_i(\mathbf{t}_{T_i}^{\beta}\in\mathbf{R_e}) \forall c_{ij} \in T_i: \|c_{ij}\| = 1$ and $\|\mathbf{t}_{T_i}^{\beta}\| = 1$]
    		\item $\|\beta\|=0$ iff [$\exists S_i(\mathbf{k}_{S_i}^{\beta}\in\mathbf{R_a}) \forall b_{ij} \in S_i: \|b_{ij}\| = 1$ and $\|\mathbf{k}_{S_i}^{\beta}\| = 1$] or [$\forall T_i(\mathbf{t}_{T_i}^{\beta}\in\mathbf{R_e}) \exists c_{ij} \in T_i: \|c_{ij}\| = 0$ or $\|\mathbf{t}_{T_i}^{\beta}\| = 0$]
    		\item $\|\beta\|=\frac{1}{2}$ iff otherwise.
    	\end{itemize}		
    	The adjacent complete labelling semantics is a labelling semantics $\mathfrak{LS}_{ac}$, such that for any EHSAF, 
    	\begin{equation*}
    		\mathfrak{LS}_{ac}(EHSAF)=\{\|\cdot\|\mid\|\cdot\| \text{ is a adjacent complete labelling of the }EHSAF\}.
    	\end{equation*}
    \end{defn}
    
    \begin{defn}[Adjacent Stable, Preferred and Grounded Semantics]\label{staetc}
    	Let $\mathcal{F}=(\mathbf{A},\mathbf{R}_{a},\mathbf{R}_{e},\mathbf{U},\mathbf{P})$ be an EHSAF satisfying the semantic requirements in Definition \ref{semantics} for any labelling.
    	\begin{enumerate}
    		\item A labelling $\|\cdot\|:\mathbf{U}\to\{0,1\}$ is an \emph{adjacent stable labelling} of $\mathcal{F}$ if it is an adjacent complete labelling of $\mathcal{F}$ and no element of $\mathbf{U}$ is assigned the value $\frac{1}{2}$. 
    		We denote by $\mathfrak{L}\mathfrak{S}_{as}(\mathcal{F})$ the set of all adjacent stable labellings of $\mathcal{F}$.
    		
    		\item An adjacent complete labelling $\|\cdot\|$ of $\mathcal{F}$ is an \emph{adjacent preferred labelling} if it is maximal with respect to the inclusion order of accepted elements: there does not exist another adjacent complete labelling $\|\cdot\|'$ such that 
    		\[
    		\{x\in\mathbf{U}\mid \|x\|=1\} \subsetneq \{x\in\mathbf{U}\mid \|x\|'=1\}.
    		\]
    		We denote by $\mathfrak{L}\mathfrak{S}_{ap}(\mathcal{F})$ the set of all adjacent preferred labellings of $\mathcal{F}$.
    		
    		\item The \emph{adjacent grounded labelling} of $\mathcal{F}$ is an adjacent complete labelling $\|\cdot\|$ that is minimal with respect to the inclusion order of accepted elements: there does not exist another adjacent complete labelling $\|\cdot\|'$ such that 
    		\[
    		\{x\in\mathbf{U}\mid \|x\|'=1\} \subsetneq \{x\in\mathbf{U}\mid \|x\|=1\}.
    		\]
    		We denote by $\mathfrak{L}\mathfrak{S}_{ag}(\mathcal{F})$ the set containing the adjacent grounded labelling.
    	\end{enumerate}
    \end{defn}
      
    \begin{defn}[Three-Valued Equational Semantics for EHSAF]\label{eq3}
    	Let \(\mathcal{F}=(\mathbf{A},\mathbf{R}_{a},\mathbf{R}_{e},\mathbf{U},\mathbf{P})\) be an EHSAF. 
    	For each \(\beta\in \mathbb{U}\), define the following two quantities:
    	\[
    	A(\beta) \;\stackrel{\text{def}}{=}\; 
    	\min_{\mathbf{k}_{S}^{\beta}\in\mathbf{R}_{a}}
    	\left(
    	\max\{1-\|\mathbf{k}_{S}^{\beta}\|,\; \max_{s\in S}(1-\|s\|)\}
    	\right),
    	\]
    	\[
    	B(\beta) \;\stackrel{\text{def}}{=}\; 
    	\max_{\mathbf{t}_{T}^{\beta}\in\mathbf{R}_{e}}
    	\left(
    	\min\{\|\mathbf{t}_{T}^{\beta}\|,\; \min_{t\in T}\|t\|\}
    	\right),
    	\]
    	where all \(\min\) and \(\max\) operations are taken over the three-element set \(\{0,\tfrac12,1\}\).
    	
    	The \emph{three-valued equational system} \(eq_{3}^{EH}\) for \(\mathcal{F}\) is given by the auxiliary-element conditions
    	\[
    	\|\bot\|=0,\qquad \|\top\|=1,\qquad \|\alpha\|=1 (\forall\alpha \in (\mathbf{R_e}\cup\mathbf{R_e})\setminus(R_a\cup R_e))
    	\]
    	and for every \(\beta\in \mathbb{U}\),
    	\begin{equation}\label{star1}
    			\|\beta\| = \min\{A(\beta),\,B(\beta)\}. 
    	\end{equation}
    	An assignment \(\|\cdot\|:\mathbf{U}\to\{0,\tfrac12,1\}\) satisfying all equations is called a solution of \(eq_{3}^{EH}\), denoted \(\|\cdot\|\models_{eq_{3}^{EH}}\mathcal{F}\).
    	
    	The \emph{three-valued equational semantics} for \(\mathcal{EHSAF}\) induced by \(eq_{3}^{EH}\) is the function
    	\[
    	\mathfrak{L}\mathfrak{S}_{Eq_{3}^{EH}}:\mathcal{EHSAF}\to 2^{\mathcal{LAB}},
    	\]
    	defined by
    	\[
    	\mathfrak{L}\mathfrak{S}_{Eq_{3}^{EH}}(\mathcal{F}) = \{\|\cdot\| \mid \|\cdot\|\models_{eq_{3}^{EH}}\mathcal{F}\}.
    	\]
    \end{defn}
    
    \begin{thm}[Equivalence of Three-Valued Equational and Adjacent Complete Semantics]\label{eqofeqandadcom}
    	For every EHSAF \(\mathcal{F}\),
    	\[
    	\mathfrak{L}\mathfrak{S}_{Eq_{3}^{EH}}(\mathcal{F}) = \mathfrak{L}\mathfrak{S}_{ac}(\mathcal{F}).
    	\]
    \end{thm}
    
    \begin{proof}
    	Fix an EHSAF \(\mathcal{F}=(\mathbf{A},\mathbf{R}_{a},\mathbf{R}_{e},\mathbf{U},\mathbf{P})\) and an assignment \(\|\cdot\|:\mathbf{U}\to\{0,\tfrac12,1\}\). For each \(\beta\in \mathbb{U}\), let
    	\[
    	\mathcal{K}_\beta \stackrel{\text{def}}{=} \{\mathbf{k}_{S}^{\beta}\in\mathbf{R}_{a}\}, \qquad 
    	\mathcal{T}_\beta \stackrel{\text{def}}{=} \{\mathbf{t}_{T}^{\beta}\in\mathbf{R}_{e}\}.
    	\]
    	Define for each attack \(k=\mathbf{k}_{S}^{\beta}\in\mathcal{K}_\beta\):
    	\[
    	v_k \stackrel{\text{def}}{=} \max\{1-\|k\|,\; \max_{s\in S}(1-\|s\|)\}.
    	\]
    	Then \(A(\beta)=\min_{k\in\mathcal{K}_\beta} v_k\). Since values are in \(\{0,\tfrac12,1\}\), we have the following equivalences:
    	\[
    	v_k=0 \;\Longleftrightarrow\; \|k\|=1 \;\text{and}\; \forall s\in S,\|s\|=1, 
    	\]
    	\[
    	v_k=1 \;\Longleftrightarrow\; \|k\|=0 \;\text{or}\; \exists s\in S,\|s\|=0. 
    	\]
    	Consequently,
    	\begin{equation}\label{tag3}
    		A(\beta)=1 \;\Longleftrightarrow\; \forall k\in\mathcal{K}_\beta,\; v_k=1 
    		\;\Longleftrightarrow\; \forall k\in\mathcal{K}_\beta,\; \bigl(\|k\|=0 \;\text{or}\; \exists s\in S,\|s\|=0\bigr),
    	\end{equation}
    	\begin{equation}\label{tag4}
    		A(\beta)=0 \;\Longleftrightarrow\; \exists k\in\mathcal{K}_\beta,\; v_k=0 
    		\;\Longleftrightarrow\; \exists k\in\mathcal{K}_\beta,\; \bigl(\|k\|=1 \;\text{and}\; \forall s\in S,\|s\|=1\bigr).
    	\end{equation}
    	Moreover, because \(A(\beta)\in\{0,\tfrac12,1\}\), the remaining case gives
    	\[
    	A(\beta)=\tfrac12 \;\Longleftrightarrow\; \text{neither (\ref{tag3}) nor (\ref{tag4}) holds}. 
    	\]
    	
    	Similarly, for each support \(t=\mathbf{t}_{T}^{\beta}\in\mathcal{T}_\beta\), define
    	\[
    	w_t \stackrel{\text{def}}{=} \min\{\|t\|,\; \min_{r\in T}\|r\|\}.
    	\]
    	Then \(B(\beta)=\max_{t\in\mathcal{T}_\beta} w_t\), and we have
    	\[
    	w_t=1 \;\Longleftrightarrow\; \|t\|=1 \;\text{and}\; \forall r\in T,\|r\|=1, 
    	\]
    	\[
    	w_t=0 \;\Longleftrightarrow\; \|t\|=0 \;\text{or}\; \exists r\in T,\|r\|=0. 
    	\]
    	Thus
    	\begin{equation}\label{tag8}
    		B(\beta)=1 \;\Longleftrightarrow\; \exists t\in\mathcal{T}_\beta,\; w_t=1 
    		\;\Longleftrightarrow\; \exists t\in\mathcal{T}_\beta,\; \bigl(\|t\|=1 \;\text{and}\; \forall r\in T,\|r\|=1\bigr),
    	\end{equation}
    	\begin{equation}\label{tag9}
    		B(\beta)=0 \;\Longleftrightarrow\; \forall t\in\mathcal{T}_\beta,\; w_t=0 
    		\;\Longleftrightarrow\; \forall t\in\mathcal{T}_\beta,\; \bigl(\|t\|=0 \;\text{or}\; \exists r\in T,\|r\|=0\bigr). 
    	\end{equation}
    	Again, \(B(\beta)=\tfrac12\) exactly when neither (\ref{tag8}) nor (\ref{tag9}) is true.
    	
    	Now Equation \eqref{star1} states that \(\|\beta\| = \min\{A(\beta),B(\beta)\}\). We analyze the three possible values of \(\|\beta\|\):
    	
    	\begin{itemize}
    		\item \textbf{Case \(\|\beta\|=1\):} 
    		\[
    		\|\beta\|=1 \;\Longleftrightarrow\; A(\beta)=1 \;\text{and}\; B(\beta)=1 \;\Longleftrightarrow\; \text{(\ref{tag3}) and (\ref{tag8})}.
    		\]
    		By (\ref{tag3}) and (\ref{tag8}), this is equivalent to $\forall k\in\mathcal{K}_\beta \bigl(\|k\|=0 \;\text{or}\; \exists s\in S,\|s\|=0\bigr)$ and $\exists t\in\mathcal{T}_\beta,\; \bigl(\|t\|=1 \;\text{and}\; \forall r\in T,\|r\|=1\bigr)$,
    	    which is precisely the adjacent complete labelling condition for \(\|\beta\|=1\) (Definition~\ref{comlab}, first bullet).
    		
    		\item \textbf{Case \(\|\beta\|=0\):}
    		\[
    		\|\beta\|=0 \;\Longleftrightarrow\; A(\beta)=0 \;\text{or}\; B(\beta)=0 \;\Longleftrightarrow\; \text{(\ref{tag4}) or (\ref{tag9})}.
    		\]
    		This is equivalent to
    		$\exists k\in\mathcal{K}_\beta \bigl(\|k\|=1 \;\text{and}\; \forall s\in S,\|s\|=1\bigr)$ or $\forall t\in\mathcal{T}_\beta \bigl(\|t\|=0 \;\text{or}\; \exists r\in T,\|r\|=0\bigr)$,
    		which matches the adjacent complete labelling condition for \(\|\beta\|=0\) (second bullet).
    		
    		\item \textbf{Case \(\|\beta\|=\tfrac12\):}
    		Since the only possible values are \(0,\tfrac12,1\), we have
    		\[
    		\|\beta\|=\tfrac12 \;\Longleftrightarrow\; \|\beta\|\neq1 \;\text{and}\; \|\beta\|\neq0.
    		\]
    		Using the above equivalences, this reduces to the statement that neither the acceptance condition (all attacks ineffective and some support effective) nor the rejection condition (some attack effective or all supports ineffective) holds. This is exactly the “otherwise” clause of Definition~\ref{comlab}, which assigns \(\tfrac12\).
    	\end{itemize}
    	
    	Thus, for every \(\beta\in \mathbb{U}\),
    	$\|\beta\| = \min\{A(\beta),B(\beta)\}$ iff $\|\cdot\|$ satisfies the adjacent complete labelling condition for $\beta$.
    	Any assignment for auxiliary elements in $\mathbf{U}$ is trivially consistent. Therefore \(\|\cdot\|\models_{eq_{3}^{EH}}\mathcal{F}\) iff \(\|\cdot\|\in\mathfrak{L}\mathfrak{S}_{ac}(\mathcal{F})\), proving the desired equality.
    \end{proof}
    
    	\subsubsection{Extension-Based Semantics of EHSAFs}
    
    We draw on the definition method for REBAF semantics proposed by Cayrol et al. in \cite{cayrol2018argumentation}, and omit the introduction of the notion of structure here.
    \begin{defn}[Fundamental definitions for extension-based semantics]\label{defnI}
    	Let $EHSAF=(\mathbf{A}, \mathbf{R_a}, \mathbf{R_e}, \mathbf{U}, \mathbf{P})$, $\mathbf{S}\subseteq\mathbf{U}$, $\mathbf{S}\cap\mathbf{R_a}=\mathbf{S_a}$ and $\mathbf{S}\cap\mathbf{R_e}=\mathbf{S_e}$.
    	\begin{itemize}
    		\item The set of \emph{attacked-defeated elements} w.r.t. $\mathbf{S}$ are defined as follows:
    		$$\mathbf{Def}_a(\mathbf{S})\stackrel{\text{def}}{=}\{x\in \mathbf{U} \mid \exists \mathbf{k}_G^x\in\mathbf{S_a}: G\subseteq \mathbf{S}\}.$$
    		
    		\item The set of \emph{evidentially supported elements} w.r.t. $\mathbf{S}$ is defined as follows:
    		\[
    		\mathbf{Sup}(\mathbf{S})\stackrel{\text{def}}{=} \mathbf{P}\cup
    		\{x\mid \exists\mathbf{t}_H^x\in \mathbf{S_e}: \mathbf{t}_H^x \in \mathbf{Sup}(\mathbf{S}\setminus \{x\}) \text{ and } H\subseteq \mathbf{S}\cap \mathbf{Sup}(\mathbf{S}\setminus \{x\})\}.
    		\]
    		
    		\item The set of \emph{unsupportable-defeated elements} w.r.t. $\mathbf{S}$ is defined as follows:
    		$$\mathbf{Def}_s(\mathbf{S})\stackrel{\text{def}}{=}\mathbf{U}\setminus{\mathbf{Sup}(\mathbf{U}\setminus \mathbf{Def}_a(\mathbf{S}))}.$$
    		
    		\item The set of \emph{defeated elements} w.r.t. $\mathbf{S}$ is defined as follows:
    		$$\mathbf{Def}(\mathbf{S})\stackrel{\text{def}}{=} \mathbf{Def}_a(\mathbf{S})\cup \mathbf{Def}_s(\mathbf{S}).$$
    		
    		\item The set of \emph{acceptable elements} w.r.t. $\mathbf{S}$ is defined as follows:
    		$$\mathbf{Acc}(\mathbf{S})\stackrel{\text{def}}{=}\{x\mid x\in \mathbf{Sup}(\mathbf{S}) \text{ and } \forall\mathbf{k}_G^x\in\mathbf{R_a}: \mathbf{k}_G^x\in \mathbf{Def}(\mathbf{S})\text{ or }
    		G\cap \mathbf{Def}(\mathbf{S})\neq\varnothing\}.$$
    	\end{itemize}
    \end{defn}		
The definition of $\mathbf{Sup}(\mathbf{S})$ in Definition~\ref{defnI} is inductive and, because $\mathbf{U}$ is finite, it gives rise to a natural finite derivation structure for every element $x\in \mathbf{Sup}(\mathbf{S})$. We call this structure a \emph{support tree} (or \emph{support derivation tree}) for $x$ with respect to $\mathbf{S}$.  
Intuitively, a support tree for $x$ records one possible way in which $x$ can be derived from prima‐facie elements via a chain of evidential supports. The root of the tree is labelled by the target element $x$ itself; importantly, the root label is \emph{not} required to belong to the given set $\mathbf{S}$. All other nodes in a support tree for $x$---the supports and the source arguments of supports in the support tree---must lie in $\mathbf{S}$, as required by the recursive clause of Definition~\ref{defnI}.  The leaves of the tree are always prima‐facie elements, which include the auxiliary constant $\top$ (as well as other added prima‐facie elements from Definition~\ref{EHSAF}).
We now give a formal definition as follows.
\begin{defn}[Support Tree]\label{suptree}
	Let $\mathcal{F}=(\mathbf{A},\mathbf{R}_{a},\mathbf{R}_{e},\mathbf{U},\mathbf{P})$ be an EHSAF and let $x\in \mathbf{Sup}(\mathbf{S})$ for some $\mathbf{S}\subseteq \mathbf{U}$. A \emph{support tree} for $x$ with respect to $\mathbf{S}$ is a finite rooted tree $\mathcal{T}$ satisfying:
	\begin{enumerate}
		\item The root of $\mathcal{T}$ is labelled by $x$.
		\item For every node $v$ of $\mathcal{T}$ (including the root), let $y\in \mathbf{U}$ be its label. Then either:
		\begin{itemize}
			\item $y\in \mathbf{P}$ and $v$ has no children (i.e., $v$ is a leaf); or
			\item there exists a support $\mathbf{t}_{H}^{y}\in \mathbf{S}_{e}$ such that the children of $v$ are labelled by the elements of $H$ and by the support name $\mathbf{t}_{H}^{y}$ itself. In this case, we require that
			\[
			\mathbf{t}_{H}^{y}\in \mathbf{Sup}(\mathbf{S}\setminus\{y\}) \quad\text{and}\quad H\subseteq \mathbf{Sup}(\mathbf{S}\setminus\{y\}).
			\]
		\end{itemize}
		\item The tree is closed with respect to $\mathbf{S}$ except for the root: every node label \emph{other than the root} belongs to $\mathbf{S}$.
	\end{enumerate}
\end{defn}
In words, a support tree witnesses the derivation of $x$ from prima‐facie elements through a chain of supports. The leaves are always prima‐facie elements from $\mathbf{P}$ (which includes the auxiliary element $\top$). The internal nodes (support names and their source arguments) are all required to belong to $\mathbf{S}$, while the root $x$ itself need not be in $\mathbf{S}$. Additionally, we say that a support tree $\mathcal{T}$ for $x$ \emph{contains} an element $z$ if $z$ appears as a label of some node in $\mathcal{T}$.    
    
 \paragraph{Monotonicity and Least Fixed Point of $\mathbf{Sup}$.}

The operator $\mathbf{Sup}:2^{\mathbf{U}}\to 2^{\mathbf{U}}$ defined in Definition~\ref{defnI} enjoys the following fundamental properties, which will be used repeatedly in the sequel.

\begin{lem}[Properties of the $\mathbf{Sup}$ Operator]\label{eqmonofix}
	Let $\mathcal{F}=(\mathbf{A},\mathbf{R}_{a},\mathbf{R}_{e},\mathbf{U},\mathbf{P})$ be an EHSAF, and let $\mathbf{Sup}:2^{\mathbf{U}}\to 2^{\mathbf{U}}$ be as in Definition~\ref{defnI}. Then:
	\begin{enumerate}
		\item \emph{Support tree characterisation:} For any $\mathbf{S}\subseteq \mathbf{U}$ and any $x\in \mathbf{U}$,
		\[
		x\in \mathbf{Sup}(\mathbf{S}) \Longleftrightarrow \text{ there exists a finite support tree for }x\text{ with respect to }\mathbf{S},
		\]
		where the notion of support tree is given in Definition~\ref{suptree}. 
		\item \emph{Monotonicity:} For any $\mathbf{S}\subseteq \mathbf{T}\subseteq \mathbf{U}$, we have
		\[
		\mathbf{Sup}(\mathbf{S}) \subseteq \mathbf{Sup}(\mathbf{T}).
		\]
		\item \emph{$\subseteq$-continuity:} For every countable chain $\mathbf{S}_0\subseteq \mathbf{S}_1\subseteq \cdots \subseteq \mathbf{U}$,
		\[
		\mathbf{Sup}\left(\bigcup_{i\ge 0}\mathbf{S}_i\right)=\bigcup_{i\ge 0}\mathbf{Sup}(\mathbf{S}_i).
		\]
		Consequently, $\mathbf{Sup}$ is a monotone and continuous operator on the complete lattice $(2^{\mathbf{U}},\subseteq)$.
		\item \emph{Least fixed point:} $\mathbf{Sup}$ has a least fixed point, given by
		\[
		\mathsf{lfp}(\mathbf{Sup})=\bigcup_{i\ge 0} \mathbf{Sup}^i(\varnothing),
		\]
		where $\mathbf{Sup}^0(\varnothing)=\varnothing$ and $\mathbf{Sup}^{i+1}(\varnothing)=\mathbf{Sup}(\mathbf{Sup}^i(\varnothing))$. 
	\end{enumerate}
\end{lem}

\begin{proof}
	(1) We prove the equivalence by separate implications.
	
	($\Rightarrow$) It follows directly from Definition~\ref {suptree}.
	
	($\Leftarrow$) Suppose there exists a support tree $\mathcal{T}$ for $x$ with respect to $\mathbf{S}$. We prove $x\in \mathbf{Sup}(\mathbf{S})$ by induction on the height of $\mathcal{T}$.
	
	If the height is $0$, then $x$ is a leaf, and by the tree definition, $x\in\mathbf{P}\subseteq \mathbf{Sup}(\mathbf{S})$.
	
	If the height is greater than $0$, then by the tree definition there exists a support $\mathbf{t}_{H}^{x}\in \mathbf{S}_{e}$ such that the children of the root are the roots of support trees for the elements of $H$ and for $\mathbf{t}_{H}^{x}$ itself. Moreover, the tree definition directly requires that
	\[
	\mathbf{t}_{H}^{x}\in \mathbf{Sup}(\mathbf{S}\setminus\{x\}) \quad\text{and}\quad H\subseteq \mathbf{S}\cap \mathbf{Sup}(\mathbf{S}\setminus\{x\}).
	\]
	Applying the recursive clause of Definition~\ref{defnI} for $\mathbf{Sup}$ with this support $\mathbf{t}_{H}^{x}$ immediately yields $x\in \mathbf{Sup}(\mathbf{S})$. This completes the induction and the proof of the equivalence.
	
	(2) Monotonicity. We prove a stronger statement $\mathcal{P}(n)$: for any two sets $A\subseteq B\subseteq \mathbf{U}$ and any $y\in \mathbf{U}$, if $y\in \mathbf{Sup}(A)$ with a derivation of depth at most $n$, then $y\in \mathbf{Sup}(B)$. The depth of a derivation is defined as the number of recursive applications of the clause for non-prima-facie elements in Definition~\ref{defnI}. We proceed by induction on $n$.
	
	Base case $n=0$: then $y\in\mathbf{P}$ (since only prima-facie elements can be derived with depth 0), so $y\in \mathbf{Sup}(B)$.
	
	Inductive step: suppose $\mathcal{P}(n)$ holds for all $n<N$, and assume $y\in \mathbf{Sup}(A)$ with depth $N$. If $y\in\mathbf{P}$, done. Otherwise, there exists a support $\mathbf{t}_{H}^{y}\in A_\mathbf{e}$ such that
	\[
	\mathbf{t}_{H}^{y}\in \mathbf{Sup}(A\setminus\{y\}) \quad\text{and}\quad H\subseteq A\cap \mathbf{Sup}(A\setminus\{y\}).
	\]
	Both $\mathbf{t}_{H}^{y}$ and each $h\in H$ have derivations of depth strictly less than $N$ (since they appear in the premises of the rule used for $y$). Since $A\subseteq B$, we have $\mathbf{t}_{H}^{y}\in B_\mathbf{e}$ and $H\subseteq B$. Also, $A\setminus\{y\}\subseteq B\setminus\{y\}$. Applying the induction hypothesis $\mathcal{P}(<N)$ to the set pair $(A\setminus\{y\}, B\setminus\{y\})$, we get
	\[
	\mathbf{t}_{H}^{y}\in \mathbf{Sup}(B\setminus\{y\}) \quad\text{and}\quad H\subseteq \mathbf{Sup}(B\setminus\{y\}).
	\]
	Thus the recursive condition is satisfied for $y$ with respect to $B$, yielding $y\in \mathbf{Sup}(B)$. This completes the induction.
	
	The monotonicity statement is exactly $\mathcal{P}(n)$ for arbitrary $n$, so it holds for all $y\in \mathbf{Sup}(A)$.
	
	(3) For $\subseteq$-continuity, the direction $\bigcup_i \mathbf{Sup}(\mathbf{S}_i)\subseteq \mathbf{Sup}(\bigcup_i \mathbf{S}_i)$ follows from monotonicity. Conversely, suppose $x\in \mathbf{Sup}(\bigcup_i \mathbf{S}_i)$. By the forward direction of (1), $x$ has a finite support tree $\mathcal{T}$ with respect to $\bigcup_i \mathbf{S}_i$. The tree $\mathcal{T}$ contains only finitely many non-root node labels; each of these labels belongs to $\bigcup_i \mathbf{S}_i$. For each such label $z$, choose an index $i_z$ such that $z\in \mathbf{S}_{i_z}$. Since the chain is increasing and the number of labels is finite, let
	\[
	k = \max\{ i_z \mid z \text{ is a non-root label appearing in } \mathcal{T} \}.
	\]
	Then for every non-root label $z$, we have $z\in \mathbf{S}_{i_z}\subseteq \mathbf{S}_k$. Therefore all non-root nodes of $\mathcal{T}$ are labelled by elements of $\mathbf{S}_k$, so $\mathcal{T}$ is also a valid support tree with respect to $\mathbf{S}_k$. By the backward direction of (1) (existence of support tree implies membership in $\mathbf{Sup}$), we obtain $x\in \mathbf{Sup}(\mathbf{S}_k)$. Hence $x\in \bigcup_i \mathbf{Sup}(\mathbf{S}_i)$. This proves the reverse inclusion and thus the equality.
	
	(4) By the Knaster–Tarski theorem, every monotone operator on a complete lattice has a least fixed point, which is the join of its iterates from the bottom element. Since $\mathbf{U}$ is finite, the chain $\mathbf{Sup}^i(\varnothing)$ stabilises after at most $|\mathbf{U}|$ steps, yielding the least fixed point.
	
\end{proof}

This fixed-point characterisation justifies the recursive definition of $\mathbf{Sup}$ and ensures that the set $\mathbf{Sup}(\mathbf{S})$ is uniquely determined for every $\mathbf{S}\subseteq\mathbf{U}$. The support tree characterisation is essential for the proofs of Lemmas~\ref{lem1} and \ref{lem:pos}.    
    
    \begin{defn}[Extension-based Semantics]\label{extsem}
    	A set $\mathbf{S}\in \mathbf{U}$ is:
    	\begin{enumerate}
    		\item \textbf{self-supporting} iff $\mathbf{S}\subseteq \mathbf{Sup}(\mathbf{S})$;
    		\item \textbf{conflict-free} iff
    		$\mathbf{S}\cap \mathbf{Def}(\mathbf{S})=\varnothing$;
    		\item \textbf{admissible} iff it is conflict-free and
    		$\mathbf{S}\subseteq \mathbf{Acc}(\mathbf{S})$;
    		\item a \textbf{complete extension} iff it is conflict-free and
    		$ \mathbf{Acc}(\mathbf{S})=\mathbf{S}$;
    		\item a \textbf{preferred extension} iff it is a $\subseteq$-maximal admissible structure;
    		\item a \textbf{grounded extension} iff it is the $\subseteq$-minimal complete structure;
    		\item a \textbf{stable extension} iff $\mathbf{S}\cup{\mathbf{Def}(\mathbf{S})}=\mathbf{U}$.
    	\end{enumerate}
    \end{defn}
    \begin{lem}[Acceptable and defeated are disjoint for conflict-free sets]
    	\label{lem:acc-def-disjoint}
    	For any conflict-free set \(S\subseteq\mathbf{U}\) (i.e. \(S\cap \mathbf{Def}(S)=\varnothing\)), we have
    	\[
    	\mathbf{Acc}(S)\cap \mathbf{Def}(S)=\varnothing.
    	\]
    \end{lem}
    
    \begin{proof}
    	Assume, for contradiction, that there exists \(x\in \mathbf{Acc}(S)\cap \mathbf{Def}(S)\).
    	
    	\textit{Case 1:} \(x\in \mathbf{Def}_a(S)\). Then there exists an attack \(\mathbf{k}_G^x\in S_a\) with \(G\subseteq S\). Since \(S\) is conflict-free, we have \(S\cap \mathbf{Def}(S)=\varnothing\). Hence \(\mathbf{k}_G^x\notin \mathbf{Def}(S)\) and \(G\cap \mathbf{Def}(S)=\varnothing\) (because \(G\subseteq S\)). But \(x\in \mathbf{Acc}(S)\) requires that for every attack \(\mathbf{k}_G^x\), either \(\mathbf{k}_G^x\in \mathbf{Def}(S)\) or \(G\cap \mathbf{Def}(S)\neq\varnothing\). Contradiction.
    	
    	\textit{Case 2:} \(x\in \mathbf{Def}_s(S)\). By definition, \(\mathbf{Def}_s(S)=\mathbf{U}\setminus \mathbf{Sup}(\mathbf{T})\), where \(\mathbf{T}=\mathbf{U}\setminus \mathbf{Def}_a(S)\). Thus \(x\notin \mathbf{Sup}(\mathbf{T})\). Since \(S\) is conflict-free, we have \(S\cap \mathbf{Def}_a(S)=\varnothing\), so \(S\subseteq \mathbf{T}\). The operator \(\mathbf{Sup}\) is monotone with respect to set inclusion; hence \(\mathbf{Sup}(S)\subseteq \mathbf{Sup}(\mathbf{T})\). But \(x\in \mathbf{Acc}(S)\) implies \(x\in \mathbf{Sup}(S)\), so \(x\in \mathbf{Sup}(\mathbf{T})\), contradicting \(x\notin \mathbf{Sup}(\mathbf{T})\).
    	
    	Both cases lead to contradictions. Therefore no such \(x\) exists, and the lemma follows.
    \end{proof}

    \begin{defn}\label{extlab}
    	For an $EHSAF = (\mathbf{A}, \mathbf{R_a}, \mathbf{R_e}, \mathbf{U}, \mathbf{P})$, any complete extension $E$ of the EHSAF and $\forall \alpha\in \mathbf{U}$, an \emph{extension-based complete labelling} of the EHSAF w.r.t. $E$ is a function $L_E:  \mathbf{U}\to\{0,1,\frac{1}{2}\}$ such that
    	\begin{equation*}
    		L_E(\alpha)=\begin{cases}
    			1 & \text{ iff } \alpha\in E;\\
    			0 & \text{ iff } \alpha\in \mathbf{Def}(E);\\
    			\frac{1}{2} & \text{ otherwise.}
    		\end{cases}
    	\end{equation*}
    	The extension-based complete labelling semantics is a labelling semantics $\mathfrak{LS}_{ex}$, such that for any EHSAF, 
    	\begin{equation*}
    		\mathfrak{LS}_{ec}(EHSAF)=\{\|\cdot\|\mid\|\cdot\| \text{ is a extension-based labelling of the }EHSAF\}.
    	\end{equation*}
    \end{defn}
    \begin{rmk}\label{remark1}
    	Note that, for any $EHSAF = (\mathbf{A}, \mathbf{R_a}, \mathbf{R_e}, \mathbf{U}, \mathbf{P})$, from Definitions \ref{extsem} and \ref{extlab} the auxiliary arguments $\top$, $(\{\bot\}, \alpha)_a$, $(\{\bot\}, \beta)_e$ and $(\{\top\}, \gamma)_e$ in $\mathbf{U}$ belong to each complete extension $E$ and the auxiliary argument $\bot$ belongs to any $\mathbf{Def}(E)$. Thus, the extension-based complete labelling coincides with the labelling requirements in Definition \ref{semantics} and also the adjacent complete labelling in Definition \ref{comlab}.
    \end{rmk}
    Note that, for any $EHSAF = (\mathbf{A}, \mathbf{R_a}, \mathbf{R_e}, \mathbf{U}, \mathbf{P})$, from Definitions \ref{extsem} and \ref{extlab} the auxiliary arguments $\top$, $(\{\bot\}, \alpha)_a$, $(\{\bot\}, \beta)_e$ and $(\{\top\}, \gamma)_e$ in $\mathbf{U}$ belong to each complete extension $E$ and the auxiliary argument $\bot$ belongs to any $\mathbf{Def}(E)$. Thus, the extension-based complete labelling coincides with the labelling requirements in Definition \ref{semantics} and also the adjacent complete labelling in Definition \ref{comlab}.
    
    \begin{lem}\label{lem1}
    	For an EHSAF $(\mathbf{A},\mathbf{R_a},\mathbf{R_e},\mathbf{U},\mathbf{P})$, let $E\subseteq\mathbf{U}$ and $\mathbf{T}=\mathbf{U}\setminus \mathbf{Def}_a(E)$. If $x\in \mathbf{Sup}(\mathbf{T})$ and $\beta\in \mathbf{U}\setminus \mathbf{Sup}(\mathbf{T})$, then $x\in \mathbf{Sup}(\mathbf{T}\setminus\{\beta\})$.
    \end{lem}
    
    \begin{proof}
    	By Lemma \ref{eqmonofix}, every element $y\in \mathbf{Sup}(\mathbf{T})$ has a finite support tree rooted at $y$, whose leaves are elements of $\mathbf{P}$ and all non-root nodes (including support names and source arguments) lie in $\mathbf{T}$.
    	Now take $x\in \mathbf{Sup}(\mathbf{T})$, and choose any support tree $\mathcal{T}$ for $x$. Suppose $\mathcal{T}$ contains $\beta$. Since $\mathcal{T}$ is closed except for the root, every non-root node label lies in $\mathbf{T}$. 
    	Suppose $\mathcal{T}$ contains $\beta$. Because $\beta \notin \mathbf{Sup}(\mathbf{T})$, $\beta$ cannot be the root (as the root is $x \in \mathbf{Sup}(\mathbf{T})$). Hence $\beta$ is a non-root node, so its label must belong to $\mathbf{T}$, which gives a valid subtree rooted at $\beta$ certifying $\beta \in \mathbf{Sup}(\mathbf{T})$ by Lemma \ref{eqmonofix}, a contradiction. Thus $\mathcal{T}$ does not contain $\beta$, so all non-root nodes of $\mathcal{T}$ belong to $\mathbf{T}\setminus\{\beta\}$, witnessing $x\in \mathbf{Sup}(\mathbf{T}\setminus\{\beta\})$.
    \end{proof}
    
    \begin{thm}\label{thmea}
    	For an EHSAF $(\mathbf{A},\mathbf{R_a},\mathbf{R_e},\mathbf{U},\mathbf{P})$, if a labelling is an extension-based complete labelling of the EHSAF, then it is an adjacent complete labelling of the EHSAF.
    \end{thm}
    
    \begin{proof}
    	Let the EHSAF $(\mathbf{A},\mathbf{R_a},\mathbf{R_e},\mathbf{U},\mathbf{P})$ be given, and let $E\subseteq \mathbf{U}$ be a complete extension (Definition~\ref{extsem}). Let $L_E:\mathbf{U}\to\{0,1,\tfrac12\}$ be the extension-based complete labelling induced by $E$ according to Definition~\ref{extlab}. We show that $L_E$ satisfies the three conditions of Definition~\ref{comlab} for every $\beta\in\mathbf{U}$, where all auxiliary arguments hold according to Remark \ref{remark1}.
    	
    	\medskip
    	\noindent\textbf{Case 1: $L_E(\beta)=1$.} 
    	Then $\beta\in E$. Since $E$ is complete, by Definition~\ref{extsem} we have $E=\mathbf{Acc}(E)$, hence $\beta\in \mathbf{Acc}(E)$. By Definition~\ref{defnI},
    	\[
    	\beta\in \mathbf{Sup}(E)
    	\quad\text{and}\quad
    	\forall \mathbf{k}_{G}^{\beta}\in \mathbf{R_a}:\ 
    	\mathbf{k}_{G}^{\beta}\in \mathbf{Def}(E)\ \text{or}\ G\cap \mathbf{Def}(E)\neq \varnothing.
    	\]
    	
    	We verify the \emph{positive} labelling condition of Definition~\ref{comlab}. First, consider any attack $\mathbf{k}_{G}^{\beta}\in \mathbf{R_a}$. Since $\beta\in \mathbf{Acc}(E)$, either $\mathbf{k}_{G}^{\beta}\in \mathbf{Def}(E)$ or $G\cap \mathbf{Def}(E)\neq\varnothing$. By the definition of $L_E$, this means either $L_E(\mathbf{k}_{G}^{\beta})=0$ or there exists some $g\in G$ with $L_E(g)=0$. Thus the attack part of the positive condition holds.
    	
    	Second, we verify the support part. Since $\beta\in \mathbf{Sup}(E)$, by Definition~\ref{defnI}:
    	\[
    	\beta\in \mathbf{P}
    	\quad\text{or}\quad
    	\exists \mathbf{t}_{H}^{\beta}\in E_e\ \text{with}\ 
    	\mathbf{t}_{H}^{\beta}\in \mathbf{Sup}(E\setminus\{\beta\})
    	\ \text{and}\ 
    	H\subseteq E\cap \mathbf{Sup}(E\setminus\{\beta\}).
    	\]
    	If $\beta\in\mathbf{P}$, then by Definition~\ref{EHSAF} there is an auxiliary support $(\{\top\},\beta)_e\in \mathbf{R_e}$ with $\top\in E$ (since $\top$ is prima-facie and cannot be defeated), so the support condition is satisfied with $H=\{\top\}$ and $\|\mathbf{t}_{H}^{\beta}\|=1$. If $\beta\notin\mathbf{P}$, the above existence gives a support $\mathbf{t}_{H}^{\beta}\in E_e$ (hence $L_E(\mathbf{t}_{H}^{\beta})=1$) such that every $h\in H$ belongs to $E\cap \mathbf{Sup}(E\setminus\{\beta\})$. Since $h\in E$ for all $h\in H$, by the definition of $L_E$ we have $L_E(h)=1$. Hence the support part holds. Therefore $L_E(\beta)=1$ satisfies the “1” condition of Definition~\ref{comlab}.

    	\medskip
    	\noindent\textbf{Case 2: $L_E(\beta)=0$.}
    	Then $\beta\in \mathbf{Def}(E)=\mathbf{Def}_a(E)\cup \mathbf{Def}_s(E)$.
    	
    	If $\beta\in \mathbf{Def}_a(E)$, there exists $\mathbf{k}_{G}^{\beta}\in E_a$ with $G\subseteq E$, hence $L_E(\mathbf{k}_{G}^{\beta})=1$ and $L_E(g)=1$ for all $g\in G$. This satisfies the attack branch of the ``0'' condition.
    	
    	If $\beta\in \mathbf{Def}_s(E)$, by Definition~\ref{defnI},
    	\[
    	\beta\notin \mathbf{Sup}(\mathbf{T}),\quad\text{where } \mathbf{T}:=\mathbf{U}\setminus \mathbf{Def}_a(E).
    	\]
    	We must prove the support branch of the ``0'' condition:
    	\[
    	\forall \mathbf{t}_{H}^{\beta}\in\mathbf{R_e}:\quad
    	\mathbf{t}_{H}^{\beta}\in \mathbf{Def}(E)\ \text{or}\ \exists h\in H: h\in \mathbf{Def}(E).
    	\]
    	
    	Assume, for contradiction, that there exists a support $\mathbf{t}_{H}^{\beta}\in\mathbf{R_e}$ such that
    	\[
    	\mathbf{t}_{H}^{\beta}\notin \mathbf{Def}(E)\quad\text{and}\quad \forall h\in H: h\notin \mathbf{Def}(E).
    	\]
    	Since $\mathbf{t}_{H}^{\beta}\notin \mathbf{Def}(E)$ and $h\notin \mathbf{Def}(E)$ for all $h\in H$, we have
    	\[
    	\mathbf{t}_{H}^{\beta}\in \mathbf{U}\setminus \mathbf{Def}(E) \subseteq\mathbf{T},\quad H\subseteq \mathbf{U}\setminus \mathbf{Def}(E)\subseteq\mathbf{T}.
    	\]
    	By the definition of $\mathbf{Def}(E)=\mathbf{Def}_a(E)\cup \mathbf{Def}_s(E)$ and $\mathbf{Def}_s(E)=\mathbf{U}\setminus \mathbf{Sup}(\mathbf{T})$, we have $\mathbf{U}\setminus \mathbf{Def}(E)=\mathbf{Sup}(\mathbf{T})$. Therefore
    	\[
    	\mathbf{t}_{H}^{\beta}\in \mathbf{Sup}(\mathbf{T}),\quad H\subseteq \mathbf{Sup}(\mathbf{T}).
    	\]		
    	Since $\beta\notin \mathbf{Sup}(\mathbf{T})$, by Lemma \ref{lem1}, we obtain
    	\[
    	\mathbf{t}_{H}^{\beta}\in \mathbf{Sup}(\mathbf{T}\setminus\{\beta\}),\quad H\subseteq \mathbf{Sup}(\mathbf{T}\setminus\{\beta\}).
    	\]
    	Now, since $\mathbf{t}_{H}^{\beta}\in \mathbf{T}_e \cap \mathbf{Sup}(\mathbf{T}\setminus\{\beta\})$ and $H\subseteq \mathbf{T} \cap \mathbf{Sup}(\mathbf{T}\setminus\{\beta\})$, the definition of $\mathbf{Sup}$ (Definition~\ref{defnI}) directly yields $\beta\in \mathbf{Sup}(\mathbf{T})$, contradicting $\beta\notin \mathbf{Sup}(\mathbf{T})$.
    	
    	Thus our assumption was false. Therefore, for every support $\mathbf{t}_{H}^{\beta}$, either $\mathbf{t}_{H}^{\beta}\in \mathbf{Def}(E)$ or some $h\in H$ belongs to $\mathbf{Def}(E)$. This is exactly the support branch of the ``0'' condition. Hence $L_E(\beta)=0$ satisfies the “0” condition of Definition~\ref{comlab}.
    	
    	\medskip
    	\noindent\textbf{Case 3: $L_E(\beta)=\tfrac12$.}
    	
    	Then $\beta\notin E$ and $\beta\notin \mathbf{Def}(E)$. Since $E$ is complete, we have $E=\mathbf{Acc}(E)$, hence $\beta\notin \mathbf{Acc}(E)$. By Definition~\ref{defnI}, this means that either
    	\[
    	\beta\notin \mathbf{Sup}(E),
    	\]
    	or there exists an attack $\mathbf{k}_G^\beta\in \mathbf{R_a}$ such that
    	\[
    	\mathbf{k}_G^\beta\notin \mathbf{Def}(E)\quad\text{and}\quad G\cap \mathbf{Def}(E)=\varnothing.
    	\]
    	We show that $\beta$ satisfies neither the ``1'' condition nor the ``0'' condition of Definition~\ref{comlab}.
    	
    	First, suppose for contradiction that the ``1'' condition holds. Then $\beta\in \mathbf{Sup}(E)$ and for every attack $\mathbf{k}_G^\beta\in \mathbf{R_a}$, we have either $\mathbf{k}_G^\beta\in \mathbf{Def}(E)$ (since $\|\mathbf{k}_G^\beta\|=0$) or $G\cap \mathbf{Def}(E)\neq\varnothing$ (since some $g\in G$ has $\|g\|=0$). Therefore, by the definition of $\mathbf{Acc}(E)$ (Definition~\ref{defnI}), $\beta\in \mathbf{Acc}(E)$, contradicting $\beta\notin \mathbf{Acc}(E)$. Hence the ``1'' condition cannot hold.
    	
    	Now suppose, toward a contradiction, that the ``0'' condition holds. There are two possible branches.
    	
    	\begin{itemize}
    		\item \emph{Attack branch:} there exists $\mathbf{k}_G^\beta\in \mathbf{R_a}$ such that $\|\mathbf{k}_G^\beta\|=1$ and $\|g\|=1$ for all $g\in G$. By definition of $L_E$, this means $\mathbf{k}_G^\beta\in E$ and $G\subseteq E$. Hence, by Definition~\ref{defnI}, $\beta\in \mathbf{Def}_a(E)\subseteq \mathbf{Def}(E)$, contradicting $\beta\notin \mathbf{Def}(E)$.
    		\item \emph{Support branch:} for every support $\mathbf{t}_H^\beta\in \mathbf{R_e}$, we have
    		\begin{equation}\label{tag*}
    			\mathbf{t}_H^\beta\in \mathbf{Def}(E)\quad\text{or}\quad H\cap \mathbf{Def}(E)\neq\varnothing.
    		\end{equation}
    		We prove that this leads to $\beta\in \mathbf{Def}_s(E)$. Recall that $\mathbf{Def}_s(E)=\mathbf{U}\setminus \mathbf{Sup}(\mathbf{T})$, where $\mathbf{T}=\mathbf{U}\setminus \mathbf{Def}_a(E)$. Since $\beta\notin \mathbf{Def}(E)$, we have $\beta\notin \mathbf{Def}_a(E)$, so $\beta\in\mathbf{T}$. We discuss two cases.
    		
    		\begin{itemize}
    			\item Case 1, $\beta\in \mathbf{Sup}(\mathbf{T})$. By Definition~\ref{defnI}, either $\beta\in P$ or there exists a support $\mathbf{t}_H^\beta$ such that
    			\[
    			\mathbf{t}_H^\beta\in \mathbf{T}\cap \mathbf{Sup}(\mathbf{T}\setminus\{\beta\})
    			\quad\text{and}\quad
    			H\subseteq \mathbf{T}\cap \mathbf{Sup}(\mathbf{T}\setminus\{\beta\}).
    			\]
    			In either case, the support $\mathbf{t}_H^\beta$ satisfies
    			\begin{equation}\label{tag**}
    				\mathbf{t}_H^\beta\in \mathbf{T}\cap \mathbf{Sup}(\mathbf{T})
    				\quad\text{and}\quad
    				H\subseteq \mathbf{T}\cap \mathbf{Sup}(\mathbf{T}).
    			\end{equation}
    			From (\ref{tag**}) we get $\mathbf{t}_H^\beta\in \mathbf{Sup}(\mathbf{T})$, so $\mathbf{t}_H^\beta\notin \mathbf{Def}_s(E)$ because $\mathbf{Def}_s(E)=\mathbf{U}\setminus \mathbf{Sup}(\mathbf{T})$. Also $\mathbf{t}_H^\beta\in \mathbf{T}$, hence $\mathbf{t}_H^\beta\notin \mathbf{Def}_a(E)$. Therefore $\mathbf{t}_H^\beta\notin \mathbf{Def}(E)$. Applying the support branch condition (\ref{tag*}) to this support, we must have $H\cap \mathbf{Def}(E)\neq\varnothing$. Pick $h\in H\cap \mathbf{Def}(E)$. But $H\subseteq \mathbf{T}\cap \mathbf{Sup}(\mathbf{T})$, so $h\in \mathbf{T}$ and $h\in \mathbf{Sup}(\mathbf{T})$. Thus $h\notin \mathbf{Def}_a(E)$ and $h\notin \mathbf{Def}_s(E)$, so $h\notin \mathbf{Def}(E)$, contradicting $h\in \mathbf{Def}(E)$. 
    			
    			\item Case 2, $\beta\notin \mathbf{Sup}(\mathbf{T})$. If $\beta\notin \mathbf{Sup}(\mathbf{T})$, then $\beta\in \mathbf{Def}_s(E)\subseteq \mathbf{Def}(E)$, contradicting $\beta\notin \mathbf{Def}(E)$.
    		\end{itemize}
    	\end{itemize}
    	
    	Thus the ``0'' condition also cannot hold. Since neither the ``1'' nor the ``0'' condition is satisfied, $L_E(\beta)=\frac{1}{2}$ satisfies the “$\frac{1}{2}$” condition of Definition~\ref{comlab}. This completes Case~3.		
    	
    	We have shown that for every $\beta\in\mathbf{U}$, $L_E(\beta)$ obeys the adjacent complete labelling conditions of Definition~\ref{comlab}. Therefore $L_E$ is an adjacent complete labelling. This proves the theorem. \qedhere
    \end{proof}
    \begin{defn}[Support-Acyclic EHSAFs]
    	An EHSAF $\mathcal{F}=(\mathbf{A},\mathbf{R_a},\mathbf{R_e},\mathbf{U},\mathbf{P})$ is called \emph{support-acyclic} (or \emph{source-seeking and disjoint}) if there exists a rank function 
    	\[
    	\rho: \mathbf{U} \to \mathbb{N}
    	\]
    	such that the following conditions hold:
    	\begin{enumerate}
    		\item For every $p\in \mathbf{P}$, $\rho(p)=0$.
    		\item For every support $\mathbf{t}_{H}^{x}\in \mathbf{R_e}$, we have
    		\[
    		\rho(\mathbf{t}_{H}^{x}) < \rho(x)
    		\quad\text{and}\quad
    		\forall h\in H:\; \rho(h) < \rho(x).
    		\]
    	\end{enumerate}
    	This condition ensures that every support chain strictly decreases toward elements in $\mathbf{P}$. Hence every supported element ultimately derives its support from $\mathbf{P}$, and no cycles occur in the support graph.
    \end{defn}
    
    \begin{lem}[Positive labelling implies evidential support]
    	\label{lem:pos}
    	Let \((\mathbf{A},\mathbf{R_a},\mathbf{R_e},\mathbf{U},\mathbf{P})\) be support-acyclic with rank function \(\rho\), \(L\) be an adjacent complete labelling and \(E=\{x\in\mathbf{U}\mid L(x)=1\}\). For every \(y\in\mathbf{U}\),
    	\[
    	L(y)=1 \;\Longrightarrow\; y\in \mathbf{Sup}(E).
    	\]
    \end{lem}
    \begin{proof}
    	We proceed by induction on \(\rho(y)\).
    	
    	\textit{Base case:} \(\rho(y)=0\). Then \(y\in\mathbf{P}\). By Definition~\ref{defnI}, \(\mathbf{P}\subseteq \mathbf{Sup}(E)\) for every \(E\), so \(y\in \mathbf{Sup}(E)\).
    	
    	\textit{Inductive step:} Assume the claim holds for all elements of rank \(<\rho(y)\). Suppose \(L(y)=1\). By the positive condition of Definition~\ref{comlab}, there exists a support \(\mathbf{t}_H^y\in\mathbf{R_e}\) such that
    	\[
    	L(\mathbf{t}_H^y)=1 \quad\text{and}\quad L(h)=1 \text{ for all } h\in H.
    	\]
    	Thus \(\mathbf{t}_H^y\in E\) and \(H\subseteq E\). By support-acyclicity, \(\rho(\mathbf{t}_H^y)<\rho(y)\) and \(\rho(h)<\rho(y)\) for every \(h\in H\). Applying the induction hypothesis yields
    	\[
    	\mathbf{t}_H^y\in \mathbf{Sup}(E),\qquad h\in \mathbf{Sup}(E)\quad(\forall h\in H).
    	\]
    	Moreover, since \(\mathbf{t}_H^y\neq y\) and \(h\neq y\) (their ranks are strictly smaller), their support trees do not contain \(y\); hence, in fact,
    	\[
    	\mathbf{t}_H^y\in \mathbf{Sup}(E\setminus\{y\}),\qquad H\subseteq \mathbf{Sup}(E\setminus\{y\}).
    	\]
    	Additionally, \(\mathbf{t}_H^y \in E\) and \(\mathbf{t}_H^y\neq y\) give \(\mathbf{t}_H^y \in E\setminus\{y\}\). Also, \(H\subseteq E\) and \(h\neq y\) give \(H\subseteq E\setminus\{y\}\). Therefore,
    	\[
    	\mathbf{t}_H^y\in (E\setminus\{y\})_e\cap \mathbf{Sup}(E\setminus\{y\})
    	\quad\text{and}\quad
    	H\subseteq (E\setminus\{y\})\cap \mathbf{Sup}(E\setminus\{y\}).
    	\]
    	By Definition~\ref{defnI} of \(\mathbf{Sup}\), this implies \(y\in \mathbf{Sup}(E)\). The induction is complete.
    \end{proof}
    \begin{lem}[Positive labelling implies support in any such \(S\)]
    	\label{lem:pos-extended}
    	Let \((\mathbf{A},\mathbf{R_a},\mathbf{R_e},\mathbf{U},\mathbf{P})\) be a support-acyclic EHSAF, \(L\) be an adjacent complete labelling, \(E=\{x\in\mathbf{U}\mid L(x)=1\}\) and \(S\subseteq\mathbf{U}\) be any set satisfying $E \subseteq S \subseteq \mathbf{U}\setminus \mathbf{Def}_a(E)$. For every \(y\in\mathbf{U}\),
    	\[
    	L(y)=1 \;\Longrightarrow\; y\in \mathbf{Sup}(S).
    	\]
    \end{lem}
    
    \begin{proof}
    	By the original positive labelling lemma (Lemma~\ref{lem:pos}), which holds under support-acyclicity, we have
    	\[
    	L(y)=1 \;\Longrightarrow\; y\in \mathbf{Sup}(E).
    	\]
    	Since \(E\subseteq S\), the monotonicity of \(\mathbf{Sup}\) (if \(A\subseteq B\) then \(\mathbf{Sup}(A)\subseteq \mathbf{Sup}(B)\)) yields \(\mathbf{Sup}(E)\subseteq \mathbf{Sup}(S)\). Hence \(y\in \mathbf{Sup}(S)\).
    \end{proof}
    \begin{lem}[Conflict-freeness (I)]\label{conflictfree}
    	Let \((\mathbf{A},\mathbf{R_a},\mathbf{R_e},\mathbf{U},\mathbf{P})\) be a support-acyclic EHSAF, \(L\) be an adjacent complete labelling, \(E=\{x\in\mathbf{U}\mid L(x)=1\}\). We have $E\cap \mathbf{Def}(E)=\varnothing$.
    \end{lem}
    \begin{proof}
    	Suppose, for contradiction, that $x\in E\cap \mathbf{Def}(E)$. Then $x\in \mathbf{Def}_a(E)\cup \mathbf{Def}_s(E)$.
    	
    	If $x\in \mathbf{Def}_a(E)$, then by Definition~\ref{defnI} there exists $\mathbf{k}_G^x\in E_a$ with $G\subseteq E$. Since $L(x)=1$, the positive condition of Definition~\ref{comlab} requires that for every attack $\mathbf{k}_G^x$, either $L(\mathbf{k}_G^x)=0$ or $\exists g\in G: L(g)=0$. But $G\subseteq E$ implies $L(g)=1$ for all $g\in G$, and $\mathbf{k}_G^x\in E$ implies $L(\mathbf{k}_G^x)=1$; contradiction. If $x\in \mathbf{Def}_s(E)$, then by Definition~\ref{defnI}, $x\notin \mathbf{Sup}(\mathbf{T})$ where $\mathbf{T}=\mathbf{U}\setminus \mathbf{Def}_a(E)$. From Lemma \ref{lem:pos-extended} and $L(x)=1$, we have $x\in \mathbf{Sup}(\mathbf{T})$; contradiction. Therefore $E\cap \mathbf{Def}(E)=\varnothing$.
    \end{proof}
    \begin{lem}[Conflict-freeness (II)]\label{conflictfreeII}
    	Let $(\mathbf{A},\mathbf{R_a},\mathbf{R_e},\mathbf{U},\mathbf{P})$ be an EHSAF, $L$ be an adjacent complete labelling and $E=\{x\in\mathbf{U}\mid L(x)=1\}$. For any adjacent complete labelling $L$, we have $E\cap \mathbf{Def}_a(E)=\varnothing$.
    \end{lem}
    \begin{proof}
    	If $x\in E\cap \mathbf{Def}_a(E)$, then by Definition~\ref{defnI} there exists an attack $\mathbf{k}_G^x\in E$ with $G\subseteq E$. Hence $L(\mathbf{k}_G^x)=1$ and $L(g)=1$ for all $g\in G$. The attack branch of the ``0''-condition in Definition~\ref{comlab} would then force $L(x)=0$, contradicting $x\in E$. Thus no such $x$ exists.
    \end{proof}
    \begin{lem}[Conflict-freeness (III)]\label{conflictfreeIII}
    Let $(\mathbf{A},\mathbf{R_a},\mathbf{R_e},\mathbf{U},\mathbf{P})$ be an EHSAF, $L$ be an adjacent complete labelling, $E=\{x\in\mathbf{U}\mid L(x)=1\}$ and $E\subseteq \mathbf{Sup}(E)$. For any adjacent complete labelling $L$, we have $E\cap \mathbf{Def}(E)=\varnothing$.
    \end{lem}
    
    \begin{proof}
    	From Lemma \ref{conflictfreeII}, $E\cap \mathbf{Def}_a(E)=\varnothing$. Let $\mathbf{T}=\mathbf{U}\setminus \mathbf{Def}_a(E)$; then $E\subseteq \mathbf{T}$. By $E\subseteq \mathbf{Sup}(E)$ and monotonicity of $\mathbf{Sup}$, we have
    	\[
    	E\subseteq \mathbf{Sup}(E)\subseteq \mathbf{Sup}(\mathbf{T}),
    	\]
    	so $E\cap \mathbf{Def}_s(E)=\varnothing$ because $\mathbf{Def}_s(E)=\mathbf{U}\setminus \mathbf{Sup}(\mathbf{T})$. Hence $E\cap \mathbf{Def}(E)=\varnothing$.
    \end{proof}
    
    \begin{lem}[Negative labelling implies defeat in general]
    	\label{lem:neg-general}
    	Let \((\mathbf{A},\mathbf{R_a},\mathbf{R_e},\mathbf{U},\mathbf{P})\) be an EHSAF, \(L\) be an adjacent complete labelling and \(E=\{x\in\mathbf{U}\mid L(x)=1\}\). For every \(y\in\mathbf{U}\),
    	\[
    	L(y)=0 \;\Longrightarrow\; y\in \mathbf{Def}(E).
    	\]
    \end{lem}
    
    \begin{proof}
    	Let \(y\in\mathbf{U}\) with \(L(y)=0\). According to Definition~\ref{comlab}, at least one of the following two conditions holds:
    	
    	\textit{Attack branch.} There exists an attack \(\mathbf{k}_G^y\in\mathbf{R_a}\) such that
    	\[
    	L(\mathbf{k}_G^y)=1 \quad\text{and}\quad L(g)=1\ \text{for all }g\in G.
    	\]
    	Then \(\mathbf{k}_G^y\in E\) and \(G\subseteq E\). Hence, by Definition~\ref{defnI}, \(y\in \mathbf{Def}_a(E)\subseteq \mathbf{Def}(E)\).
    	
    	\textit{Support branch.} For every support \(\mathbf{t}_H^y\in\mathbf{R_e}\),
    	\begin{equation}\label{18tag1}
    		L(\mathbf{t}_H^y)=0 \quad\text{or}\quad \exists h\in H:\ L(h)=0.
    	\end{equation}
    	We prove by contradiction that this branch also implies \(y\in \mathbf{Def}(E)\).
    	
    	Assume, for contradiction, that \(y\notin \mathbf{Def}(E)\). Then \(y\notin \mathbf{Def}_a(E)\) and \(y\notin \mathbf{Def}_s(E)\). Let
    	\[
    	\mathbf{T}=\mathbf{U}\setminus \mathbf{Def}_a(E).
    	\]
    	Since \(y\notin \mathbf{Def}_a(E)\), we have \(y\in\mathbf{T}\). Since \(y\notin \mathbf{Def}_s(E)=\mathbf{U}\setminus \mathbf{Sup}(\mathbf{T})\), it follows that \(y\in \mathbf{Sup}(\mathbf{T})\).
    	
    	By Definition~\ref{defnI} of \(\mathbf{Sup}\) (as the least fixed point), the membership \(y\in \mathbf{Sup}(\mathbf{T})\) implies that either \(y\in\mathbf{P}\), or there exists a support \(\mathbf{t}_H^y\in\mathbf{T}_e\) such that
    	\begin{equation}\label{18tag2}
    		\mathbf{t}_H^y\in \mathbf{Sup}(\mathbf{T}\setminus\{y\})
    		\quad\text{and}\quad
    		H\subseteq \mathbf{T}\cap \mathbf{Sup}(\mathbf{T}\setminus\{y\}).  		
    	\end{equation}
    	(The case \(y\in\mathbf{P}\) is impossible here because if \(y\in\mathbf{P}\), then by Definition~\ref{EHSAF} there is an auxiliary support from \(\top\) that would contradict the support branch \eqref{18tag1}, as \(\top\) and the auxiliary support are always labelled 1. Thus the support case must hold.)
    	
    	From \eqref{18tag2} we obtain:
    	\[
    	\mathbf{t}_H^y\in \mathbf{T}\cap \mathbf{Sup}(\mathbf{T}),\qquad H\subseteq \mathbf{T}\cap \mathbf{Sup}(\mathbf{T}).
    	\]
    	Hence \(\mathbf{t}_H^y\notin \mathbf{Def}_a(E)\) and \(h\notin \mathbf{Def}_a(E)\) for all \(h\in H\). Also, since \(\mathbf{t}_H^y\in \mathbf{Sup}(\mathbf{T})\) and \(h\in \mathbf{Sup}(\mathbf{T})\), they are not in \(\mathbf{Def}_s(E)=\mathbf{U}\setminus \mathbf{Sup}(\mathbf{T})\). Therefore
    	\begin{equation}\label{19tag3}
    	\mathbf{t}_H^y\notin \mathbf{Def}(E),\qquad h\notin \mathbf{Def}(E)\quad(\forall h\in H).
    	\end{equation}
    	In particular, \(H\cap \mathbf{Def}(E)=\varnothing\).
    	
    	But \eqref{19tag3} directly contradicts the support branch condition \eqref{18tag1}, which requires for this very support \(\mathbf{t}_H^y\) that either \(\mathbf{t}_H^y\in \mathbf{Def}(E)\) or \(H\cap \mathbf{Def}(E)\neq\varnothing\). Therefore our assumption \(y\notin \mathbf{Def}(E)\) is false, and we conclude \(y\in \mathbf{Def}(E)\).
    	
    	Combining the attack and support branches, we have \(L(y)=0\Rightarrow y\in \mathbf{Def}(E)\).
    \end{proof}
    
    \begin{lem}[Defeat implies negative labelling]
    	\label{lem:def-to-zero}
    	Let \((\mathbf{A},\mathbf{R_a},\mathbf{R_e},\mathbf{U},\mathbf{P})\) be a support-acyclic EHSAF with rank function \(\rho\), \(L\) be an adjacent complete labelling and \(E=\{x\in\mathbf{U}\mid L(x)=1\}\). For every \(y\in\mathbf{U}\),
    	\[
    	y\in \mathbf{Def}(E) \;\Longrightarrow\; L(y)=0.
    	\]
    \end{lem}
    
    \begin{proof}
    	We prove by induction on \(\rho(y)\).
    	
    	\emph{Base case:} \(\rho(y)=0\). Then \(y\in\mathbf{P}\). Since \(y\in \mathbf{Def}(E)\), \(y\) cannot be in \(\mathbf{Def}_s(E)\) (because \(\mathbf{Def}_s(E)=\mathbf{U}\setminus \mathbf{Sup}(\mathbf{T})\) and \(\mathbf{P}\subseteq \mathbf{Sup}(\mathbf{T})\) for any \(\mathbf{T}\)). Hence \(y\in \mathbf{Def}_a(E)\), so there is an attack \(\mathbf{k}_G^y\in E\) with \(G\subseteq E\). Thus \(L(\mathbf{k}_G^y)=1\) and \(L(g)=1\) for all \(g\in G\). By the attack branch of the ``0''-condition in Definition~\ref{comlab}, \(L(y)=0\).
    	
    	\emph{Inductive step:} Assume the claim holds for all elements of rank \(<\rho(y)\). Let \(y\in \mathbf{Def}(E)\).
    	
    	We first rule out \(L(y)=1\). If \(L(y)=1\), then by Lemma  \ref{lem:pos} \(y\in \mathbf{Sup}(E)\). Since \(E\) is conflict-free by Lemma  \ref{conflictfree}, \(E\cap \mathbf{Def}_a(E)=\varnothing\), so \(E\subseteq \mathbf{T}:=\mathbf{U}\setminus \mathbf{Def}_a(E)\). Monotonicity of \(\mathbf{Sup}\) gives \(\mathbf{Sup}(E)\subseteq \mathbf{Sup}(\mathbf{T})\), hence \(y\in \mathbf{Sup}(\mathbf{T})\). But if \(y\in \mathbf{Def}_s(E)\), then \(y\notin \mathbf{Sup}(\mathbf{T})\), contradiction. If \(y\in \mathbf{Def}_a(E)\), by Definition \ref{comlab} we have \(L(y)=0\). Thus in all cases \(L(y)\neq 1\).
    	
    	Now suppose, for contradiction, that \(L(y)=\tfrac12\). By Definition~\ref{comlab}, \(L(y)=\tfrac12\) means that neither the ``1''-condition nor the ``0''-condition holds. In particular, the ``0''-condition fails. The ``0''-condition is the disjunction of an attack branch and a support branch. The attack branch would require an attack \(\mathbf{k}_G^y\) with \(L(\mathbf{k}_G^y)=1\) and \(L(g)=1\) for all \(g\in G\); if that held, it would force \(L(y)=0\), contradicting \(L(y)=\tfrac12\). Hence the attack branch fails. Therefore, for the 0-condition to fail, the support branch must also fail. The support branch is
    	\[
    	\forall \mathbf{t}_H^y\in\mathbf{R_e}:\; L(\mathbf{t}_H^y)=0 \text{ or } \exists h\in H\,L(h)=0.
    	\]
    	Its failure yields the existence of a support \(\mathbf{t}_H^y\in\mathbf{R_e}\) such that
    	\begin{equation*}
    		L(\mathbf{t}_H^y)\neq 0 \quad\text{and}\quad \forall h\in H\,L(h)\neq 0.
    	\end{equation*}
    	
    	Now, for each \(z\in\{\mathbf{t}_H^y\}\cup H\), we have \(\rho(z)<\rho(y)\) by support-acyclicity. By the induction hypothesis (contrapositive), since \(L(z)\neq 0\), we must have \(z\notin \mathbf{Def}(E)\). Hence \(z\notin \mathbf{Def}_a(E)\) and \(z\notin \mathbf{Def}_s(E)\). From \(z\notin \mathbf{Def}_a(E)\), we get \(z\in\mathbf{T}\). From \(z\notin \mathbf{Def}_s(E)=\mathbf{U}\setminus \mathbf{Sup}(\mathbf{T})\), we get \(z\in \mathbf{Sup}(\mathbf{T})\). Therefore
    	\[
    	\mathbf{t}_H^y\in \mathbf{T}_e\cap \mathbf{Sup}(\mathbf{T}),\qquad H\subseteq \mathbf{T}\cap \mathbf{Sup}(\mathbf{T}).
    	\]
    	By the definition of \(\mathbf{Sup}\) (Definition~\ref{defnI}), this implies \(y\in \mathbf{Sup}(\mathbf{T})\), contradicting \(y\in \mathbf{Def}_s(E)\) if \(y\in \mathbf{Def}_s(E)\). If \(y\in \mathbf{Def}_a(E)\), we already have \(L(y)=0\). Thus the assumption \(L(y)=\tfrac12\) is impossible.
    	
    	Since \(L(y)\) is neither 1 nor \(\tfrac12\), it must be that \(L(y)=0\). The induction is complete.
    \end{proof}
    \begin{cor}\label{coreq}
    	Let \((\mathbf{A},\mathbf{R_a},\mathbf{R_e},\mathbf{U},\mathbf{P})\) be a support-acyclic EHSAF, \(L\) be an adjacent complete labelling and \(E=\{x\in\mathbf{U}\mid L(x)=1\}\). For every \(y\in\mathbf{U}\),
    	\[
    	y\in \mathbf{Def}(E) \;\Longleftrightarrow\; L(y)=0.
    	\]
    \end{cor}
    \begin{proof}
    	This follows immediately from Lemmas \ref{lem:neg-general} and \ref{lem:def-to-zero}.
    \end{proof}
    \begin{thm}\label{thmae}
    	For a support-acyclic EHSAF $\mathcal{F}=(\mathbf{A},\mathbf{R_a},\mathbf{R_e},\mathbf{U},\mathbf{P})$, any adjacent complete labelling of $\mathcal{F}$ is an extension-based complete labelling of $\mathcal{F}$.
    \end{thm}
    
    \begin{proof}		
    	Let $L:\mathbf{U}\to\{0,1,\tfrac12\}$ be an adjacent complete labelling (Definition~\ref{comlab}). Define
    	\[
    	E \triangleq \{x\in\mathbf{U}\mid L(x)=1\}.
    	\]
    	We must show that $E$ is a complete extension in the sense of Definition~\ref{extsem}, i.e.
    	\[
    	E\cap \mathbf{Def}(E)=\varnothing \quad\text{and}\quad \mathbf{Acc}(E)=E.
    	\]
    	
    	Since the conflict-freeness have been proved in Lemma \ref{conflictfree}, we proceed via the following two claims.

    	\begin{claim}[$E\subseteq \mathbf{Acc}(E)$]
    		Every $x\in E$ is acceptable w.r.t. $E$.
    	\end{claim}
    	\begin{proof}
    		Let $x\in E$, so $L(x)=1$. By Lemma \ref{lem:pos}, we obtain $x\in \mathbf{Sup}(E)$. Moreover, for every attack $\mathbf{k}_G^x\in\mathbf{R_a}$, the positive condition of Definition \ref{comlab} gives either $L(\mathbf{k}_G^x)=0$ or $\exists g\in G: L(g)=0$. By Lemma \ref{lem:neg-general}, $L(y)=0$ implies $y\in \mathbf{Def}(E)$. Hence every attack against $x$ is either itself in $\mathbf{Def}(E)$ or has a source in $\mathbf{Def}(E)$. Therefore, by Definition~\ref{defnI}, $x\in \mathbf{Acc}(E)$.
    	\end{proof}
    	
    	\begin{claim}[$\mathbf{Acc}(E)\subseteq E$]
    		Every acceptable element w.r.t. $E$ belongs to $E$.
    	\end{claim}
    	
    	\begin{proof}
    		Let \(x\in \mathbf{Acc}(E)\). By Definition~\ref{defnI}, \(x\in \mathbf{Sup}(E)\) and for every attack \(\mathbf{k}_G^x\in\mathbf{R_a}\),
    		\begin{equation}\label{20tag1}
    			\mathbf{k}_G^x\in \mathbf{Def}(E)\quad\text{or}\quad G\cap \mathbf{Def}(E)\neq\varnothing.
    		\end{equation}
    		
    		We verify that the ``1''-condition of Definition~\ref{comlab} holds for \(x\).
    		
    		\emph{Support part.} Since \(x\in \mathbf{Sup}(E)\), by the inductive definition of \(\mathbf{Sup}\) (Definition~\ref{defnI}) there exists a support \(\mathbf{t}_H^x\in E_e\) such that
    		\[
    		\mathbf{t}_H^x\in \mathbf{Sup}(E\setminus\{x\})
    		\quad\text{and}\quad
    		H\subseteq E\cap \mathbf{Sup}(E\setminus\{x\}).
    		\]
    		Therefore
    		\[
    		L(\mathbf{t}_H^x)=1,\qquad L(h)=1\ \text{for all }h\in H.
    		\]
    		Thus the support part of the 1-condition is satisfied.
    		
    		\emph{Attack part.} For any attack \(\mathbf{k}_G^x\), \eqref{20tag1} gives either \(\mathbf{k}_G^x\in \mathbf{Def}(E)\) or \(G\cap \mathbf{Def}(E)\neq\varnothing\). By Lemma \ref{lem:def-to-zero}, this is equivalent to \(L(\mathbf{k}_G^x)=0\) or \(\exists g\in G: L(g)=0\). Hence the attack part of the 1-condition is satisfied.
    		
    		Since both parts of the 1-condition are fulfilled, Definition~\ref{comlab} yields \(L(x)=1\), i.e. \(x\in E\). Therefore \(\mathbf{Acc}(E)\subseteq E\).
    	\end{proof}
    	
    	Combining Claims 1 and 2 yields $\mathbf{Acc}(E)=E$. Hence $E$ is a complete extension (Definition~\ref{extsem}), and $L$ is exactly the extension-based complete labelling induced by $E$ (Definition~\ref{extlab} and Corollary \ref{coreq}). This completes the proof.
    \end{proof}
    
    \begin{cor}
    	For a support-acyclic EHSAF $\mathcal{F}=(\mathbf{A},\mathbf{R_a},\mathbf{R_e},\mathbf{U},\mathbf{P})$, a labelling $\mathbf{U}\to \{0,1,\frac{1}{2}\}$ is an adjacent complete labelling of $\mathcal{F}$ iff it is an extension-based complete labelling of $\mathcal{F}$.
    \end{cor}
    \begin{proof}
    	This follows immediately from Theorems \ref{thmea} and \ref{thmae}.
    \end{proof}
         \begin{exmp}[Counterexample for the converse of Theorem~\ref{thmea}]
  	\label{ex:counterexample}
  	This example constructs an EHSAF containing a cycle of evidential supports. 
  	It serves two purposes: (i) to demonstrate that, in the presence of support cycles, 
  	adjacent complete labellings need not coincide with extension-based complete labellings; 
  	and (ii) to illustrate the distinct epistemic attitudes underlying the two semantics.
  	
  	We first define the pre‐EHSAF $\mathcal{F}_0 = (\mathbb{A}, \mathbb{R}_a, \mathbb{R}_e, \mathbb{U}, \mathbb{P})$ as follows:
  	$\mathbb{A} = \{a,b\}$, $\mathbb{R}_a = \varnothing$, $\mathbb{R}_e = \{\mathbf{t}_1,\mathbf{t}_2\}$ with $\mathbf{t}_1 =(\{a\}, b)$ and $\mathbf{t}_2 =(\{b\}, a)$, $\mathbb{P} = \{\mathbf{t}_1,\mathbf{t}_2\}$, and $\mathbb{U} = \mathbb{A} \cup \mathbb{R}_a \cup \mathbb{R}_e = \{a,b,\mathbf{t}_1,\mathbf{t}_2\}$. The support relation forms a cycle: $a$ supports $b$ via $\mathbf{t}_1$, and $b$ supports $a$ via $\mathbf{t}_2$. 
  	The support names $\mathbf{t}_1,\mathbf{t}_2$ are prima‐facie, whereas the arguments $a,b$ are not. Following Definition~\ref{EHSAF}, we construct the corresponding EHSAF 
  	$\mathcal{F} = (\mathbf{A},\mathbf{R_a},\mathbf{R_e},\mathbf{U},\mathbf{P})$ by adding auxiliary elements. Explicitly:
  	\begin{align*}
  		\mathbf{A} &= \{a,b,\bot,\top\},\\
  		\mathbf{R_a} &= \{ (\{\bot\},a)_a,\; (\{\bot\},b)_a,\; (\{\bot\},\mathbf{t}_1)_a,\; (\{\bot\},\mathbf{t}_2)_a \},\\
  		\mathbf{R_e} &= \{ \mathbf{t}_1:(\{a\},b)_e,\; \mathbf{t}_2:(\{b\},a)_e,\; (\{\top\},\mathbf{t}_1)_e,\; (\{\top\},\mathbf{t}_2)_e \},\\
  		\mathbf{U} &= \mathbf{A} \cup \mathbf{R_a} \cup \mathbf{R_e},\\
  		\mathbf{P} &= \{\mathbf{t}_1,\mathbf{t}_2,\top\} \cup \mathbf{R_a} \cup \{ (\{\top\},\mathbf{t}_1)_e,\; (\{\top\},\mathbf{t}_2)_e \}.
  	\end{align*}
  	By Definitions~\ref{semantics}, \ref{comlab} and \ref{extlab}, the auxiliary elements have fixed labels: $\|\bot\|=0$, $\|\top\|=1$, and all newly added attacks and supports are labelled $1$ for any adjacent complete labellings or extension-based complete labelling; hence we focus on these four core elements in the analysis below.
  	
  	\medskip
  	\noindent\textbf{Adjacent Complete Labellings.}
  	By Definition~\ref{comlab}, the support names $\mathbf{t}_1,\mathbf{t}_2\in P$ are labelled $1$ in every adjacent complete labelling. Thus, Definition~\ref{comlab} implies exactly three possible labellings denoted by $L_1$, $L_0$ and $L_{1/2}$:
  	\begin{enumerate}
  		\item \textit{Optimistic scenario (anticipated evidence emerges in the future):}
  		\[
  		L_1(a)=1,\quad L_1(b)=1,\quad L_1(\mathbf{t}_1)=1,\quad L_1(\mathbf{t}_2)=1.
  		\]
  		
  		\item \textit{Pessimistic scenario (future evidential counterexamples refute the arguments):}
  		\[
  		L_0(a)=0,\quad L_0(b)=0,\quad L_0(\mathbf{t}_1)=1,\quad L_0(\mathbf{t}_2)=1.
  		\]
  		
  		\item \textit{Agnostic scenario (evidence remains unavailable):}
  		\[
  		L_{1/2}(a)=\frac{1}{2},\quad L_{1/2}(b)=\frac{1}{2},\quad L_{1/2}(\mathbf{t}_1)=1,\quad L_{1/2}(\mathbf{t}_2)=1.
  		\]
  	\end{enumerate}
  	These three labellings are exactly the adjacent complete labellings of $\mathcal{F}$.
  	
  	\medskip
  	\noindent\textbf{Extension-Based Complete Labellings.}
  	We now compute the complete extensions according to Definition~\ref{extsem}. 
  	Let \(E \subseteq \mathbf{U}\) be a complete extension, so \(E = \mathbf{Acc}(E)\) and \(E \cap \mathbf{Def}(E) = \varnothing\).
  	
  	\medskip
  	\noindent\textbf{Step 1: Auxiliary elements (already fixed).}
  	As noted above, \(\top \in E\), \(\bot \in \mathbf{Def}(E)\), and all added default attacks and supports (which are prima-facie) belong to \(E\). 
  	We will use these facts below.
  	
  	\medskip
  	\noindent\textbf{Step 2: The support names \(\mathbf{t}_1,\mathbf{t}_2\) are in \(E\).}
  	Since \(\mathbf{t}_1,\mathbf{t}_2 \in \mathbf{P}\), they are always in \(\mathbf{Sup}(E)\) by Definition~\ref{defnI}. 
  	Moreover, only attacker $\{\bot\}$ targets them. 
  	Hence \(\mathbf{t}_1,\mathbf{t}_2 \in \mathbf{Acc}(E)\). 
  	As \(E\) is complete, \(\mathbf{Acc}(E)=E\), so \(\mathbf{t}_1,\mathbf{t}_2 \in E\).
  	
  	\medskip
  	\noindent\textbf{Step 3: \(a,b\) cannot be in \(E\).}
  	Suppose \(a \in E\). Since \(E = \mathbf{Acc}(E) \subseteq \mathbf{Sup}(E)\), we must have \(a \in \mathbf{Sup}(E)\). 
  	The only support for \(a\) is \(\mathbf{t}_2\) with source \(\{b\}\). 
  	By Definition~\ref{defnI}, \(a \in \mathbf{Sup}(E)\) requires \(\mathbf{t}_2 \in \mathbf{Sup}(E\setminus\{a\})\) (which holds, as \(\mathbf{t}_2\in\mathbf{P}\)) and \(b \in \mathbf{Sup}(E\setminus\{a\})\). 
  	Since \(b\) is not prima-facie and the only supporter of $b$ is $\{a\}$ with $a\notin E\setminus\{a\}$, we have \(b \notin \mathbf{Sup}(E\setminus\{a\})\). Contradiction. 
  	Hence \(a \notin E\). By symmetry, \(b \notin E\).
  	
  	Thus, combining with Steps 1 and 2, the only candidate for \(E\) is:
  	\[
  	E = \mathbf{P}.
  	\]
  	
  	\medskip
  	\noindent\textbf{Step 4: Determining the labels of \(a\) and \(b\) via \(\mathbf{Def}(E)\).}
  	For \(E = \mathbf{P}\), we compute \(\mathbf{Def}(E)\). 
  	Since every attack has source \(\{\bot\}\) and \(\bot \notin E\), no attack is effective; hence \(\operatorname{\mathbf{Def}}_a(E) = \varnothing\).
  	
  	For \(\operatorname{\mathbf{Def}}_s(E) = \mathbf{U} \setminus \mathbf{Sup}(\mathbf{U} \setminus \operatorname{\mathbf{Def}}_a(E)) = \mathbf{U} \setminus \mathbf{Sup}(\mathbf{U})\), 
  	we note that \(a,b\notin \mathbf{Sup}(\mathbf{U})\) by the same cyclic argument as in Step 3. 
  	Thus \(\mathbf{Sup}(\mathbf{U}) = \mathbf{P}\) (which contains \(\mathbf{t}_1,\mathbf{t}_2,\top\) and all auxiliary interactions, but not \(\bot,a,b\)). 
  	Therefore
  	\[
  	\mathbf{Def}(E) = \varnothing \cup (\mathbf{U} \setminus \mathbf{P}) = \{\bot, a, b\}.
  	\]
  	
  	The extension-based complete labelling (Definition~\ref{extlab}) assigns \(1\) to elements in \(E\), and \(0\) to elements in \(\mathbf{Def}(E)\). 
  	Since \(a,b \in \mathbf{Def}(E)\) and \(\mathbf{t}_1,\mathbf{t}_2 \in E\), we obtain the unique extension-based complete labelling $L_{E^*}$ such that:
  	\[
  	L_{E^*}(a)=0,\quad L_{E^*}(b)=0,\quad L_{E^*}(\mathbf{t}_1)=1,\quad L_{E^*}(\mathbf{t}_2)=1,
  	\]
  	which is exactly the pessimistic scenario among the adjacent complete labellings.
  	
  	\medskip
  	\noindent\textbf{Semantic Comparison.}
  	This example clearly manifests the core philosophical divergence:
  	\begin{itemize}
  		\item \emph{Extension‐based complete semantics} adopts a strict evidentialist stance: 
  		arguments that cannot be grounded in current prima‐facie evidence (here, those trapped in a support cycle) 
  		are definitively judged false, following the conservative principle that an argument without present support is regarded as wrong.
  		\item \emph{Adjacent complete labelling semantics} adopts a tolerant, open epistemic attitude: 
  		it accommodates three possible truth‐value outcomes for arguments in support cycles, 
  		corresponding to future confirmation, future refutation, or persistent undecidability. 
  		It does not prematurely rule out potentially sound arguments solely on the grounds of currently absent evidence.
  	\end{itemize}
  	Hence the statement that for every EHSAF every adjacent complete labelling is extension‐based complete fails when the support relation contains cycles. $\square$   	    
  \end{exmp}

    	\section{Encoded Semantics of EHSAFs}
	We claim that for a given $\mathcal{PL}$ with the propositional constant set $\text{Con}$ and the propositional variable set $\text{Var}$ and for any $EHSAF=(\mathbf{A},\mathbf{R_a},\mathbf{R_e},\mathbf{U},\mathbf{P})$, always let $\mathbf{U}\subseteq \text{Con}\cup \text{Var}$ and $\{\bot, \top\}\cup ((\mathbf{R_e}\cup\mathbf{R_e})\setminus(R_a\cup R_e))\subseteq\text{Con}$ with $\|\bot\|=0$, $\|\top\|=1$ and $\|\alpha\|=1$ for each $\alpha \in (\mathbf{R_e}\cup\mathbf{R_e})\setminus(R_a\cup R_e)$.
	\subsection{General Encoded Semantics}
			\begin{defn}[Encoding for EHSAFs]
		An \emph{encoding} of $\mathcal{EHSAF}$ w.r.t. $\mathcal{PL}$ is a function $ec: \mathcal{EHSAF} \rightarrow \mathcal{F_{\mathcal{PL}}}$ such that for an $EHSAF=(\mathbf{A}, \mathbf{R_a}, \mathbf{R_e}, \mathbf{U}, \mathbf{P})$, the set of propositional variables, constant $\bot$ and constant $\top$ appearing in $ec(EHSAF)$ coincides exactly with the universal set $\mathbf{U}$. The formula $ec(EHSAF)$ is called the \emph{encoded formula} of the EHSAF.
	\end{defn}
	
	\begin{defn}[Encoded Semantics]\label{defn:encoded_semantics}
		Given a propositional logic system $\mathcal{PL}$ and an encoding function $ec: \mathcal{EHSAF} \to \mathcal{F_{\mathcal{PL}}}$, the \emph{encoded semantics} induced by $ec$ and $\mathcal{PL}$ is a labelling semantics $\mathfrak{LS}_{ec}^{\mathcal{PL}}$ defined by:
		\[
		\mathfrak{LS}_{ec}^{\mathcal{PL}}(EHSAF) = \left\{ \|\cdot\| \mid \|ec(EHSAF)\| = 1 \right\}.
		\]
		An assignment $\|\cdot\| \in \mathfrak{LS}_{ec}^{\mathcal{PL}}(EHSAF)$ is called a \emph{model} of the EHSAF under this encoded semantics, denoted as $\|\cdot\| \models_{\mathfrak{LS}_{ec}^{\mathcal{PL}}} EHSAF$.
	\end{defn}
	Obviously, for an assignment $\|\cdot\|$, $\|\cdot\| \models_{\mathfrak{LS}_{ec}^{\mathcal{PL}}} EHSAF$ iff $\|ec(EHSAF)\| = 1$ in the $\mathcal{PL}$ iff $\|\cdot\| \models_\mathcal{PL} ec(EHSAF)$. 
		\begin{defn}[Skeptical and Credulous Encoded Semantics for EHSAF]
		Let $\mathfrak{L}\mathfrak{S}_{ec}^{\mathcal{PL}}$ be an encoded semantics associated with a propositional logic $\mathcal{PL}$ with truth value set $L\subseteq[0,1]$ and an encoding function $ec$. 
		
		The \emph{skeptical encoded semantics} induced by $ec$ and $\mathcal{PL}$ is a semantics $\mathfrak{L}\mathfrak{S}_{ec}^{\mathcal{PL},\text{sk}}$ such that for any EHSAF $\mathcal{F}=(\mathbf{A},\mathbf{R}_{a},\mathbf{R}_{e},\mathbf{U},\mathbf{P})$,
		\[
		\mathfrak{L}\mathfrak{S}_{ec}^{\mathcal{PL},\text{sk}}(\mathcal{F}) =
		\left\{
		\|\cdot\|_{\mathcal{F}}^{\text{sk}} : \mathbf{U}\to L
		\;\middle|\;
		\|x\|_{\mathcal{F}}^{\text{sk}} = \min\{\|x\| \mid \|\cdot\| \in \mathfrak{L}\mathfrak{S}_{ec}^{\mathcal{PL}}(\mathcal{F})\},\ \forall x\in\mathbf{U}
		\right\}.
		\]
		
		Analogously, the \emph{credulous encoded semantics} induced by $ec$ and $\mathcal{PL}$ is a semantics $\mathfrak{L}\mathfrak{S}_{ec}^{\mathcal{PL},\text{cr}}$ such that for any EHSAF $\mathcal{F}=(\mathbf{A},\mathbf{R}_{a},\mathbf{R}_{e},\mathbf{U},\mathbf{P})$,
		\[
		\mathfrak{L}\mathfrak{S}_{ec}^{\mathcal{PL},\text{cr}}(\mathcal{F}) =
		\left\{
		\|\cdot\|_{\mathcal{F}}^{\text{cr}} : \mathbf{U}\to L
		\;\middle|\;
		\|x\|_{\mathcal{F}}^{\text{cr}} = \max\{\|x\| \mid \|\cdot\| \in \mathfrak{L}\mathfrak{S}_{ec}^{\mathcal{PL}}(\mathcal{F})\},\ \forall x\in\mathbf{U}
		\right\}.
		\]
	\end{defn}
	In words, the model of skeptical encoded semantics assigns to each element $x\in\mathbf{U}$ the minimum truth value it receives across all models of the basic encoded semantics, thereby capturing the intersection (or infimum) of all possible interpretations. The credulous encoded semantics, on the other hand, assigns the maximum truth value across all models, capturing the union (or supremum) of all possible interpretations. Both semantics yield a unique labelling for each EHSAF under their respective reasoning modes.
	
	\begin{defn}\label{normalencoding}
		For a given $EHSAF = (\mathbf{A}, \mathbf{R_a}, \mathbf{R_e}, \mathbf{U}, \mathbf{P})$, the \emph{normal encoding} of the EHSAF w.r.t. a $\mathcal{PL}$ is an encoding $ec_{n}$ such that
		\begin{equation*}
			ec_{n}(EHSAF)=\bigwedge_{\beta\in \mathbb{U}}(\beta\leftrightarrow(\bigwedge_{\mathbf{k}_{S_i}^\beta\in \mathbf{R_a}}\neg(\mathbf{k}_{S_i}^{\beta}\wedge\bigwedge_{b_{ij}\in S_i}b_{ij}))\wedge(\bigvee_{\mathbf{t}_{T_i}^\beta\in \mathbf{R_e}}(\mathbf{t}_{T_i}^{\beta}\wedge\bigwedge_{c_{ij}\in T_i}c_{ij}))).
		\end{equation*}
		The normal encoded semantics w.r.t. a $\mathcal{PL}$ is a encoded semantics induced by $ec_{n}$ and $\mathcal{PL}$, denoted by $\mathfrak{LS}_{ec_{n}}^{\mathcal{PL}}$.
	\end{defn}
		\subsection{Discrete Encoded Semantics}

\begin{thm}\label{acpl3}
	For any EHSAF $\mathcal{F}=(\mathbf{A},\mathbf{R}_a,\mathbf{R}_e,\mathbf{U},\mathbf{P})$, 
	\[
	\mathfrak{LS}_{ec_n}^{\mathcal{PL}_3^L}(EHSAF)= \mathfrak{LS}_{ac}(EHSAF).
	\]
\end{thm}

\begin{proof}
	Any assignment on auxiliary elements in $\mathbf{U}$ is trivially consistent between adjacent complete labellings and models of $ec_n(\mathcal{F})$ in $\mathcal{PL}_3^L$. We need to check that for a given assignment $\|\cdot\|$ and for each $\beta\in{U}$, the value of $\beta$ satisfies the adjacent complete labelling conditions (Definition~\ref{comlab}) iff
	\[
	\left\|\beta \leftrightarrow \left( \bigwedge_{\mathbf{k}_{S}^{\beta}\in\mathbf{R}_a} \neg(\mathbf{k}_{S}^{\beta}\wedge\bigwedge_{s\in S}s) \;\wedge\; \bigvee_{\mathbf{t}_{T}^{\beta}\in\mathbf{R}_e} (\mathbf{t}_{T}^{\beta}\wedge\bigwedge_{t\in T}t) \right)\right\|=1
	\]
	in $\mathcal{PL}_3^L$. Denote the right-hand side of the equivalence by $\Phi_\beta$. We discuss three cases.
	
	\begin{itemize}
		\item \textbf{Case 1, $\|\beta\|=1$.}
		\\
		$\|\beta\|=1$ by adjacent complete labelling
		\\
		$\Longleftrightarrow$ $[\forall \mathbf{k}_{S}^{\beta}\in\mathbf{R}_a: \|\mathbf{k}_{S}^{\beta}\|=0$ or $\exists s\in S: \|s\|=0]$ and $[\exists \mathbf{t}_{T}^{\beta}\in\mathbf{R}_e: \|\mathbf{t}_{T}^{\beta}\|=1$ and $\forall t\in T: \|t\|=1]$ by Definition~\ref{comlab}
		\\
		$\Longleftrightarrow$ $[\forall \mathbf{k}_{S}^{\beta}\in\mathbf{R}_a: \|\mathbf{k}_{S}^{\beta}\wedge\bigwedge_{s\in S}s\|=0]$ and $[\exists \mathbf{t}_{T}^{\beta}\in\mathbf{R}_e: \|\mathbf{t}_{T}^{\beta}\wedge\bigwedge_{t\in T}t\|=1]$ in $\mathcal{PL}_3^L$
		\\
		$\Longleftrightarrow$ $[\forall \mathbf{k}_{S}^{\beta}\in\mathbf{R}_a: \|\neg(\mathbf{k}_{S}^{\beta}\wedge\bigwedge_{s\in S}s)\|=1]$ and $[\|\bigvee_{\mathbf{t}_{T}^{\beta}\in\mathbf{R}_e}(\mathbf{t}_{T}^{\beta}\wedge\bigwedge_{t\in T}t)\|=1]$ in $\mathcal{PL}_3^L$
		\\
		$\Longleftrightarrow$ $\|\bigwedge_{\mathbf{k}_{S}^{\beta}\in\mathbf{R}_a}\neg(\mathbf{k}_{S}^{\beta}\wedge\bigwedge_{s\in S}s)\|=1$ and $\|\bigvee_{\mathbf{t}_{T}^{\beta}\in\mathbf{R}_e}(\mathbf{t}_{T}^{\beta}\wedge\bigwedge_{t\in T}t)\|=1$ in $\mathcal{PL}_3^L$
		\\
		$\Longleftrightarrow$ $\|\Phi_\beta\|=1$ in $\mathcal{PL}_3^L$
		\\
		$\Longleftrightarrow$ $\|\beta\leftrightarrow \Phi_\beta\|=1$ in $\mathcal{PL}_3^L$.
		
		\item \textbf{Case 2, $\|\beta\|=0$.}
		\\
		$\|\beta\|=0$ by adjacent complete labelling
		\\
		$\Longleftrightarrow$ $[\exists \mathbf{k}_{S}^{\beta}\in\mathbf{R}_a: \|\mathbf{k}_{S}^{\beta}\|=1$ and $\forall s\in S: \|s\|=1]$ or $[\forall \mathbf{t}_{T}^{\beta}\in\mathbf{R}_e: \|\mathbf{t}_{T}^{\beta}\|=0$ or $\exists t\in T: \|t\|=0]$ by Definition~\ref{comlab}
		\\
		$\Longleftrightarrow$ $[\exists \mathbf{k}_{S}^{\beta}\in\mathbf{R}_a: \|\mathbf{k}_{S}^{\beta}\wedge\bigwedge_{s\in S}s\|=1]$ or $[\forall \mathbf{t}_{T}^{\beta}\in\mathbf{R}_e: \|\mathbf{t}_{T}^{\beta}\wedge\bigwedge_{t\in T}t\|=0]$ in $\mathcal{PL}_3^L$
		\\
		$\Longleftrightarrow$ $[\exists \mathbf{k}_{S}^{\beta}\in\mathbf{R}_a: \|\neg(\mathbf{k}_{S}^{\beta}\wedge\bigwedge_{s\in S}s)\|=0]$ or $[\|\bigvee_{\mathbf{t}_{T}^{\beta}\in\mathbf{R}_e}(\mathbf{t}_{T}^{\beta}\wedge\bigwedge_{t\in T}t)\|=0]$ in $\mathcal{PL}_3^L$
		\\
		$\Longleftrightarrow$ $\|\bigwedge_{\mathbf{k}_{S}^{\beta}\in\mathbf{R}_a}\neg(\mathbf{k}_{S}^{\beta}\wedge\bigwedge_{s\in S}s)\|=0$ or $\|\bigvee_{\mathbf{t}_{T}^{\beta}\in\mathbf{R}_e}(\mathbf{t}_{T}^{\beta}\wedge\bigwedge_{t\in T}t)\|=0$ in $\mathcal{PL}_3^L$
		\\
		$\Longleftrightarrow$ $\|\Phi_\beta\|=0$ in $\mathcal{PL}_3^L$
		\\
		$\Longleftrightarrow$ $\|\beta\leftrightarrow \Phi_\beta\|=1$ in $\mathcal{PL}_3^L$.
		
		\item \textbf{Case 3, $\|\beta\|=\frac12$.}
		\\
		$\|\beta\|=\frac12$ by adjacent complete labelling
		\\
		$\Longleftrightarrow$ $\|\beta\|\neq 1$ and $\|\beta\|\neq 0$ under adjacent complete labelling
		\\
		$\Longleftrightarrow$ $\|\Phi_\beta\|\neq 1$ and $\|\Phi_\beta\|\neq 0$ in $\mathcal{PL}_3^L$ by Case 1 and Case 2
		\\
		$\Longleftrightarrow$ $\|\Phi_\beta\|=\frac12$ in $\mathcal{PL}_3^L$
		\\
		$\Longleftrightarrow$ $\|\beta\leftrightarrow \Phi_\beta\|=1$ in $\mathcal{PL}_3^L$.
	\end{itemize}
	
	From the three cases above, for a given assignment $\|\cdot\|$, the value $\|\beta\|$ of any $\beta\in{U}$ satisfies the adjacent complete labelling conditions iff
	\[
	\left\|\beta \leftrightarrow \left( \bigwedge_{\mathbf{k}_{S}^{\beta}\in\mathbf{R}_a} \neg(\mathbf{k}_{S}^{\beta}\wedge\bigwedge_{s\in S}s) \;\wedge\; \bigvee_{\mathbf{t}_{T}^{\beta}\in\mathbf{R}_e} (\mathbf{t}_{T}^{\beta}\wedge\bigwedge_{t\in T}t) \right)\right\|=1
	\]
	in $\mathcal{PL}_3^L$.
	
	Thus, an assignment $\|\cdot\|$ of the EHSAF is an adjacent complete labelling iff the assignment $\|\cdot\|$ is a model of $ec_n(\mathcal{F})$ in $\mathcal{PL}_3^L$.
\end{proof}

\begin{cor}
	For every EHSAF \(\mathcal{F}\),
	\[
	\mathfrak{L}\mathfrak{S}_{Eq_{3}^{EH}}(\mathcal{F}) = \mathfrak{L}\mathfrak{S}_{ac}(\mathcal{F}) = \mathfrak{L}\mathfrak{S}_{ec_n}^{\mathcal{PL}_3^L}(\mathcal{F}).
	\]
\end{cor}
\begin{proof}
	It is obtained immediately from Theorem~\ref{eqofeqandadcom} and the above theorem.
\end{proof}	

\begin{thm}
	For every EHSAF \(\mathcal{F}\),
	\[
	\mathfrak{L}\mathfrak{S}_{as}(\mathcal{F}) = \mathfrak{L}\mathfrak{S}_{ec_n}^{\mathcal{PL}_2}(\mathcal{F}).
	\]
\end{thm}

\begin{proof}
	Let \(\mathcal{F}\) be an arbitrary EHSAF. By Theorem \ref{acpl3}, the normal encoding \(ec_n\) exactly captures the adjacent complete semantics in the three-valued logic \(\mathcal{PL}_3^L\):
	\[
	\mathfrak{L}\mathfrak{S}_{ec_n}^{\mathcal{PL}_3^L}(\mathcal{F}) = \mathfrak{L}\mathfrak{S}_{ac}(\mathcal{F}).
	\]
	The encoding formula \(ec_n(\mathcal{F})\) uses only the connectives \(\neg,\wedge,\vee,\leftrightarrow\), whose truth tables on the set \(\{0,1\}\) coincide with those of classical two-valued logic \(\mathcal{PL}_2\). Therefore, restricting the universe of assignments to \(\{0,1\}\), we have
	\[
	\mathfrak{L}\mathfrak{S}_{ec_n}^{\mathcal{PL}_2}(\mathcal{F}) 
	= \bigl\{\|\cdot\| \in \mathfrak{L}\mathfrak{S}_{ec_n}^{\mathcal{PL}_3^L}(\mathcal{F}) \;\big|\; \|\cdot\| : \mathbf{U}\to \{0,1\}\bigr\}.
	\]
	By Definition \ref{staetc}, a labelling is adjacent stable exactly when it is adjacent complete and assigns only the values \(0\) or \(1\) to every element of \(\mathbf{U}\). Hence
	\[
	\mathfrak{L}\mathfrak{S}_{as}(\mathcal{F})
	= \bigl\{\|\cdot\| \in \mathfrak{L}\mathfrak{S}_{ac}(\mathcal{F}) \;\big|\; \|\cdot\| : \mathbf{U}\to \{0,1\}\bigr\}.
	\]
	Combining the two equalities yields
	\[
	\mathfrak{L}\mathfrak{S}_{ec_n}^{\mathcal{PL}_2}(\mathcal{F}) = \mathfrak{L}\mathfrak{S}_{as}(\mathcal{F}).
	\]
	Since \(\mathcal{F}\) was arbitrary, the theorem follows. \qedhere
\end{proof}

	\subsection{General Fuzzy Encoded Semantics and Its Core Properties}
	\subsubsection{General Fuzzy Encoded Semantics}

	\paragraph{Continuous Fuzzy Operator-Based Equational (CFOE) Semantics}
	
	\begin{defn}[CFOE System for EHSAF]\label{CFOE}
		Let \(N^*:[0,1]\to[0,1]\) be a continuous negation and \(\otimes:[0,1]^2\to[0,1]\) be a continuous t-norm. 
		For an EHSAF \(\mathcal{F}=(\mathbf{A},\mathbf{R}_{a},\mathbf{R}_{e},\mathbf{U},\mathbf{P})\), denote by 
		\(\mathcal{K}_\beta=\{\mathbf{k}_{S}^{\beta}\in\mathbf{R}_{a}\}\) the set of attacks targeting \(\beta\), and 
		\(\mathcal{T}_\beta=\{\mathbf{t}_{T}^{\beta}\in\mathbf{R}_{e}\}\) the set of supports targeting \(\beta\). Let \(\|\cdot\|:\mathbf{U}\to[0,1]\) be an assignment. 
		
		A \emph{continuous fuzzy operator-based equational (CFOF) system} \(eq_{[0,1]}^{*,\otimes}\) for \(\mathcal{F}\) is defined by the auxiliary-element conditions
		\[
		\|\bot\|=0,\qquad \|\top\|=1,\qquad \|\alpha\|=1 (\forall\alpha \in (\mathbf{R_e}\cup\mathbf{R_e})\setminus(R_a\cup R_e))
		\]
		and for every \(\beta\in \mathbb{U}\),
		\begin{equation}\label{star2}
			\|\beta\|
			\;=\;
			\left(
			\bigotimes_{\mathbf{k}_{S}^{\beta}\in\mathcal{K}_\beta}
			N^*\!\left(
			\|\mathbf{k}_{S}^{\beta}\|\otimes \bigotimes_{s\in S}\|s\|
			\right)
			\right)
			\;\otimes\;
			\left(
			N^*\!\left(
			\bigotimes_{\mathbf{t}_{T}^{\beta}\in\mathcal{T}_\beta}
			N^*\!\left(
			\|\mathbf{t}_{T}^{\beta}\|\otimes \bigotimes_{t\in T}\|t\|
			\right)
			\right)
			\right). 
		\end{equation}
		Here \(\bigotimes\) denotes the finite iteration of the t-norm \(\otimes\), and the second factor uses the standard t-conorm defined by \(x\oplus y = N^*(N^*(x)\otimes N^*(y))\), so that \(\oplus_{t\in\mathcal{T}_\beta}(\|\mathbf{t}_{T}^{\beta}\|\otimes\bigotimes_{t\in T}\|t\|)\) is expressed as \(N^*(\bigotimes_{t\in\mathcal{T}_\beta} N^*(\|\mathbf{t}_{T}^{\beta}\|\otimes\bigotimes_{t\in T}\|t\|))\).
		
		An assignment \(\|\cdot\|:\mathbf{U}\to[0,1]\) satisfying all equations is called a \emph{solution} of \(eq_{[0,1]}^{*,\otimes}\), denoted \(\|\cdot\|\models_{eq_{[0,1]}^{*,\otimes}}\mathcal{F}\).
		
		The \emph{continuous fuzzy operator-based equational (CFOE) semantics} induced by \(eq_{[0,1]}^{*,\otimes}\) is the function
		\[
		\mathfrak{L}\mathfrak{S}_{Eq_{[0,1]}^{*,\otimes}}:\mathcal{EHSAF}\to 2^{\mathcal{LAB}},
		\quad
		\mathfrak{L}\mathfrak{S}_{Eq_{[0,1]}^{*,\otimes}}(\mathcal{F})
		=
		\{\|\cdot\| \mid \|\cdot\|\models_{eq_{[0,1]}^{*,\otimes}}\mathcal{F}\}.
		\]
	\end{defn}
	
	\paragraph{Continuous Fuzzy Normal Encoded (CFNE) Semantics}
	
	\begin{defn}[CFNE Semantics for EHSAF]\label{CFNE}
		Let \(\mathcal{PL}_{[0,1]}^{*,\otimes}\) be a fuzzy propositional logic system equipped with a continuous negation \(N^*\), a continuous t-norm \(\otimes\), and its residuated implication. 
		The \emph{continuous fuzzy normal encoded (CFNE) semantics} is the encoded semantics induced by the normal encoding \(ec_n\) (Definition~\ref{normalencoding}) and \(\mathcal{PL}_{[0,1]}^{*,\otimes}\), denoted by
		\[
		\mathfrak{L}\mathfrak{S}_{ec_n}^{\mathcal{PL}_{[0,1]}^{*,\otimes}}.
		\]
		That is, for every EHSAF \(\mathcal{F}\),
		\[
		\mathfrak{L}\mathfrak{S}_{ec_n}^{\mathcal{PL}_{[0,1]}^{*,\otimes}}(\mathcal{F})
		=
		\{\|\cdot\| : \mathbf{U}\to[0,1] \mid \|ec_n(\mathcal{F})\|_{\mathcal{PL}_{[0,1]}^{*,\otimes}}=1\}.
		\]
	\end{defn}
	
	\paragraph{Equivalence between CFOE and CFNE Semantics}
	
	\begin{thm}[Equivalence of CFOE and CFNE Semantics]\label{fuzzyequivalence}
		For every EHSAF \(\mathcal{F}\in\mathcal{EHSAF}\),
		\[
		\mathfrak{L}\mathfrak{S}_{Eq_{[0,1]}^{*,\otimes}}(\mathcal{F})
		=
		\mathfrak{L}\mathfrak{S}_{ec_n}^{\mathcal{PL}_{[0,1]}^{*,\otimes}}(\mathcal{F}).
		\]
	\end{thm}
	
	\begin{proof}
		Fix an EHSAF \(\mathcal{F}=(\mathbf{A},\mathbf{R}_{a},\mathbf{R}_{e},\mathbf{U},\mathbf{P})\) and an assignment \(\|\cdot\|:\mathbf{U}\to[0,1]\). By Definition~\ref{CFNE}, \(\|\cdot\|\in\mathfrak{L}\mathfrak{S}_{ec_n}^{\mathcal{PL}_{[0,1]}^{*,\otimes}}(\mathcal{F})\) iff
		\[
		\|ec_n(\mathcal{F})\| = 1
		\]
		in \(\mathcal{PL}_{[0,1]}^{*,\otimes}\). Since \(ec_n(\mathcal{F}) = \bigwedge_{\beta\in \mathbb{U}} \Phi_\beta\), where
		\[
		\Phi_\beta \stackrel{\text{def}}{=}
		\beta \leftrightarrow
		\left(
		\bigwedge_{\mathbf{k}_{S}^{\beta}\in\mathcal{K}_\beta}
		\neg\left(\mathbf{k}_{S}^{\beta}\wedge \bigwedge_{s\in S}s\right)
		\right)
		\wedge
		\left(
		\bigvee_{\mathbf{t}_{T}^{\beta}\in\mathcal{T}_\beta}
		\left(\mathbf{t}_{T}^{\beta}\wedge \bigwedge_{t\in T}t\right)
		\right),
		\]
		we have \(\|\bigwedge_{\beta}\Phi_\beta\|=1\) iff for every \(\beta\in \mathbb{U}\), \(\|\Phi_\beta\|=1\). In any fuzzy logic with residuated implication, \(\|a\leftrightarrow b\|=1\) iff \(\|a\|=\|b\|\). Hence for each \(\beta\),
		\[
		\|\Phi_\beta\|=1
		\;\Longleftrightarrow\;
		\|\beta\|
		=
		\left\|
		\left(
		\bigwedge_{\mathbf{k}_{S}^{\beta}\in\mathcal{K}_\beta}
		\neg\left(\mathbf{k}_{S}^{\beta}\wedge \bigwedge_{s\in S}s\right)
		\right)
		\wedge
		\left(
		\bigvee_{\mathbf{t}_{T}^{\beta}\in\mathcal{T}_\beta}
		\left(\mathbf{t}_{T}^{\beta}\wedge \bigwedge_{t\in T}t\right)
		\right)
		\right\|.
		\]
		
		Now interpret the connectives in \(\mathcal{PL}_{[0,1]}^{*,\otimes}\):\(\|a\wedge b\| = \|a\|\otimes\|b\|\), \(\|\neg a\| = N^*(\|a\|)\), \(\|a\vee b\| = N^*(N^*(\|a\|)\otimes N^*(\|b\|))\) (the t-conorm), and finite conjunctions/disjunctions are iterated accordingly.
		
		Thus the right-hand side of the equality above becomes
		\[
		\left(
		\bigotimes_{\mathbf{k}_{S}^{\beta}\in\mathcal{K}_\beta}
		N^*\!\left(
		\|\mathbf{k}_{S}^{\beta}\|\otimes \bigotimes_{s\in S}\|s\|
		\right)
		\right)
		\;\otimes\;
		\left(
		N^*\!\left(
		\bigotimes_{\mathbf{t}_{T}^{\beta}\in\mathcal{T}_\beta}
		N^*\!\left(
		\|\mathbf{t}_{T}^{\beta}\|\otimes \bigotimes_{t\in T}\|t\|
		\right)
		\right)
		\right),
		\]
		which is exactly the right-hand side (RHS) of Equation \eqref{star2} in Definition~\ref{CFOE}. Therefore,
		\[
		\|\Phi_\beta\|=1
		\;\Longleftrightarrow\;
		\|\beta\| = \text{RHS of Equation \eqref{star2}.} 
		\]
		Since this holds for every \(\beta\in \mathbb{U}\) and any assignment is trivially consistent among auxiliary elements, we have \(\|\cdot\|\models_{\mathcal{PL}_{[0,1]}^{*,\otimes}}ec_n(\mathcal{F})\) iff \(\|\cdot\|\models_{eq_{[0,1]}^{*,\otimes}}\mathcal{F}\). Hence the two semantics coincide.
	\end{proof}

	\subsubsection{Core Properties of CFNE Semantics}\label{property}

		Let $\mathcal{F} = (\mathbf{A}, \mathbf{R}_{a}, \mathbf{R}_{e}, \mathbf{U}, \mathbf{P})$ be an EHSAF. For each \(\beta\in \mathbb{U}\), let its attacking set and supporting set be
	\[
	\mathcal{K}_\beta = \{\mathbf{k}_{S_1}^{\beta}, \ldots, \mathbf{k}_{S_{n_a}}^{\beta}\},\qquad
	\mathcal{T}_\beta = \{\mathbf{t}_{T_1}^{\beta}, \ldots, \mathbf{t}_{T_{n_s}}^{\beta}\}.
	\]
	According to CFOE systems of EHSAFs, define the corresponding aggregation function
	\[
	f_\beta:\ [0,1]^{m_\beta} \longrightarrow [0,1]
	\]
	by
	\[
	f_\beta(\mathbf{x})
	\;=\;
	\underbrace{\bigotimes_{\mathbf{k}_{S}^{\beta}\in\mathcal{K}_\beta}
		N^*\!\left(
		\|\mathbf{k}_{S}^{\beta}\| \otimes \bigotimes_{s\in S} \|s\|
		\right)}_{\text{attack part } A}
	\;\otimes\;
	\underbrace{N^*\!\left(
		\bigotimes_{\mathbf{t}_{T}^{\beta}\in\mathcal{T}_\beta}
		N^*\!\left(
		\|\mathbf{t}_{T}^{\beta}\| \otimes \bigotimes_{t\in T} \|t\|
		\right)
		\right)}_{\text{support part } B},
	\]
	where \(\mathbf{x}\) ranges over all variables occurring in the source sets and the attack/support names relevant to \(\beta\), $f_\beta$ is decomposed into the attack part $A$ and the support part $B$, $\|\cdot\| : \mathbf{U} \to [0,1]$ is an arbitrary assignment, all \(\otimes\) are finite iterations of the continuous t-norm, and $m_\beta =|\mathcal{K}_\beta| +|\mathcal{T}_\beta| + \sum_{i=1}^{n_a}|\mathbf{k}_{S_{i}}^{\beta}| + \sum_{j=1}^{n_s} |\mathbf{t}_{T_j}^{\beta}|$. We characterize the properties of CFNE semantics via an analysis of the function \(f_\beta\), without discussing trivial constant assignments of auxiliary elements.
	
	\begin{thm}[Continuity of \(f_\beta\)]\label{continuity}
		For every \(\beta\in \mathbb{U}\), the function \(f_\beta\) is continuous on its domain \([0,1]^{m_\beta}\).
	\end{thm}
	\begin{proof}
		The continuity of \(N^*\) and \(\otimes\) is assumed. Since \(f_\beta\) is obtained by finite compositions and finite iterations of these continuous operations, it is continuous on the entire domain.
	\end{proof}
		
	\begin{thm}[Commutativity of \(f_\beta\)]
		The value of \(f_\beta\) is invariant under:
		\begin{enumerate}
			\item any permutation of the variables inside each source set \(S_i\) or \(T_j\);
			\item any permutation of the attack groups \(\{\mathbf{k}_{S_i}^{\beta}\}_{i=1}^{n_a}\);
			\item any permutation of the support groups \(\{\mathbf{t}_{T_j}^{\beta}\}_{j=1}^{n_s}\).
		\end{enumerate}
	\end{thm}
	\begin{proof}
		This follows directly from the commutativity and associativity of the t-norm \(\otimes\). Finite t-norm conjunctions are invariant under permutation of their arguments. Hence each inner conjunction, each attack term, each support term, and the outer aggregation are all permutation-invariant.
	\end{proof}
		
	\begin{thm}[Monotonicity of \(f_\beta\)]
		The function \(f_\beta\) is monotone with respect to the truth value of each element in \(\mathbb{U}\), with the following direction:
		\begin{itemize}
			\item \emph{Non-increasing} with respect to every element that appears in the source set of an attack against \(\beta\), or is an attack name attacking \(\beta\);
			\item \emph{Non-decreasing} with respect to every element that appears in the source set of a support of \(\beta\), or is a support name supporting \(\beta\).
		\end{itemize}
	\end{thm}
	
	\begin{proof}
		We analyze the attack part \(A\) and the support part \(B\) separately, based on the monotonicity of t-norms and the antitonicity of negation \(N^*\).
		
		\medskip
		\noindent\textbf{Attack part:}
		Recall \(A = \bigotimes_{\mathbf{k}_{S}^{\beta}\in\mathcal{K}_\beta} N^*\bigl(\|\mathbf{k}_{S}^{\beta}\| \otimes \bigotimes_{s\in S} \|s\|\bigr)\).
		Take any attack-related element \(u\): \(u\) is either an attack name \(\mathbf{k}_{S}^{\beta}\) or an element \(s\in S\) in the source of some attack against \(\beta\).
		
		The inner conjunction \(\|\mathbf{k}_{S}^{\beta}\| \otimes \bigotimes_{s\in S} \|s\|\) is non-decreasing in \(\|u\|\), since t-norms are monotone non-decreasing in each factor.
		As \(N^*\) is a negation (antitone), \(N^*\bigl(\|\mathbf{k}_{S}^{\beta}\| \otimes \bigotimes_{s\in S} \|s\|\bigr)\) is non-increasing in \(\|u\|\).
		The outer t-norm is non-decreasing in each of its factors, so the whole attack part \(A\) is non-increasing in \(\|u\|\).
		
		\medskip
		\noindent\textbf{Support part:}
		Recall \(B = N^*\bigl(\bigotimes_{\mathbf{t}_{T}^{\beta}\in\mathcal{T}_\beta} N^*\bigl(\|\mathbf{t}_{T}^{\beta}\| \otimes \bigotimes_{t\in T} \|t\|\bigr)\bigr)\).
		Take any support-related element \(v\): \(v\) is either a support name \(\mathbf{t}_{T}^{\beta}\) or an element \(t\in T\) in the source of some support for \(\beta\).
		
		The inner conjunction \(\|\mathbf{t}_{T}^{\beta}\| \otimes \bigotimes_{t\in T} \|t\|\) is non-decreasing in \(\|v\|\).
		Applying the inner negation \(N^*\), the term \(N^*\bigl(\|\mathbf{t}_{T}^{\beta}\| \otimes \bigotimes_{t\in T} \|t\|\bigr)\) becomes non-increasing in \(\|v\|\).
		The inner t-norm is non-decreasing in each factor, so the product \(P = \bigotimes_{\mathbf{t}_{T}^{\beta}\in\mathcal{T}_\beta} N^*(\cdots)\) is non-increasing in \(\|v\|\).
		Applying the outer negation \(N^*\) reverses the order again, hence \(B = N^*(P)\) is non-decreasing in \(\|v\|\).
		
		\medskip
		Finally, the outer t-norm \(f_\beta = A \otimes B\) is non-decreasing in both \(A\) and \(B\).
		Therefore:
		\begin{itemize}
			\item for any attack-related element \(u\), \(f_\beta\) is non-increasing in \(\|u\|\);
			\item for any support-related element \(v\), \(f_\beta\) is non-decreasing in \(\|v\|\).
		\end{itemize}
		This completes the proof.
	\end{proof}
	\begin{rmk}
		The monotonicity pattern reflects the intuitive roles: attacks decrease the acceptability of \(\beta\), while evidential supports increase it. 
	\end{rmk}
We define the following conditions for a given $\beta\in \mathbb{U}$:
	\begin{itemize}
		\item[(A)] \emph{All attacks are ineffective:} for every $\mathbf{k}_{S}^{\beta}\in\mathcal{K}_{\beta}$, either $\|\mathbf{k}_{S}^{\beta}\|=0$ or $\exists s\in S$ with $\|s\|=0$.
		\item[(E)] \emph{At least one support is effective:} $\exists \mathbf{t}_{T}^{\beta}\in\mathcal{T}_{\beta}$ such that $\|\mathbf{t}_{T}^{\beta}\|=1$ and $\forall t\in T$, $\|t\|=1$.
		\item[(B)] \emph{There exists an effective attack:} $\exists \mathbf{k}_{S}^{\beta}\in\mathcal{K}_{\beta}$ such that $\|\mathbf{k}_{S}^{\beta}\|=1$ and $\forall s\in S$, $\|s\|=1$.
		\item[(I)] \emph{All supports are ineffective:} for every $\mathbf{t}_{T}^{\beta}\in\mathcal{T}_{\beta}$, either $\|\mathbf{t}_{T}^{\beta}\|=0$ or $\exists t\in T$ with $\|t\|=0$.
	\end{itemize}
	\begin{thm}[Boundary Conditions of \(f_\beta\)]
	For every \(\beta\in \mathbb{U}\), the following two extremal cases hold:
	
	\begin{enumerate}
		\item \emph{Rejection boundary:} If $(B)$ or $(I)$, then $f_\beta = 0$.
		
		\item \emph{Acceptance boundary:} If $(A)$ and $(E)$, then $f_\beta = 1$.
	\end{enumerate}
\end{thm}

\begin{proof}
	We prove case (1) in two subcases:
	
	\noindent\textbf{Subcase 1.1 (There exists an effective attack):}
	Let \(\mathbf{k}_{S}^{\beta}\in\mathcal{K}_\beta\) be an effective attack, i.e., \(\|\mathbf{k}_{S}^{\beta}\|=1\) and \(\|s\|=1\) for all \(s\in S\).
	Then \(\|\mathbf{k}_{S}^{\beta}\| \otimes \bigotimes_{s\in S} \|s\| = 1\).
	Since \(N^*\) is a negation, \(N^*(1)=0\).
	Hence the corresponding factor in \(A\) equals \(0\).
	By the annihilation property of t-norms (\(0\otimes a = 0\) for all \(a\in[0,1]\)), the attack part \(A = 0\).
	Thus \(f_\beta = 0 \otimes B = 0\).
	
	\noindent\textbf{Subcase 1.2 (All supports are ineffective):}
	For every \(\mathbf{t}_{T}^{\beta}\in\mathcal{T}_\beta\), since it is ineffective, either \(\|\mathbf{t}_{T}^{\beta}\|=0\) or there exists \(t\in T\) with \(\|t\|=0\).
	In either case, \(\|\mathbf{t}_{T}^{\beta}\| \otimes \bigotimes_{t\in T} \|t\| = 0\), so \(N^*\!\left(\|\mathbf{t}_{T}^{\beta}\| \otimes \bigotimes_{t\in T} \|t\|\right) = N^*(0) = 1\).
	Therefore the inner t-norm \(\bigotimes_{\mathbf{t}_{T}^{\beta}\in\mathcal{T}_\beta} N^*(\cdots) = \bigotimes 1 = 1\).
	Applying the outer negation yields \(B = N^*(1) = 0\).
	Thus \(f_\beta = A \otimes 0 = 0\).
	
	\medskip
	Now we prove case (2):
	
	For every attack \(\mathbf{k}_{S}^{\beta}\in\mathcal{K}_\beta\), since it is ineffective, either \(\|\mathbf{k}_{S}^{\beta}\|=0\) or there exists \(s\in S\) with \(\|s\|=0\).
	Hence \(\|\mathbf{k}_{S}^{\beta}\| \otimes \bigotimes_{s\in S} \|s\| = 0\), and \(N^*(0)=1\).
	It follows that every factor in \(A\) equals \(1\), so \(A = \bigotimes 1 = 1\).
	
	For the support part, let \(\mathbf{t}_{T}^{\beta}\in\mathcal{T}_\beta\) be an effective support, i.e., \(\|\mathbf{t}_{T}^{\beta}\|=1\) and \(\|t\|=1\) for all \(t\in T\).
	Then \(\|\mathbf{t}_{T}^{\beta}\| \otimes \bigotimes_{t\in T} \|t\| = 1\), so \(N^*(1)=0\).
	The inner t-norm \(\bigotimes_{\mathbf{t}_{T}^{\beta}\in\mathcal{T}_\beta} N^*(\cdots)\) contains at least one factor \(0\), thus the whole inner product equals \(0\).
	Consequently, \(B = N^*(0) = 1\).
	
	Combining the two parts, \(f_\beta = A \otimes B = 1\otimes 1 = 1\).
\end{proof}

These two boundary conditions are analogous to the acceptance and rejection conditions for adjacent complete labelling semantics under single-direction evaluation over the extremal truth values 0 and 1. For example, an element is accepted if all attacks are blocked and at least one support is active.

If the t‑norm $\otimes$ employed in the function $f_{\beta}$ is zero‑divisor‑free, then the converse implications of the boundary properties hold as well.
\begin{thm}[CFOE Equation Trifurcation]\label{thm:trifurcation}
	Assume that the t-norm $\otimes$ in the function $f_{\beta}$ is zero‑divisor‑free (i.e., $x\otimes y=0 \Rightarrow x=0$ or $y=0$). For every $\beta\in \mathbb{U}$, the following equivalences hold:
	\begin{enumerate}
		\item $f_{\beta} = 1 \;\Longleftrightarrow\; (A) \text{ and } (E)$.
		\item $f_{\beta} = 0 \;\Longleftrightarrow\; (B) \text{ or } (I)$.
		\item $0 < f_{\beta} < 1 \;\Longleftrightarrow\; \text{neither of the above holds}$.
	\end{enumerate}
\end{thm}

\begin{proof}
	Write $f_{\beta} = A_{\beta} \otimes B_{\beta}$, where
	\[
	A_{\beta} = \bigotimes_{\mathbf{k}_{S}^{\beta}\in \mathcal{K}_{\beta}} 
	N^{*}\!\left(\|\mathbf{k}_{S}^{\beta}\| \otimes \bigotimes_{s\in S} \|s\|\right),
	\quad
	B_{\beta} = N^{*}\!\left(
	\bigotimes_{\mathbf{t}_{T}^{\beta}\in \mathcal{T}_{\beta}} 
	N^{*}\!\left(\|\mathbf{t}_{T}^{\beta}\| \otimes \bigotimes_{t\in T} \|t\|\right)
	\right).
	\]
	
	We use the facts: $N^{*}(0)=1$, $N^{*}(1)=0$, and $\otimes$ is monotone, associative, commutative, has identity 1 and zero 0, and is zero‑divisor‑free.
	
	\paragraph{1. $f_{\beta}=1$ iff (A) and (E).}
	Since $1\otimes x = x$ and $x\otimes y=1$ implies $x=y=1$ (as $x,y\le 1$), we have $f_{\beta}=1$ iff $A_{\beta}=B_{\beta}=1$.
	
	Now $A_{\beta}=1$ iff every factor in the product is 1, i.e., for every attack,
	$N^{*}(\|\mathbf{k}_{S}^{\beta}\| \otimes \bigotimes_{s\in S}\|s\|) = 1$, which is equivalent to $\|\mathbf{k}_{S}^{\beta}\| \otimes \bigotimes_{s\in S}\|s\| = 0$. By zero‑divisor‑freeness, this is equivalent to either $\|\mathbf{k}_{S}^{\beta}\|=0$ or some $\|s\|=0$, i.e., condition (A).
	
	$B_{\beta}=1$ iff $N^{*}(\bigotimes_{\mathbf{t}_{T}^{\beta}} N^{*}(\|\mathbf{t}_{T}^{\beta}\| \otimes \bigotimes_{t\in T}\|t\|)) = 1$, i.e., the inner product equals 0. By zero‑divisor‑freeness, the product of numbers in $[0,1]$ is 0 iff at least one factor is 0. Thus $B_{\beta}=1$ iff there exists a support $\mathbf{t}_{T}^{\beta}$ such that $N^{*}(\|\mathbf{t}_{T}^{\beta}\| \otimes \bigotimes_{t\in T}\|t\|) = 0$, i.e., $\|\mathbf{t}_{T}^{\beta}\| \otimes \bigotimes_{t\in T}\|t\| = 1$, which is equivalent to $\|\mathbf{t}_{T}^{\beta}\|=1$ and $\|t\|=1$ for all $t\in T$, i.e., condition (E). Hence $f_{\beta}=1$ iff (A) and (E).
	
	\paragraph{2. $f_{\beta}=0$ iff (B) or (I).}
	Since $x\otimes y=0$ iff $x=0$ or $y=0$ (zero‑divisor‑free), $f_{\beta}=0$ iff either $A_{\beta}=0$ or $B_{\beta}=0$.
	
	$A_{\beta}=0$ iff there exists an attack with $N^{*}(\|\mathbf{k}_{S}^{\beta}\| \otimes \bigotimes_{s\in S}\|s\|) = 0$, i.e., $\|\mathbf{k}_{S}^{\beta}\| \otimes \bigotimes_{s\in S}\|s\| = 1$, which is equivalent to $\|\mathbf{k}_{S}^{\beta}\|=1$ and $\|s\|=1$ for all $s$, i.e., condition (B).
	
	$B_{\beta}=0$ iff $N^{*}(\bigotimes_{\mathbf{t}_{T}^{\beta}} N^{*}(\|\mathbf{t}_{T}^{\beta}\| \otimes \bigotimes_{t\in T}\|t\|)) = 0$, i.e., the inner product equals 1. This is equivalent to every factor in the product to be 1, i.e., for every support, $N^{*}(\|\mathbf{t}_{T}^{\beta}\| \otimes \bigotimes_{t\in T}\|t\|) = 1$, equivalently $\|\mathbf{t}_{T}^{\beta}\| \otimes \bigotimes_{t\in T}\|t\| = 0$, which by zero‑divisor‑freeness is equivalent to either $\|\mathbf{t}_{T}^{\beta}\|=0$ or some $\|t\|=0$. This is condition (I). Hence $f_{\beta}=0$ iff (B) or (I).
	
	\paragraph{3. The remaining case.}
	The values of $f_{\beta}$ lie in $[0,1]$. The above two points show exactly when $f=0$ and when $f=1$. Therefore, if neither holds, we have $0<f<1$. Conversely, if $0<f<1$, then neither $f=0$ nor $f=1$ holds, so the conditions for 0 and 1 are false. This completes the proof.
\end{proof}

	\begin{thm}[Existence of Solutions]
		For every EHSAF \(\mathcal{F}=(\mathbf{A},\mathbf{R}_{a},\mathbf{R}_{e},\mathbf{U},\mathbf{P})\), the CFOE system \(eq_{[0,1]}^{*,\otimes}\) has at least one solution. Consequently, the CFNE semantic model set \(\mathfrak{L}\mathfrak{S}_{ec_n}^{\mathcal{PL}_{[0,1]}^{*,\otimes}}(\mathcal{F})\) is non-empty for every \(\mathcal{F}\).
	\end{thm}
	\begin{proof}
		Let \(n = |\mathbb{U}|\). Define the vector-valued function
		\[
		F:\ [0,1]^n \longrightarrow [0,1]^n,\qquad
		F(\mathbf{x}) = (f_{\beta_1}(\mathbf{x}), \ldots, f_{\beta_n}(\mathbf{x})),
		\]
		where \(\{\beta_1,\ldots,\beta_n\}=\mathbb{U}\). By Theorem~\ref{continuity}, every \(f_\beta\) is continuous, hence \(F\) is continuous. The domain \([0,1]^n\) is non-empty, compact and convex. Moreover, since each \(f_\beta\) maps into \([0,1]\), \(F\) is a self-map of \([0,1]^n\). By Brouwer's Fixed-Point Theorem, there exists \(\mathbf{x}^*\in[0,1]^n\) such that \(F(\mathbf{x}^*)=\mathbf{x}^*\). This fixed point is precisely a solution of the CFOE system \(eq_{[0,1]}^{*,\otimes}\). By Theorem~\ref{fuzzyequivalence}, this fixed point is also a model of the CFNE semantics.
	\end{proof}
	
	\begin{cor}
		For any EHSAF, each of the fuzzy encoded semantics \(\mathfrak{L}\mathfrak{S}_{ec_n}^{\mathcal{PL}_{[0,1]}^{G}}\), \(\mathfrak{L}\mathfrak{S}_{ec_n}^{\mathcal{PL}_{[0,1]}^{P}}\), and \(\mathfrak{L}\mathfrak{S}_{ec_n}^{\mathcal{PL}_{[0,1]}^{L}}\) has at least one model.
	\end{cor}
	\begin{proof}
		Immediate from the theorem by taking \(\otimes\) to be the G{\"o}del, Product, or {\L}ukasiewicz t-norm, respectively.
	\end{proof}
	
	\subsubsection{Relationships between CFNE Semantics and 3-Valued Semantics}
	
	First we define the ternarization function.
	
	\begin{defn}[Ternarization for EHSAF]
		The \emph{ternarization function} $T_3 : \mathcal{L}\mathcal{A}\mathcal{B} \to \mathcal{L}\mathcal{A}\mathcal{B}_3$ maps a fuzzy assignment $\|\cdot\|$ to a 3‑valued assignment $\|\cdot\|_3 = T_3(\|\cdot\|)$ defined by
		\[
		\|x\|_3 =
		\begin{cases}
			1, & \|x\| = 1,\\
			0, & \|x\| = 0,\\
			\tfrac12, & \text{otherwise}.
		\end{cases}
		\]
	\end{defn}
		\paragraph*{Semantic Correspondence under Zero‑Divisor‑Free t-Norms\\}
	
	The 3‑valued equational system $eq_3^{EH}$ (Definition~\ref{eq3}) is given by
	\[
	\|\beta\|_3 = \min\{ A_3(\beta),\, B_3(\beta) \},
	\]
	with
	\[
	A_3(\beta) = \min_{\mathbf{k}_{S}^{\beta}\in \mathcal{K}_{\beta}} 
	\max\{ 1 - \|\mathbf{k}_{S}^{\beta}\|_3,\, \max_{s\in S}(1 - \|s\|_3) \},
	\]
	\[
	B_3(\beta) = \max_{\mathbf{t}_{T}^{\beta}\in \mathcal{T}_{\beta}} 
	\min\{ \|\mathbf{t}_{T}^{\beta}\|_3,\, \min_{t\in T}\|t\|_3 \},
	\]
	where all operations are over $\{0,\tfrac12,1\}$.
	
	\begin{thm}[Ternarisation of CFNE Models]\label{thm:ternarization}
		Let the CFNE semantics be induced by the normal encoding and a $\mathcal{PL}_{[0,1]}$ equipped with a continuous negation $N^{*}$ and a continuous zero‑divisor‑free t‑norm $\otimes$. For any EHSAF $\mathcal{F}$ and any fuzzy assignment $\|\cdot\|$,
		\[
		\|\cdot\| \models_{\text{CFNE}} \mathcal{F}
		\quad\Longrightarrow\quad
		T_3(\|\cdot\|) \;\text{is a solution of } eq_3^{EH}.
		\]
		Consequently, by Theorem~\ref{eqofeqandadcom}, $T_3(\|\cdot\|)$ is an adjacent complete labelling of $\mathcal{F}$.
	\end{thm}
	
	\begin{proof}
		Assume $\|\cdot\| \models_{\text{CFNE}} \mathcal{F}$. Then for every $\beta\in \mathbb{U}$, we have $\|\beta\| = f_{\beta}$ (by Theorem~\ref{fuzzyequivalence}). Let $\|\cdot\|_3 = T_3(\|\cdot\|)$. We need to show that for every $\beta$, $\|\beta\|_3 = \min\{A_3(\beta), B_3(\beta)\}$.
		
		We apply Theorem~\ref{thm:trifurcation} to the value $\|\beta\| = f_{\beta}$.
		
		\paragraph{Case $\|\beta\|=1$.} By Theorem~\ref{thm:trifurcation}, condition (A) and (E) hold. Condition (A) means for every attack, either $\|\mathbf{k}_{S}^{\beta}\|=0$ or some $\|s\|=0$. In ternarized terms, this implies $1-\|\mathbf{k}_{S}^{\beta}\|_3 = 1$ or $\max_{s}(1-\|s\|_3)=1$, so the max for that attack is 1. Hence $A_3(\beta)=1$. Condition (E) gives a support with $\|\mathbf{t}_{T}^{\beta}\|=1$ and all $\|t\|=1$, so $\|\mathbf{t}_{T}^{\beta}\|_3=1$ and all $\|t\|_3=1$, making the min for that support equal 1, hence $B_3(\beta)=1$. Thus $\min\{A_3,B_3\}=1=\|\beta\|_3$.
		
		\paragraph{Case $\|\beta\|=0$.} By Theorem~\ref{thm:trifurcation}, either (B) or (I) holds. If (B) holds, there is an effective attack, giving for that attack $\|\mathbf{k}_{S}^{\beta}\|_3=1$ and all $\|s\|_3=1$, so $1-\|\mathbf{k}_{S}^{\beta}\|_3=0$ and $\max_{s}(1-\|s\|_3)=0$, hence the max is 0, so $A_3(\beta)=0$. If condition (I) holds, every support is ineffective. This means that for each support, either $\|\mathbf{t}_{T}^{\beta}\|=0$ or $\|t\|=0$ for some $t$. Consequently, either $\|\mathbf{t}_{T}^{\beta}\|_3=0$ or $\|t\|_3=0$ for some $t$, so the minimum associated with this support equals $0$; therefore, $B_3(\beta)=0$. Thus $\min\{A_3,B_3\}=0=\|\beta\|_3$.
		
		\paragraph{Case $0<\|\beta\|<1$.} Then neither (A\&E) nor (B or I) holds under $\|\cdot\|$. Therefore, neither (A\&E) nor (B or I) holds under $\|\cdot\|_3$. Thus from Definitions~\ref{comlab} and \ref{eq3} and Theorem~\ref{eqofeqandadcom}, $\min\{A_3,B_3\}=\frac{1}{2}=\|\beta\|_3$. 
		
		Thus in all cases, the $eq_3^{EH}$ equation holds for every $\beta$. The auxiliary elements have fixed values satisfying the system. Therefore $T_3(\|\cdot\|)$ is a solution of $eq_3^{EH}$. By Theorem~\ref{eqofeqandadcom}, it is an adjacent complete labelling.
	\end{proof}

	\paragraph*{Semantic Correspondence under $\frac{1}{2}$-Idempotent t-Norms\\}
	
	Let $\mathcal{F}=(\mathbf{A},\mathbf{R}_{a},\mathbf{R}_{e},\mathbf{U},\mathbf{P})$ be an EHSAF.  
	For each $\beta\in \mathbb{U}$, let $\mathcal{K}_\beta$ be the set of attacks against $\beta$, and $\mathcal{T}_\beta$ the set of supports for $\beta$.  
	Let $\odot$ be a continuous $\frac{1}{2}$‑idempotent t‑norm (i.e., $\frac{1}{2}\odot\frac{1}{2}=\frac{1}{2}$) and let $N$ be the standard negation $N(x)=1-x$.  
	We consider assignments $\|\cdot\|:\mathbf{U}\to\{0,\tfrac12,1\}$.
	
	For $\beta\in \mathbb{U}$, define the following quantities:
	\[
	\mathbf{A}_{\beta} \;\stackrel{\text{def}}{=}\;
	\bigodot_{\mathbf{k}_{S}^{\beta}\in \mathcal{K}_\beta}
	N\!\left(\|\mathbf{k}_{S}^{\beta}\| \odot \bigodot_{s\in S}\|s\|\right),
	\]
	\[
	\mathbf{B}_{\beta} \;\stackrel{\text{def}}{=}\;
	N\!\left(
	\bigodot_{\mathbf{t}_{T}^{\beta}\in \mathcal{T}_\beta}
	N\!\left(\|\mathbf{t}_{T}^{\beta}\| \odot \bigodot_{t\in T}\|t\|\right)
	\right).
	\]
	
	The CFNE semantics (induced by the normal encoding and a $\mathcal{PL}_{[0,1]}$ equipped with the standard negation and the continuous $\frac{1}{2}$‑idempotent t‑norm) is equivalent to the equational system
	\begin{equation}\label{star3}
		\|\beta\| = \mathbf{A}_{\beta} \odot \mathbf{B}_{\beta} \qquad \text{for every } \beta\in \mathbb{U}.
	\end{equation}
	
	The adjacent complete semantics corresponding to the encoding is given by:
	
	\begin{align}\label{31tag1}
		\|\beta\|=1 \;&\Longleftrightarrow\;
		\Bigl(\forall \mathbf{k}_{S}^{\beta}\in \mathcal{K}_\beta,\;
		\exists s\in S \text{ with } \|s\|=0 \text{ or } \|\mathbf{k}_{S}^{\beta}\|=0\Bigr) \nonumber\\
		&\quad\text{and}\;
		\Bigl(\exists \mathbf{t}_{T}^{\beta}\in \mathcal{T}_\beta,\;
		\forall t\in T,\; \|t\|=1 \text{ and } \|\mathbf{t}_{T}^{\beta}\|=1\Bigr),
		\end{align}
		
		\begin{align}\label{31tag2}
		\|\beta\|=0 \;&\Longleftrightarrow\;
		\Bigl(\exists \mathbf{k}_{S}^{\beta}\in \mathcal{K}_\beta,\;
		\forall s\in S,\; \|s\|=1 \text{ and } \|\mathbf{k}_{S}^{\beta}\|=1\Bigr) \nonumber\\
		&\quad\text{or}\;
		\Bigl(\forall \mathbf{t}_{T}^{\beta}\in \mathcal{T}_\beta,\;
		\exists t\in T \text{ with } \|t\|=0 \text{ or } \|\mathbf{t}_{T}^{\beta}\|=0\Bigr),
	\end{align}
	
	\begin{align}\label{31tag3}
		\|\beta\|=\tfrac12 \;&\Longleftrightarrow\;
		\text{otherwise}. 
	\end{align}
	
	We now prove that every 3‑valued assignment satisfying equivalences \eqref{31tag1}–\eqref{31tag3} also satisfies Equation \eqref{star3}.
	
	\begin{thm}\label{1/2}
		For every EHSAF $\mathcal{F}$ and every $\|\cdot\|:\mathbf{U}\to\{0,\tfrac12,1\}$,
		\[
		\|\cdot\| \text{ is an adjacent complete labelling}
		\quad\Longrightarrow\quad
		\|\cdot\| \models_{\text{CFNE}} \mathcal{F}.
		\]
	\end{thm}
	
	\begin{proof}
		Fix $\beta\in \mathbb{U}$. We show the equivalence $\|\beta\| = \mathbf{A}_{\beta} \odot \mathbf{B}_{\beta}$ by proving the three cases.  
		We rely on the following fact: since all values lie in $\{0,\tfrac12,1\}$ and $\odot$ is $\frac{1}{2}$‑idempotent, the $\odot$‑product of any finite number of such values is:
		\begin{itemize}
			\item $1$ iff every factor is $1$;
			\item $0$ iff at least one factor is $0$;
			\item $\tfrac12$ otherwise (no $0$, but at least one $\tfrac12$).
		\end{itemize}
		Consequently, for any $x,y\in\{0,\tfrac12,1\}$, we have $x\odot y=0 \iff x=0\text{ or }y=0$, and $x\odot y=1 \iff x=1\text{ and }y=1$.
		
		\paragraph{Case 1: $\|\beta\|=0$.}
		We prove the equivalence:
		\[
		\|\beta\|=0
		\;\Longleftrightarrow\;
		\mathbf{A}_{\beta}\odot\mathbf{B}_{\beta}=0.
		\]
		
		First, assume $\|\beta\|=0$. By \eqref{31tag2}, either (i) there exists an effective attack, or (ii) every support is ineffective.
		
		(i) If there exists $\mathbf{k}_{S}^{\beta}$ with $\|\mathbf{k}_{S}^{\beta}\|=1$ and $\|s\|=1$ for all $s\in S$, then
		\[
		\|\mathbf{k}_{S}^{\beta}\| \odot \bigodot_{s\in S}\|s\| = 1,
		\]
		so $N(1)=0$. Hence $\mathbf{A}_{\beta}$ contains a factor $0$, so $\mathbf{A}_{\beta}=0$, and therefore $\mathbf{A}_{\beta}\odot\mathbf{B}_{\beta}=0$.
		
		(ii) If every support is ineffective, then for each $\mathbf{t}_{T}^{\beta}$,
		\[
		\|\mathbf{t}_{T}^{\beta}\| \odot \bigodot_{t\in T}\|t\| = 0,
		\]
		so $N(0)=1$ for every support. Thus the inner product
		\[
		\bigodot_{\mathbf{t}_{T}^{\beta}} N\!\left(\|\mathbf{t}_{T}^{\beta}\| \odot \bigodot_{t\in T}\|t\|\right)
		\]
		is a product of $1$'s, hence equals $1$. Therefore $\mathbf{B}_{\beta}=N(1)=0$, and thus $\mathbf{A}_{\beta}\odot\mathbf{B}_{\beta}=0$.
		
		Conversely, assume $\mathbf{A}_{\beta}\odot\mathbf{B}_{\beta}=0$. Since $x\odot y=0$ iff $x=0$ or $y=0$, we have either $\mathbf{A}_{\beta}=0$ or $\mathbf{B}_{\beta}=0$.
		
		If $\mathbf{A}_{\beta}=0$, then there exists an attack $\mathbf{k}_{S}^{\beta}$ such that
		\[
		N\!\left(\|\mathbf{k}_{S}^{\beta}\| \odot \bigodot_{s\in S}\|s\|\right)=0,
		\]
		which implies $\|\mathbf{k}_{S}^{\beta}\| \odot \bigodot_{s\in S}\|s\| = 1$. This is equivalent to $\|\mathbf{k}_{S}^{\beta}\|=1$ and $\|s\|=1$ for all $s\in S$, i.e., an effective attack exists.
		
		If $\mathbf{B}_{\beta}=0$, then
		\[
		\bigodot_{\mathbf{t}_{T}^{\beta}} N\!\left(\|\mathbf{t}_{T}^{\beta}\| \odot \bigodot_{t\in T}\|t\|\right) = 1,
		\]
		which implies that every factor is $1$, i.e., for each support,
		\[
		N\!\left(\|\mathbf{t}_{T}^{\beta}\| \odot \bigodot_{t\in T}\|t\|\right)=1,
		\]
		so $\|\mathbf{t}_{T}^{\beta}\| \odot \bigodot_{t\in T}\|t\| = 0$. This means either $\|\mathbf{t}_{T}^{\beta}\|=0$ or some $\|t\|=0$, i.e., every support is ineffective.
		
		Thus in either case, the condition in \eqref{31tag2} holds, so $\|\beta\|=0$.
		
		Therefore, $\|\beta\|=0 \Longleftrightarrow \mathbf{A}_{\beta}\odot\mathbf{B}_{\beta}=0$.
		
		\paragraph{Case 2: $\|\beta\|=1$.}
		We prove the equivalence:
		\[
		\|\beta\|=1
		\;\Longleftrightarrow\;
		\mathbf{A}_{\beta}\odot\mathbf{B}_{\beta}=1.
		\]
		
		Assume $\|\beta\|=1$. By \eqref{31tag1}, every attack is blocked and there exists an effective support.
		
		If every attack is blocked, then for each $\mathbf{k}_{S}^{\beta}$,
		\[
		\|\mathbf{k}_{S}^{\beta}\| \odot \bigodot_{s\in S}\|s\| = 0,
		\]
		so $N(0)=1$. Hence every factor in $\mathbf{A}_{\beta}$ is $1$, so $\mathbf{A}_{\beta}=1$.
		
		If there exists an effective support, then for that $\mathbf{t}_{T}^{\beta}$,
		\[
		\|\mathbf{t}_{T}^{\beta}\| \odot \bigodot_{t\in T}\|t\| = 1,
		\]
		so $N(1)=0$. Thus the inner product over supports contains a factor $0$, hence equals $0$. Therefore $\mathbf{B}_{\beta}=N(0)=1$.
		
		Consequently, $\mathbf{A}_{\beta}\odot\mathbf{B}_{\beta}=1\odot1=1$.
		
		Conversely, assume $\mathbf{A}_{\beta}\odot\mathbf{B}_{\beta}=1$. Since $x\odot y=1$ iff $x=1$ and $y=1$, we have $\mathbf{A}_{\beta}=1$ and $\mathbf{B}_{\beta}=1$.
		
		$\mathbf{A}_{\beta}=1$ means that every factor in the product is $1$, i.e., for every attack,
		\[
		N\!\left(\|\mathbf{k}_{S}^{\beta}\| \odot \bigodot_{s\in S}\|s\|\right)=1,
		\]
		so $\|\mathbf{k}_{S}^{\beta}\| \odot \bigodot_{s\in S}\|s\| = 0$. This is equivalent to either $\|\mathbf{k}_{S}^{\beta}\|=0$ or some $\|s\|=0$, i.e., every attack is blocked.
		
		$\mathbf{B}_{\beta}=1$ means that
		\[
		\bigodot_{\mathbf{t}_{T}^{\beta}} N\!\left(\|\mathbf{t}_{T}^{\beta}\| \odot \bigodot_{t\in T}\|t\|\right) = 0.
		\]
		Since the product is $0$, at least one factor is $0$, i.e., there exists a support $\mathbf{t}_{T}^{\beta}$ such that
		\[
		N\!\left(\|\mathbf{t}_{T}^{\beta}\| \odot \bigodot_{t\in T}\|t\|\right)=0,
		\]
		so $\|\mathbf{t}_{T}^{\beta}\| \odot \bigodot_{t\in T}\|t\| = 1$. This means $\|\mathbf{t}_{T}^{\beta}\|=1$ and $\|t\|=1$ for all $t\in T$, i.e., an effective support exists.
		
		Thus the conditions in \eqref{31tag1} hold, so $\|\beta\|=1$.
		
		Therefore, $\|\beta\|=1 \Longleftrightarrow \mathbf{A}_{\beta}\odot\mathbf{B}_{\beta}=1$.
		
		\paragraph{Case 3: $\|\beta\|=\tfrac12$.}
		By \eqref{31tag3}, $\|\beta\|$ is neither $0$ nor $1$. From the equivalences proven in Cases 1 and 2, we have:
		\[
		\|\beta\|=0 \Longleftrightarrow \mathbf{A}_{\beta}\odot\mathbf{B}_{\beta}=0,
		\qquad
		\|\beta\|=1 \Longleftrightarrow \mathbf{A}_{\beta}\odot\mathbf{B}_{\beta}=1.
		\]
		Hence $\mathbf{A}_{\beta}\odot\mathbf{B}_{\beta}$ is neither $0$ nor $1$. Since the $\odot$‑product of values from $\{0,\tfrac12,1\}$ also lies in that set (by $\frac{1}{2}$‑idempotency), the only remaining value is $\tfrac12$. Thus
		\[
		\mathbf{A}_{\beta}\odot\mathbf{B}_{\beta} = \tfrac12 = \|\beta\|.
		\]
		
		In all three cases, we have established $\|\beta\| = \mathbf{A}_{\beta} \odot \mathbf{B}_{\beta}$. Since $\beta$ was arbitrary, Equation \eqref{star3} holds for every element of $\mathbb{U}$. The auxiliary elements $\bot,\top$ and the added attacks/supports have fixed values that trivially satisfy the system. Hence $\|\cdot\| \models_{\text{CFNE}} \mathcal{F}$.
	\end{proof}
\begin{thm}\label{T3equivalent}
		Let the CFNE semantics be induced by the normal encoding and a $\mathcal{PL}_{[0,1]}$ equipped with a continuous $\frac{1}{2}$‑idempotent zero‑divisor‑free t‑norm and a standard negation. For any EHSAF $\mathcal{F}$,
		\begin{equation*}
			\{\|\cdot\|\mid\|\cdot\| \models_{\mathfrak{L}\mathfrak{S}_{ac}} \mathcal{F}\}=\{T_3(\|\cdot\|)\mid \|\cdot\| \models_{CFNE} \mathcal{F}\}.		
		\end{equation*}
\end{thm}
\begin{proof}
	From Theorem~\ref{thm:ternarization}, if $\|\cdot\|$ is a model of $\mathcal{F}$ under the given CFNE semantics, then $T_3(\|\cdot\|)$ is an adjacent complete labelling of $\mathcal{F}$. Thus,
		\begin{equation*}
		\{T_3(\|\cdot\|)\mid \|\cdot\| \models_{CFNE} \mathcal{F}\}\subseteq\{\|\cdot\|\mid\|\cdot\| \models_{\mathfrak{L}\mathfrak{S}_{ac}} \mathcal{F}\}.		
	\end{equation*}
	
	From Theorem~\ref{1/2}, if a 3-valued labelling $\|\cdot\|$ is an adjacent complete labelling of $\mathcal{F}$ then it is a model of $\mathcal{F}$ under the given CFNE semantics. Since $T_3(\|\cdot\|)=\|\cdot\|$ for this 3-valued labelling $\|\cdot\|$, we have 
	\begin{equation*}
		\{\|\cdot\|\mid\|\cdot\| \models_{\mathfrak{L}\mathfrak{S}_{ac}} \mathcal{F}\}\subseteq
		\{T_3(\|\cdot\|)\mid \|\cdot\| \models_{CFNE} \mathcal{F}\}.		
	\end{equation*}
	This completes the proof. 
\end{proof}
	\subsection{Key Instances of Fuzzy Encoded Semantics}
	Based on the unified CFOE/CFNE framework established in Section 4.1, we present three canonical instances induced by standard continuous t-norms: G\"{o}del, Product and {\L}ukasiewicz. Each instance corresponds to a specific aggregation pattern of collective attacks and evidential supports, and the semantic equivalence between equational and encoded forms follows directly from Theorem \ref{fuzzyequivalence}. We also discuss their core algebraic properties and intuitive interpretations.
	
	\subsubsection{G\"{o}del Fuzzy Encoded Semantics}
	The G\"{o}del t-norm $\otimes_G = \min$ follows the ``weakest link'' principle: the strength of a collective attack is determined by the weakest element in its source set, and the acceptability of the target is constrained by the strongest attack. This is the most natural extension of the discrete max-min equational semantics to the continuous $[0,1]$ domain.
	
	\begin{thm}[G\"{o}del-type CFOE system]
		Let $\mathcal{F} = (\mathbf{A}, \mathbf{R_a}, \mathbf{R_e}, \mathbf{U}, \mathbf{P})$ be an EHSAF and \(\|\cdot\|:\mathbf{U}\to[0,1]\) be an assignment. If the negation is the standard negation $N^{*}(x) = 1-x$ and the t-norm is the G\"{o}del t-norm $\otimes_G = \min$, then the G\"{o}del-type CFOE system $eq_{[0,1]}^{G}$ is given by the auxiliary-element conditions
		\[
		\|\perp\| = 0,\quad \|\top\| = 1,\quad \|\alpha\| = 1 \quad (\forall \alpha \in (\mathbf{R_e} \cup \mathbf{R_e}) \setminus (\mathbf{R_a} \cup \mathbf{R_e}))
		\]
		and for each $\beta \in \mathbb{U}$, the fixed-point equation:
		\[
		\|\beta\| = \min\left\{
		\underbrace{\min_{k_S^\beta \in \mathcal{K}_\beta} \max\left\{ 1-\|k_S^\beta\|, \max_{s\in S}(1-\|s\|) \right\}}_{A_G(\beta) \text{: attack part}},
		\underbrace{\max_{t_T^\beta \in \mathcal{T}_\beta} \min\left\{ \|t_T^\beta\|, \min_{t\in T}\|t\| \right\}}_{B_G(\beta) \text{: support part}}
		\right\}.
		\]
	\end{thm}
	\begin{proof}
		The result follows immediately from Definition~\ref{CFOE} by specifying the continuous negation and t-norm as the standard negation $N^{*}(x) = 1-x$ and the G\"{o}del t-norm $\otimes_G = \min$.
	\end{proof}
	Intuitively, the attack part $A_G(\beta)$ measures the degree to which \emph{all} collective attacks against $\beta$ fail. For each set attack, its failure degree equals the maximum failure degree among the attack name and all source elements; the overall attack failure degree is the minimum of these values, i.e., determined by the most effective attack.
	The support part $B_G(\beta)$ measures the degree to which \emph{at least one} evidential support for $\beta$ succeeds. For each set support, its success degree equals the minimum success degree among the support name and all source elements; the overall support success degree is the maximum of these values, i.e., determined by the strongest support.
	The final acceptability of $\beta$ is the degree to which both ``no effective attack'' and ``some effective support'' hold, aggregated by the G\"{o}del t-norm (minimum).
	
	\begin{thm}[Equivalence of G\"{o}del CFOE and CFNE semantics]
		For every EHSAF $\mathcal{F}$, the G\"{o}del-type continuous fuzzy operator-based equational semantics coincides with the G\"{o}del fuzzy normal encoded semantics, i.e.,
		\[
		\mathfrak{LS}_{Eq_{[0,1]}^{G}}(\mathcal{F}) = \mathfrak{LS}_{ec_n}^{\mathcal{P}\mathcal{L}_{[0,1]}^{G}}(\mathcal{F}).
		\]
	\end{thm}
	
	\begin{proof}
		The result follows immediately from Theorem \ref{fuzzyequivalence} by taking the continuous negation as the standard negation and the t-norm as the G\"{o}del t-norm $\otimes_G = \min$.
	\end{proof}
	
	As a special case of the general CFNE framework, the G\"{o}del fuzzy encoded semantics inherits all core algebraic properties established in Section \ref{property}, such as:
	\begin{enumerate}
		\item \emph{Continuity.} The aggregation function $f_\beta$ for each element $\beta$ is continuous on $[0,1]^{m_\beta}$.
		\item \emph{Boundary conditions.} If there exists a fully effective attack (the attack name and all arguments in its source set take value 1), or all supports are fully ineffective, then $\|\beta\| = 0$; if all attacks are fully ineffective and there exists at least one fully effective support, then $\|\beta\| = 1$.
		\item \emph{Monotonicity.} $\|\beta\|$ is non-increasing with respect to the truth value of any attack-related element, and non-decreasing with respect to the truth value of any support-related element.
		\item \emph{Existence of solutions.} The G\"{o}del CFOE system for any EHSAF has at least one solution, i.e., the model set of G\"{o}del CFNE semantics is non-empty.
	\end{enumerate}
	From Theorems \ref{thm:ternarization} to \ref{T3equivalent}, we obtain the following corollaries concerning the relationships between G\"{o}del CFNE semantics and adjacent complete labelling semantics.
\begin{cor}\label{}
	For any EHSAF $\mathcal{F}$ and any fuzzy assignment $\|\cdot\|:\mathbf{U}\to[0,1]$,
	\[
	\|\cdot\| \models_{\mathfrak{LS}_{ec_n}^{\mathcal{P}\mathcal{L}_{[0,1]}^{G}}} \mathcal{F}
	\quad\Longrightarrow\quad
	T_3(\|\cdot\|) \text{ is an adjacent complete labelling of } \mathcal{F}.
	\]
\end{cor}
\begin{proof}
	Since the G{\"o}del t-norm is zero‑divisor‑free and the standard negation is continuous in $\mathcal{P}\mathcal{L}_{[0,1]}^{G}$, this corollary follows from Theorem~\ref{thm:ternarization}.
\end{proof}

\begin{cor}\label{}
	For every EHSAF $\mathcal{F}$ and every $\|\cdot\|:\mathbf{U}\to\{0,\tfrac12,1\}$,
	\[
	\|\cdot\| \text{ is an adjacent complete labelling}
	\quad\Longrightarrow\quad
	\|\cdot\| \models_{\mathfrak{LS}_{ec_n}^{\mathcal{P}\mathcal{L}_{[0,1]}^{G}}} \mathcal{F}.
	\]
\end{cor}
\begin{proof}
	The G{\"o}del t-norm is $\tfrac12$-idempotent, and standard negation behaves appropriately in $\mathcal{P}\mathcal{L}_{[0,1]}^{G}$. Therefore, this corollary follows from Theorem~\ref{1/2}.
\end{proof}

\begin{cor}
	For any EHSAF $\mathcal{F}$,
	\begin{equation*}
		\{\|\cdot\|\mid\|\cdot\| \models_{\mathfrak{L}\mathfrak{S}_{ac}} \mathcal{F}\}=\{T_3(\|\cdot\|)\mid \|\cdot\| \models_{\mathfrak{LS}_{ec_n}^{\mathcal{P}\mathcal{L}_{[0,1]}^{G}}} \mathcal{F}\}.		
	\end{equation*}
\end{cor}
\begin{proof}
	Standard negation is continuous, and the G{\"o}del t-norm is continuous, $\tfrac12$-idempotent and zero‑divisor‑free in $\mathcal{P}\mathcal{L}_{[0,1]}^{G}$. Hence, this corollary follows from Theorem~\ref{T3equivalent}.
\end{proof}
	\subsubsection{Product Fuzzy Encoded Semantics}
	The Product t-norm $\otimes_P(x,y) = x \cdot y$ corresponds to probabilistic independent aggregation: it naturally models the joint effectiveness of multiple independent attacks and supports. This semantics is suitable for argumentation scenarios where different attack sources and evidence sources are mutually independent.
	
	\begin{defn}[Product-type CFOE system]
		Let $\mathcal{F} = (\mathbf{A}, \mathbf{R_a}, \mathbf{R_e}, \mathbf{U}, \mathbf{P})$ be an EHSAF and \(\|\cdot\|:\mathbf{U}\to[0,1]\) be an assignment. If the negation is the standard negation $N^{*}(x) = 1-x$ and the t-norm is the Product t-norm $\otimes_P(x,y) = x \cdot y$, then the Product-type CFOE system $eq_{[0,1]}^{P}$ is given by the auxiliary-element conditions
		\[
		\|\perp\| = 0,\quad \|\top\| = 1,\quad \|\alpha\| = 1 \quad (\forall \alpha \in (\mathbf{R_e} \cup \mathbf{R_e}) \setminus (\mathbf{R_a} \cup \mathbf{R_e}))
		\]
		and for each $\beta \in \mathbb{U}$, the fixed-point equation:
		\[
		\|\beta\| = \left( \prod_{k_S^\beta \in \mathcal{K}_\beta} \left( 1 - \|k_S^\beta\| \cdot \prod_{s\in S} \|s\| \right) \right) \cdot \left( 1 - \prod_{t_T^\beta \in \mathcal{T}_\beta} \left( 1 - \|t_T^\beta\| \cdot \prod_{t\in T} \|t\| \right) \right).
		\]
	\end{defn}
	
	Intuitively, the attack part aggregates the joint failure probability of all independent set attacks: the success probability of a single set attack is the product of the acceptability of the attack name and all source arguments; its failure probability is $1$ minus this value; the joint failure probability of all attacks is the product of individual failure probabilities. The support part aggregates the probability that at least one independent support succeeds: the success probability of a single set support is the product of the acceptability of the support name and all source arguments; the joint failure probability of all supports is the product of individual failure probabilities; the probability that at least one support succeeds is $1$ minus this value. The final acceptability of $\beta$ is the joint probability that all attacks fail and at least one support succeeds.
	
	\begin{thm}[Equivalence of Product CFOE and CFNE semantics]
		For every EHSAF $\mathcal{F}$, the Product-type continuous fuzzy operator-based equational semantics coincides with the Product fuzzy normal encoded semantics, i.e.,
		\[
		\mathfrak{LS}_{Eq_{[0,1]}^{P}}(\mathcal{F}) = \mathfrak{LS}_{ec_n}^{\mathcal{P}\mathcal{L}_{[0,1]}^{P}}(\mathcal{F}).
		\]
	\end{thm}
	
	\begin{proof}
		The result follows immediately from Theorem \ref{fuzzyequivalence} by taking the continuous negation as the standard negation and the t-norm as the Product t-norm $\otimes_P(x,y) = x \cdot y$.
	\end{proof}
	
	The Product fuzzy encoded semantics also inherits all general algebraic properties: continuity, commutativity, boundary conditions and monotonicity all hold, and the existence of solutions is guaranteed by Brouwer's fixed-point theorem. The following corollaries present the relationships between Product CFNE semantics and adjacent complete labelling semantics.
	\begin{cor}\label{product}
		For any EHSAF $\mathcal{F}$ and any fuzzy assignment $\|\cdot\|:\mathbf{U}\to[0,1]$,
		\[
		\|\cdot\| \models_{\mathfrak{LS}_{ec_n}^{\mathcal{P}\mathcal{L}_{[0,1]}^{P}}} \mathcal{F}
		\quad\Longrightarrow\quad
		T_3(\|\cdot\|) \text{ is an adjacent complete labelling of } \mathcal{F}.
		\]
	\end{cor}
	\begin{proof}
		Since the Product t-norm is zero‑divisor‑free and the standard negation is continuous in $\mathcal{P}\mathcal{L}_{[0,1]}^{P}$, this corollary follows from Theorem~\ref{thm:ternarization}.
	\end{proof}
	\begin{cor}
		For any EHSAF $\mathcal{F}$,
		\begin{equation*}
			\{T_3(\|\cdot\|)\mid \|\cdot\| \models_{\mathfrak{LS}_{ec_n}^{\mathcal{P}\mathcal{L}_{[0,1]}^{P}}} \mathcal{F}\}\subseteq\{\|\cdot\|\mid\|\cdot\| \models_{\mathfrak{L}\mathfrak{S}_{ac}} \mathcal{F}\}.		
		\end{equation*}
	\end{cor}
	\begin{proof}
		The result follows immediately from Corollary~\ref{product}.
	\end{proof}
	Unlike the idempotent G\"{o}del semantics, the Product t-norm is non-idempotent: multiple independent attacks gradually accumulate and weaken the acceptability of the target, rather than being determined solely by the strongest attack. Similarly, multiple independent supports gradually enhance the acceptability. This cumulative feature makes Product semantics more suitable for modeling argumentative reasoning with additive strength.
	
	\subsubsection{{\L}ukasiewicz Fuzzy Encoded Semantics}
	The {\L}ukasiewicz t-norm $\otimes_L(x,y) = \max\{0, x+y-1\}$ corresponds to threshold-based additive aggregation. It models mutually exclusive or competitive attacks and supports: the total strength of attacks/supports accumulates linearly, and exerts a full effect only when the total exceeds a certain threshold.
	
	The {\L}ukasiewicz-type CFOE system can be obtained directly by substituting $\otimes$ with $\otimes_L$ in Definition \ref{CFOE}, yielding a piecewise-linear threshold-sum expression. By the general equivalence theorem (Theorem \ref{fuzzyequivalence}), it also coincides with the corresponding {\L}ukasiewicz fuzzy normal encoded semantics:
	\[
	\mathfrak{LS}_{Eq_{[0,1]}^{L}}(\mathcal{F}) = \mathfrak{LS}_{ec_n}^{\mathcal{P}\mathcal{L}_{[0,1]}^{L}}(\mathcal{F}).
	\]
	
	The {\L}ukasiewicz semantics is suitable for scenarios where attack strengths and support strengths are competitively additive. Since its expanded form is relatively cumbersome and all core properties are special cases of the general framework, we omit the detailed derivation here.

\section{Related Work}

Our EHSAF framework builds upon and tightly integrates several established lines of research in abstract argumentation. We emphasize the precise relationships and correspondences between EHSAF and its predecessors, showing how each can be recovered as a special case of our unified framework, and how their semantics and logical encodings align with ours under appropriate restrictions.

\subsection{Dung's Abstract Argumentation and Basic Extensions}

Dung's seminal work \cite{dung1995acceptability} introduced DAFs with a binary attack relation and a family of extension-based semantics. While DAFs provide a simple and powerful model, they cannot represent support, higher-order interactions, or collective attacks. Our EHSAF framework conservatively generalizes DAFs: when the support relation is empty, all elements are prima-facie, all sources are singletons of arguments, and all targets are arguments, the extension-based semantics of EHSAF coincide exactly with Dung's semantics. This makes EHSAF a direct extension of the foundational framework.

\subsection{Evidential Support Frameworks}

The EAS \cite{oren2008semantics,polberg2014revisiting} introduced the notion of evidential support, where arguments must be backed by chains of support rooted in a special evidence argument \(\eta\). EAS allows collective sources (sets of arguments) for both attacks and supports, but does not support higher-order interactions (elements targeting interactions). Our EHSAF extends EAS by permitting attacks and supports to target arbitrary elements (including other attacks and supports) and by allowing sources to include attacks and supports as well as arguments. In the restriction where all targets are arguments (no higher-order interactions) and all sources are sets of arguments, EHSAF's extension-based semantics reduce to those of EAS, as both frameworks then share the same definition of evidential support and attack validity.

The REBAF \cite{cayrol2018argumentation} integrates evidential support with higher‑order attacks and supports. It further introduces the concept of \emph{structures}, which partitions a set into three disjoint components: arguments, attacks, and supports. However, REBAF restricts the source of every interaction to a set of arguments only; attacks and supports cannot have sources that include other attacks or supports. Our EHSAF generalizes REBAF by allowing sources to be arbitrary sets of elements from the universal set (arguments, attacks, or supports). Furthermore, the extension-based semantics of EHSAF are a direct generalization of REBAF's structure semantics: when the source of every interaction is restricted to arguments, the definitions of \(\operatorname{\mathbf{Def}}_a\), \(\operatorname{\mathbf{Sup}}\), \(\operatorname{\mathbf{Def}}_s\), and \(\operatorname{\mathbf{Acc}}\) in EHSAF coincide with those of REBAF. Thus, REBAF is a proper subframework of EHSAF, obtained by restricting source sets to arguments only.

We now compare EHSAF with two closely related lines of work on evidential argumentation with labelling semantics, highlighting both conceptual connections and technical differences.

\paragraph{Comparison with incremental evidence-based semantics}
Chen et al. \cite{chen2023evidence} proposed an incremental labelling semantics for evidence-based argumentation frameworks based on SCC decomposition.
\begin{itemize}
	\item \emph{Connection}: Both frameworks share the intuition that evidential support propagates upward from primitive evidence, and that layered computation can improve efficiency. Both adopt the labelling paradigm with in/out/undec statuses.
	\item \emph{Difference}: The work of Chen et al. focuses on computational efficiency under dynamic updates and is restricted to first-order, argument-to-argument attacks and supports. EHSAF, by contrast, prioritises expressive uniformity: it generalises sources to arbitrary sets of elements and targets to arguments, attacks and supports, and further provides propositional logic encodings and fuzzy extensions rather than incremental algorithms.
\end{itemize}

\paragraph{Comparison with principled labelling analysis of bipolar argumentation}
Al Anaissy et al. \cite{alanaissy2025principle} presented a principled classification of seven complete labelling semantics for bipolar argumentation under deductive, necessary and evidential support interpretations.
\begin{itemize}
	\item \emph{Connection}: Both works take complete labelling semantics as the core and cover the evidential interpretation of support. The adjacent complete labellings of EHSAF, when restricted to first-order singleton interactions, align with the evidential-support labelling semantics studied in their taxonomy.
	\item \emph{Difference}: Their work focuses on principled property comparison and classification of existing semantics within standard bipolar frameworks. EHSAF goes further by constructing a unified formal syntax that accommodates higher-order targets and collective sources, proving equivalence between labellings and logical models, and extending the whole framework to continuous fuzzy semantics based on t-norms.
\end{itemize}

\subsection{Higher-Order and Collective Attack Frameworks}

The HSAF \cite{tang2025encoding2} unifies higher-order attacks (attacks targeting other attacks) and collective attacks (attacks from sets of arbitrary elements) in a single framework, but does not include any support relation. In HSAF, every element is considered to be valid-evidential-supported by default—there is no notion of evidential support. Our EHSAF can be seen as an extension of HSAF by adding evidential support. Concretely, if we take an EHSAF and let the support relation to be empty, and declare every element (including attacks and supports) to be prima-facie, then the resulting framework is semantically equivalent to an HSAF with the same set of arguments, attacks, and targets. The logical encoding of EHSAF (Definition~\ref{normalencoding}) then reduces to the encoding of HSAF provided in \cite{tang2025encoding2}, because the support part of the formula becomes trivial. Thus, HSAF is exactly the special case of EHSAF where evidential support is absent and all elements are self-supported.

\subsection{Logical Encodings of Argumentation Frameworks}

Logical encodings of argumentation frameworks provide a bridge to automated reasoning by translating abstract structures into logical formulae whose models correspond to extensions or labellings. For REBAF, Cayrol and Lagasquie-Schiex \cite{cayrol2020logical} proposed a first-order logical encoding, establishing a correspondence between REBAF structures (triples of arguments, attacks, and supports) under given semantics and models of a first-order theory. Besnard et al. \cite{besnard2023generic} developed a generic first-order logical encoding that uniformly captures several families of frameworks, including those with coalitions, higher-order relations, and evidential supports, by parameterizing the theory according to the allowed enrichments.

Our normal encoding of EHSAF (Definition~\ref{normalencoding}) provides a \emph{propositional} counterpart to these first-order encodings. While our encoding shares the same fundamental principles—characterizing acceptance via logical equivalences and establishing a model-extension correspondence—it operates entirely within propositional logic, enabling direct use of SAT solvers and other efficient propositional reasoning tools. When EHSAF is restricted to the REBAF case (sources are sets of arguments), the semantic correspondence established by our propositional encoding coincides with that of the REBAF first-order encoding: both characterize the same acceptance modulo the difference in logical language. Similarly, when EHSAF is restricted to the HSAF case (no support, all elements prima-facie), our encoding reduces to the propositional encoding of HSAF presented in \cite{tang2025encoding2}, since the support conditions become trivial. This demonstrates that our encoding is a conservative propositionalization of the first-order encodings for the relevant sub-frameworks, while extending them to handle sources that are sets of arbitrary elements and providing a uniform treatment of attacks and supports at any order.

Furthermore, our fuzzy encoding provides a numerical semantics for EHSAF, which has no direct counterpart in the aforementioned encodings. However, the fuzzy encoded semantics can be viewed as a continuous generalization of the 3-valued complete labelling semantics, and under the G{\"o}del or Product, it recovers the 3-valued semantics when restricted to the \(\{0,\frac{1}{2},1\}\) domain.

\subsection{Fuzzy and Gradual Argumentation Semantics}

Several works have explored fuzzy or gradual semantics for argumentation. Wu and Oren \cite{wu2016properties} examined Janssen's fuzzy argumentation frameworks (JAF) \cite{janssen2008fuzzy}, clarifying definitions and adding auxiliary notions that make the system more understandable and simplify the computation of extensions. In JAF, arguments and attacks are assigned fuzzy degrees, and the acceptability of an argument is determined by comparing a computed value against a threshold \cite{wu2016properties}. This threshold-based approach, while providing a quantitative dimension, introduces parameters that must be externally specified.

To address this limitation, Wu et al. \cite{wu2016godel} introduced G{\"o}del Fuzzy Argumentation Frameworks (GFAFs), which combine fuzzy set theory with argumentation without requiring any external parameters. By specializing the framework using the G{\"o}del t-norm, they showed that the standard Dung extensions are recovered, though the stable and preferred semantics coincide \cite{wu2016godel}. This work brings fuzzy argumentation closer to Dung's original spirit by unequivocally identifying justified sets of arguments without reference to parameters. Wu et al. \cite{wu2020godel} further developed G{\"o}del semantics of fuzzy argumentation frameworks with consistency degrees, exploring the relationship between fuzzy argumentation and consistency measures.

Corsi \cite{corsi2025bipolar} investigated bipolar argumentative semantics for t-norm based fuzzy logics, and Wang and Shen \cite{wang2025fuzzy} proposed fuzzy labelling semantics for quantitative argumentation, where a triple of acceptability, rejectability, and undecidability degrees is used to evaluate argument strength \cite{wang2025fuzzy}.

Our fuzzy encoded semantics for EHSAF is closely related to these fuzzy-operator-based approaches, providing a natural gradual interpretation where attacks decrease acceptability and supports increase it, and the aggregation functions are continuous and monotone. Similar to the G{\"o}del Fuzzy Argumentation Frameworks \cite{wu2016godel}, our semantics under the G{\"o}del t-norm recovers the standard complete semantics when restricted to the 3-valued domain. Unlike the threshold-based approach of JAF \cite{janssen2008fuzzy,wu2016properties}, our semantics requires no external parameters. Furthermore, our framework extends the encoding approach of Tang et al. \cite{tang2025encoding2} by handling sources that are sets of arbitrary elements (including attacks and supports) and by providing a uniform treatment of both attacks and supports at any order. Most importantly, unlike many existing fuzzy argumentation frameworks, our semantics is derived directly from a logical encoding, ensuring a tight correspondence between the fuzzy models and the discrete complete labellings. This logical foundation also enables the systematic construction of new equational semantics by choosing different fuzzy logic operations, following the methodology proposed in \cite{tang2025encoding1}.

\section{Conclusion}
\label{sec:conclusion}

In this paper, we introduced the Evidential-based Higher-order Set Argumentation Framework (EHSAF), a unified abstract argumentation formalism that integrates three key generalizations: evidential support (rooted in prima-facie elements), higher-order interactions (attacks and supports targeting arbitrary elements), and collective interactions (sources as sets of arbitrary elements). The framework conservatively generalises existing formalisms including EAS, REBAF, and HSAF, providing a single expressive setting for studying complex argumentative structures.

We developed two complete semantics for EHSAFs: an \emph{adjacent complete labelling semantics} that adopts an open epistemic attitude toward arguments in support cycles, admitting multiple possible truth values (true, false, undecided); and an \emph{extension-based complete semantics} that follows a strict evidentialist stance, requiring well-founded support chains. We demonstrated through a concrete counterexample that these two semantics diverge in the presence of support cycles. Under support-acyclicity, we proved the two semantics coincide, establishing that the divergence is exactly due to cyclic dependencies.

To provide a computational foundation, we gave a normal propositional encoding of EHSAFs and proved that, in three-valued {\L}ukasiewicz logic, the models of the encoded theory correspond exactly to the adjacent complete labellings. We further extended the encoding to continuous fuzzy logics (G{\"o}del, Product, and {\L}ukasiewicz), defining a continuous fuzzy normal encoded semantics. We showed that this fuzzy semantics satisfies desirable properties—continuity, monotonicity, boundary conditions, and solution existence—and established that, under zero-divisor-free and \(\frac{1}{2}\)-idempotent t-norms, the ternarisation of any fuzzy model is an adjacent complete labelling, and conversely every adjacent labelling is a ternary model of the fuzzy encoding.

Our work thus bridges the gap between expressive argumentation frameworks and principled semantic foundations, offering both a flexible three-valued labelling theory and a continuous fuzzy counterpart with a clear logical interpretation. Future work includes investigating the incremental computation of adjacent complete labellings, exploring preference-based or weighted extensions, and applying the framework to real-world domains such as legal reasoning or scientific evidence evaluation, where the distinction between current and prospective evidence is epistemically significant.

\bibliographystyle{plain}
\bibliography{refofehsaf}

\end{document}